\documentclass{article}

\PassOptionsToPackage{numbers, sort, compress}{natbib}
\usepackage[final]{nips_2017}

\usepackage[utf8]{inputenc} % allow utf-8 input
\usepackage[T1]{fontenc}    % use 8-bit T1 fonts
\usepackage{hyperref}       % hyperlinks
\usepackage{url}            % simple URL typesetting
\usepackage{booktabs}       % professional-quality tables
\usepackage{amsfonts}       % blackboard math symbols
\usepackage{nicefrac}       % compact symbols for 1/2, etc.
\usepackage{microtype}      % microtypography

\usepackage{calc}
\usepackage{amsmath, amssymb, amsthm}
\usepackage{parskip}
\usepackage{color,hyperref}
\usepackage{epsfig}
\usepackage{subfigure}
\usepackage{verbatim}
\usepackage{rotating}
\usepackage[ruled,noend]{algorithm2e}
\usepackage{algorithmic}
\usepackage{enumitem}
\usepackage{wrapfig}
\usepackage{cleveref}
\usepackage{bbm}
\usepackage{array}
\usepackage{soul}
\usepackage{kky}
\usepackage{thompsonDefns}
\usepackage{xspace}

\definecolor{darkgreen}{rgb}{0.1,0.6,0.1}

\title{\large Asynchronous Parallel Bayesian Optimisation via Thompson Sampling}

\newcommand{\instcmu}{$\,^\natural$}
\newcommand{\instumass}{$\,^\flat$}
\newcommand{\authspace}{$\;\;$}

\author{
Kirthevasan Kandasamy\instcmu, \authspace
Akshay Krishnamurthy\instumass, \authspace
Jeff Schneider\instcmu, \authspace
Barnab\'as P\'oczos\instcmu
\\
  \instcmu $\,$Carnegie Mellon University,  \authspace \authspace
  \instumass $\,$University of Massachusetts, Amherst \\
 \incmtt{\{kandasamy, schneide, bapoczos\}@cs.cmu.edu, $\,$akshay@cs.umass.edu
   }
}

\begin{document}
\pdfoutput=1

\maketitle

% All Figures go here.
\newcommand{\imarrwtwo}{2.60in}
\newcommand{\imhsptwo}{0.20in}
\newcommand{\imarrwthree}{1.87in}
\newcommand{\imhspthree}{-0.03in}
\newcommand{\imleftspace}{-0.30in}
\newcommand{\imrightspace}{-0.20in}
\newcommand{\imtextspace}{-0.15in}
\newcommand{\imsinglecol}{2.495in}
\newcommand{\imcaptionspace}{-0.1in}

\newcommand{\insertFigSynAsynSchemes}{
\begin{figure}
\centering
\hspace{\imleftspace}
  \includegraphics[width=\imarrwtwo]{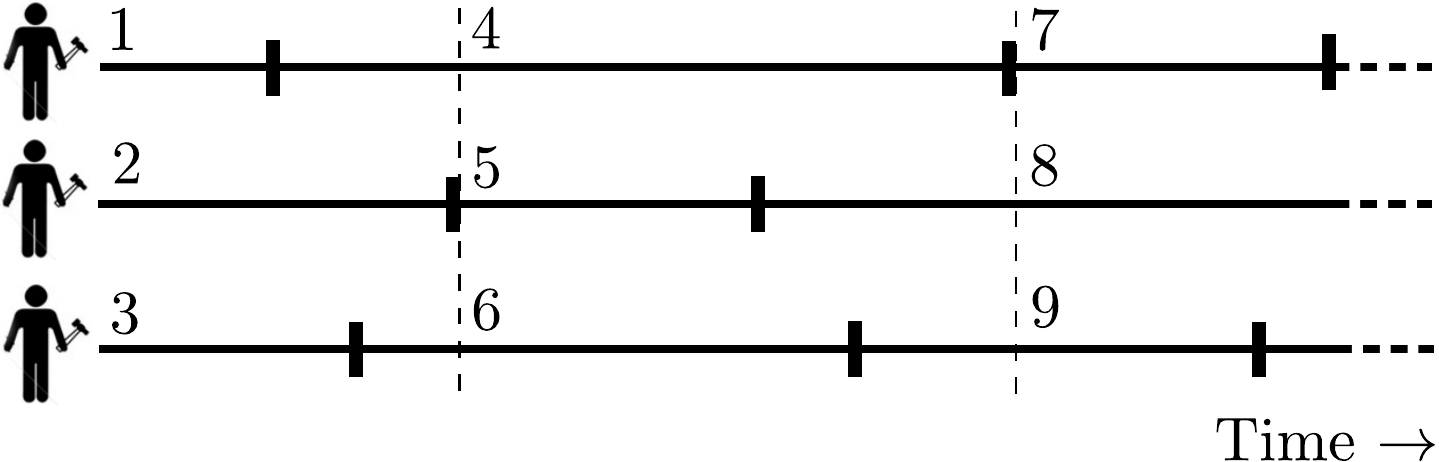} \hspace{\imhsptwo}
  \includegraphics[width=\imarrwtwo]{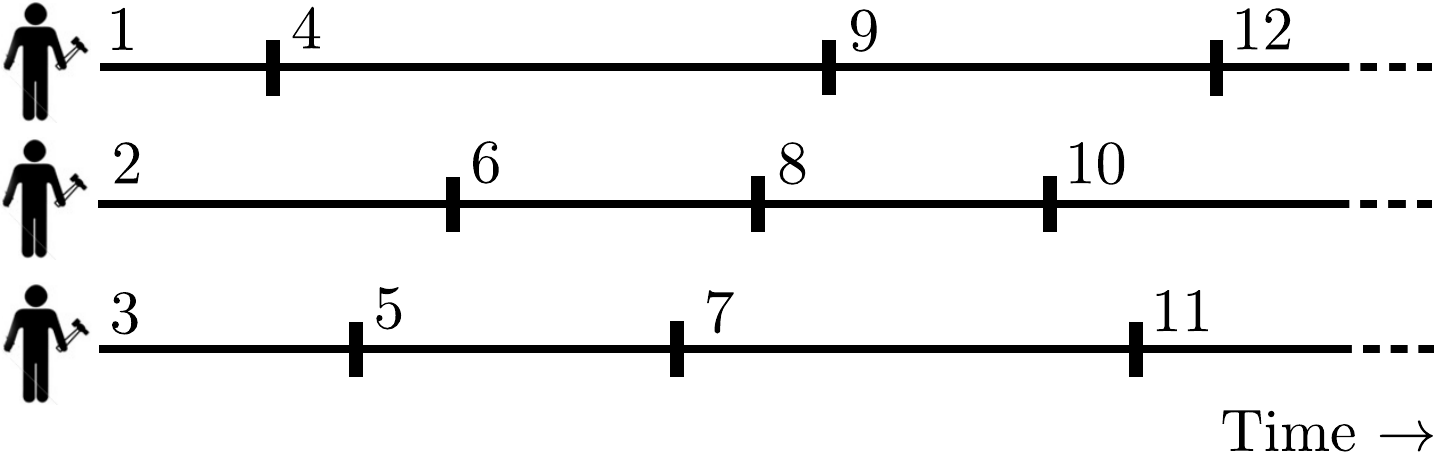}
\hspace{\imrightspace}
% \vspace{-0.05in}
\caption{\small
\label{fig:parallelschemes}
% An illustration of the indexing of function evaluations for synchronous (left) and
% asynchronous (right) strategies using $M=3$ workers.
An illustration of the synchronous (left) and asynchronous (right) settings using $M=3$
workers.
The short vertical lines indicate when a worker finished its last evaluation.
The horizontal location of a number indicates when the worker started its next evaluation
while the number itself denotes the order in which the evaluation was dispatched by
the algorithm.
% A synchronous algorithm waits for all workers to finish before
% dispatching the next batch of function evaluations.
% A asynchronous algorithm redeploys a worker as soon as it finishes
% its previous evaluation.
% \vspace{-0.1in}
}
\end{figure}
}

\newcommand{\insertFigToyOne}{
\begin{figure}
\centering
\hspace{\imleftspace}
  \includegraphics[width=\imarrwthree]{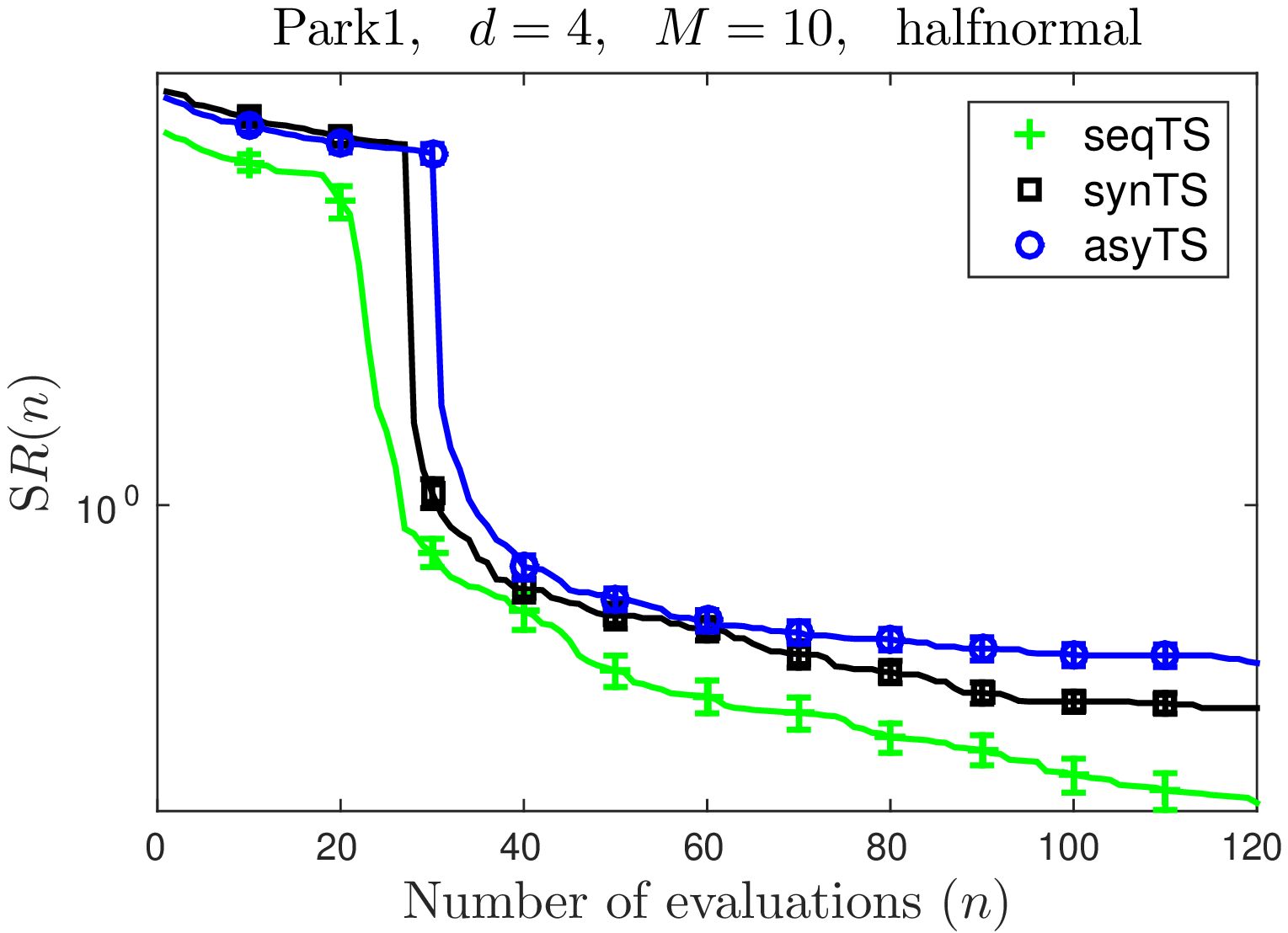} \hspace{\imhspthree}
  \includegraphics[width=\imarrwthree]{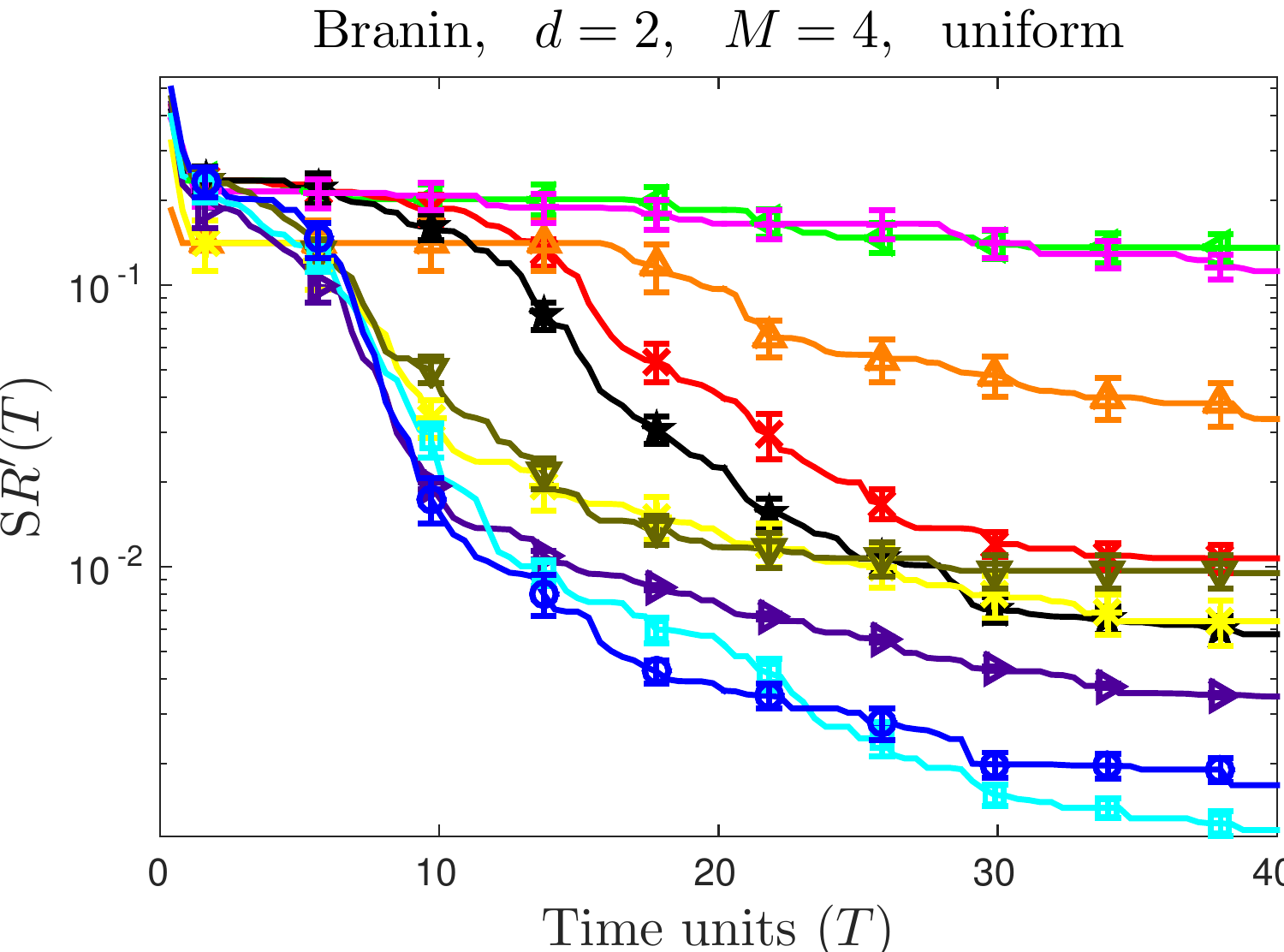} \hspace{\imhspthree}
  \includegraphics[width=1.82in]{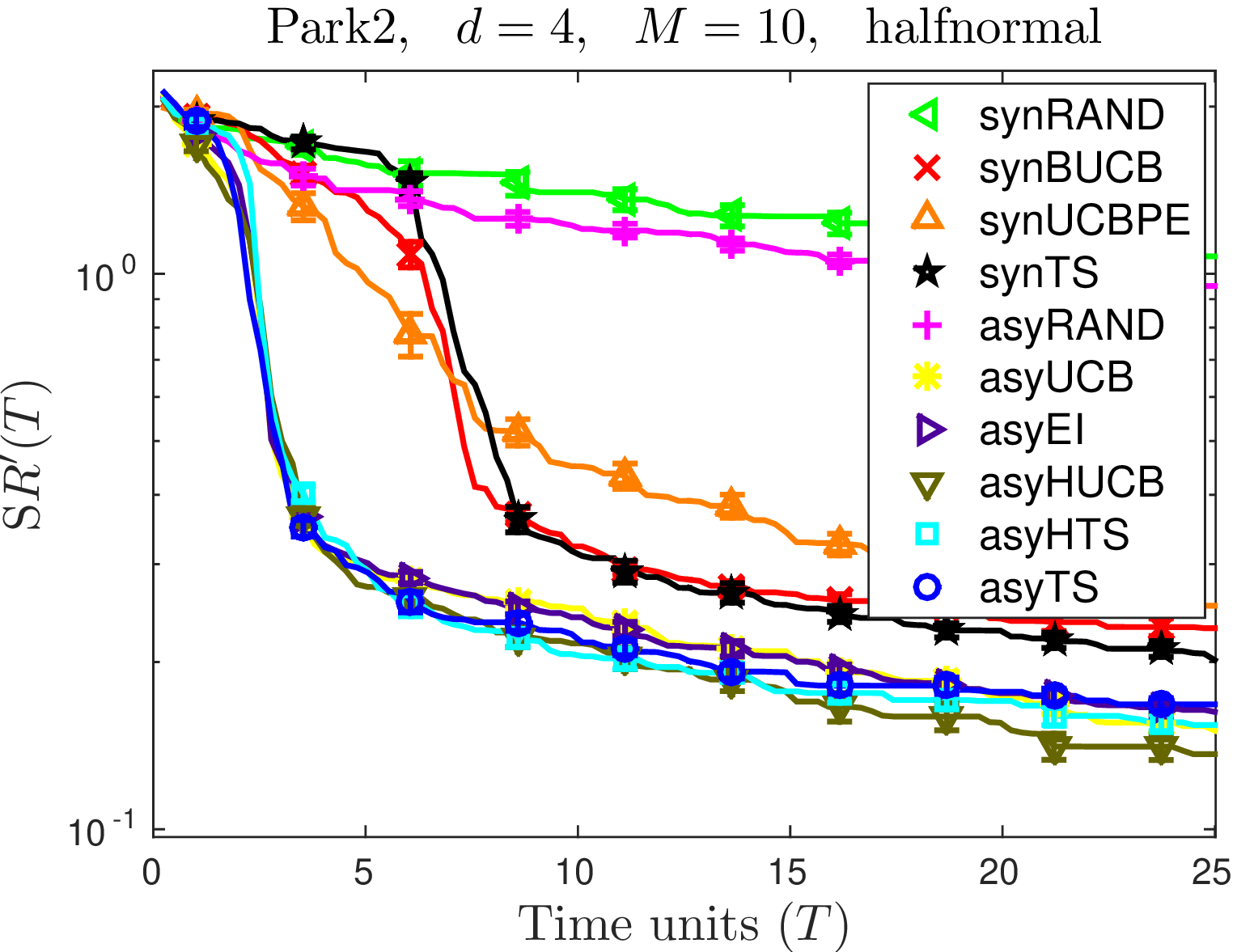} 
\hspace{\imrightspace}
\\[0.05in]%
\hspace{\imleftspace}
  \includegraphics[width=\imarrwthree]{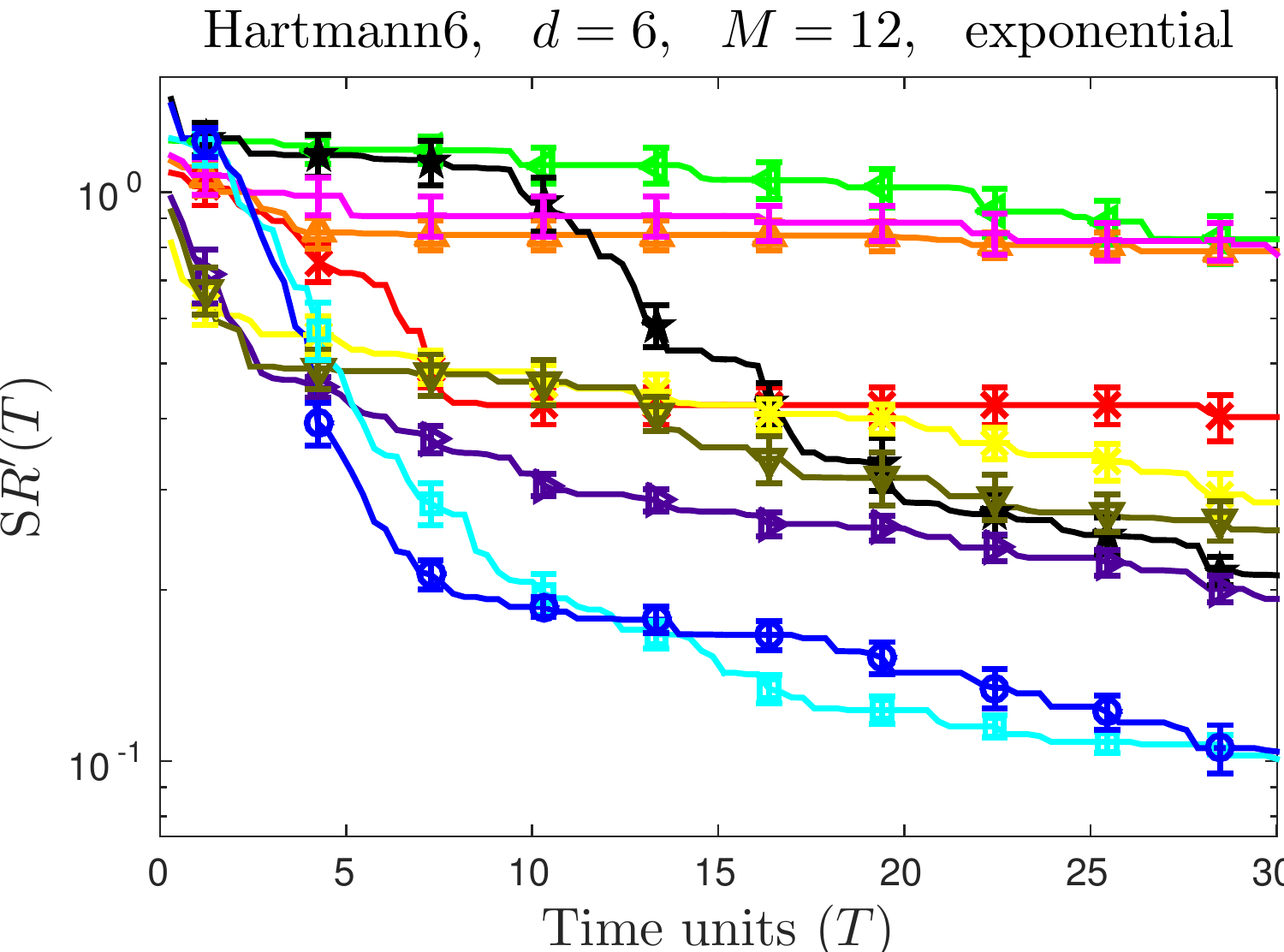} \hspace{\imhspthree}
  \includegraphics[width=\imarrwthree]{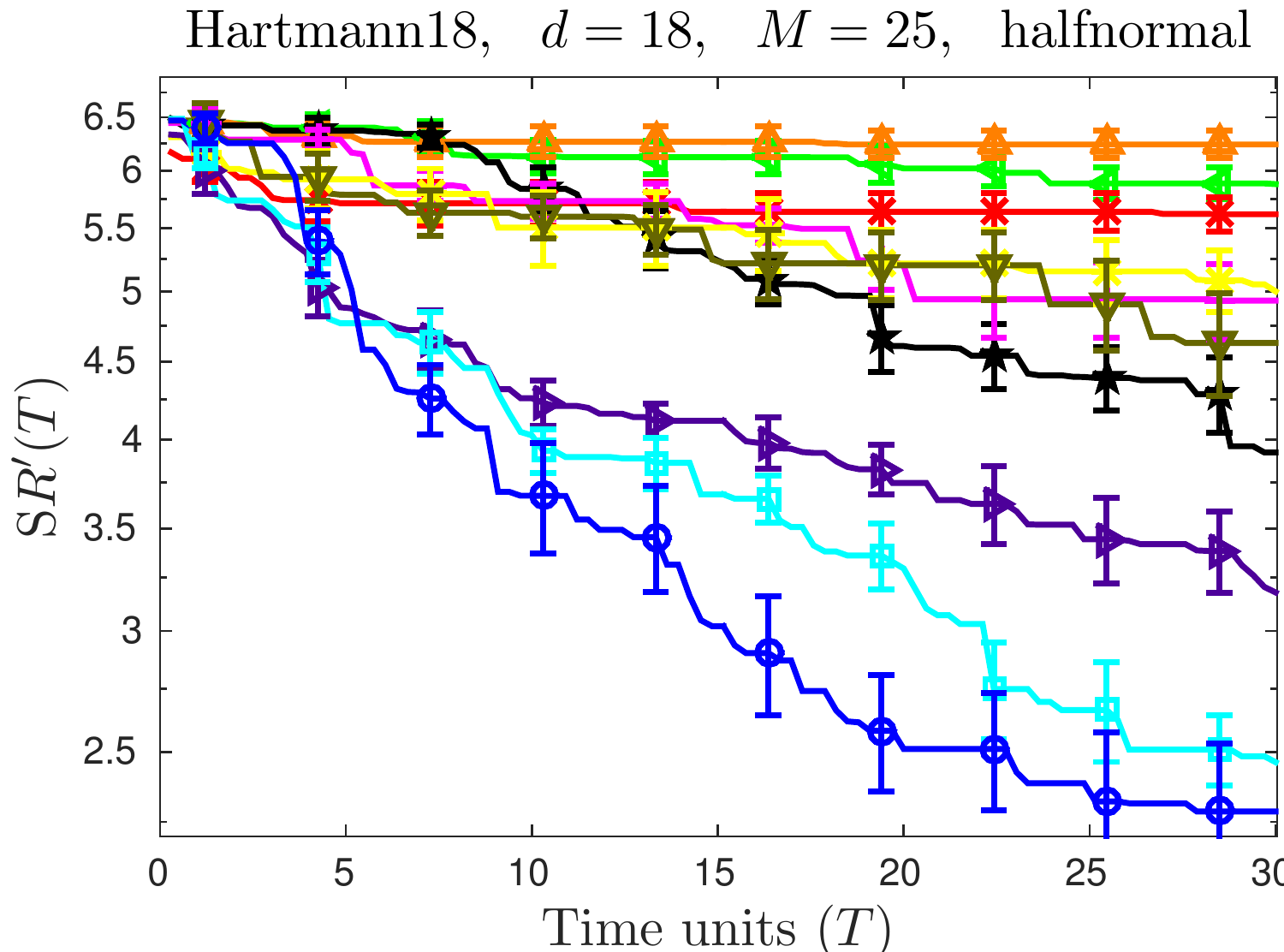} \hspace{\imhspthree}
  \includegraphics[width=\imarrwthree]{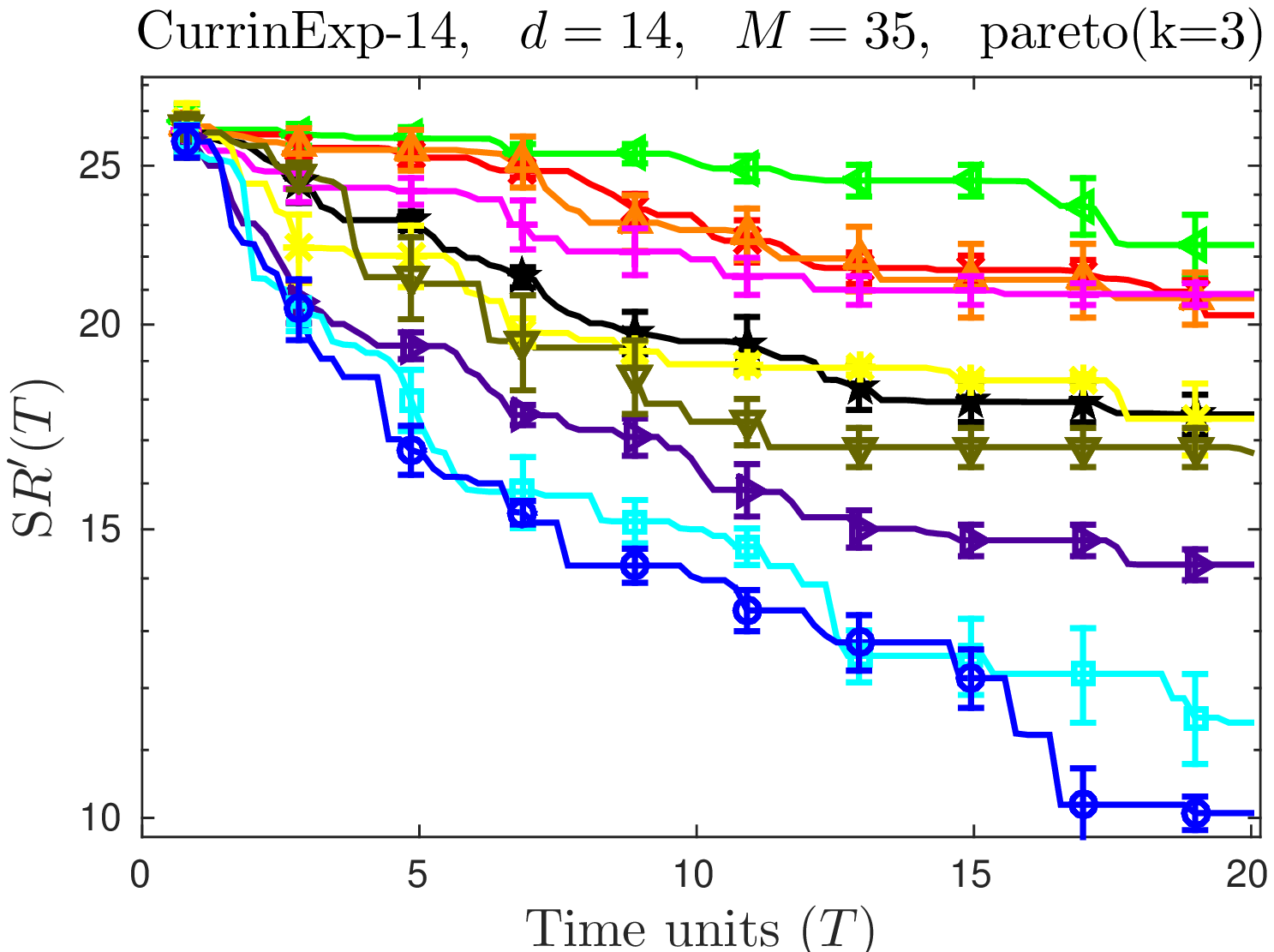} 
% \hspace{\imrightspace}
% \hspace{\imrightspace}
% \vspace{-0.05in}
\caption{\small
\label{fig:toymain}
Results on the synthetic experiments.
The title states the function used, its dimensionality $d$, the number of
workers $M$ and the distribution used for the time.
All distributions were constructed so that the expected time for one evaluation
was one time unit.
All figures were averaged over at least $15$ experiments.
\vspace{-0.2in}
}
\end{figure}
}

\newcommand{\imapparrwtwo}{2.80in}
\newcommand{\imapphsptwo}{0.10in}

\newcommand{\insertFigToyAppOne}{
\begin{figure}
\centering
\hspace{\imleftspace}
  \includegraphics[width=\imapparrwtwo]{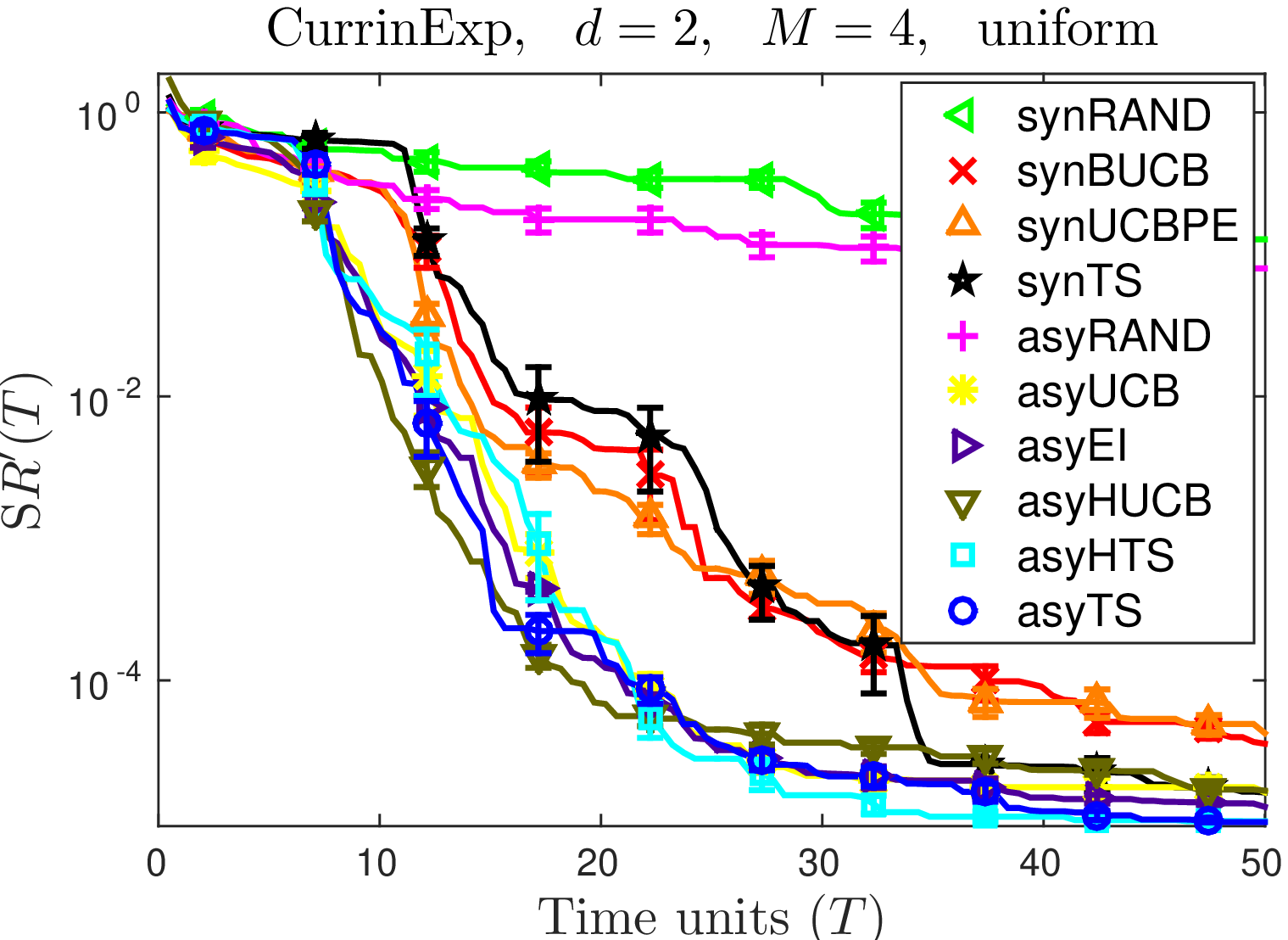} \hspace{\imapphsptwo}
  \includegraphics[width=\imapparrwtwo]{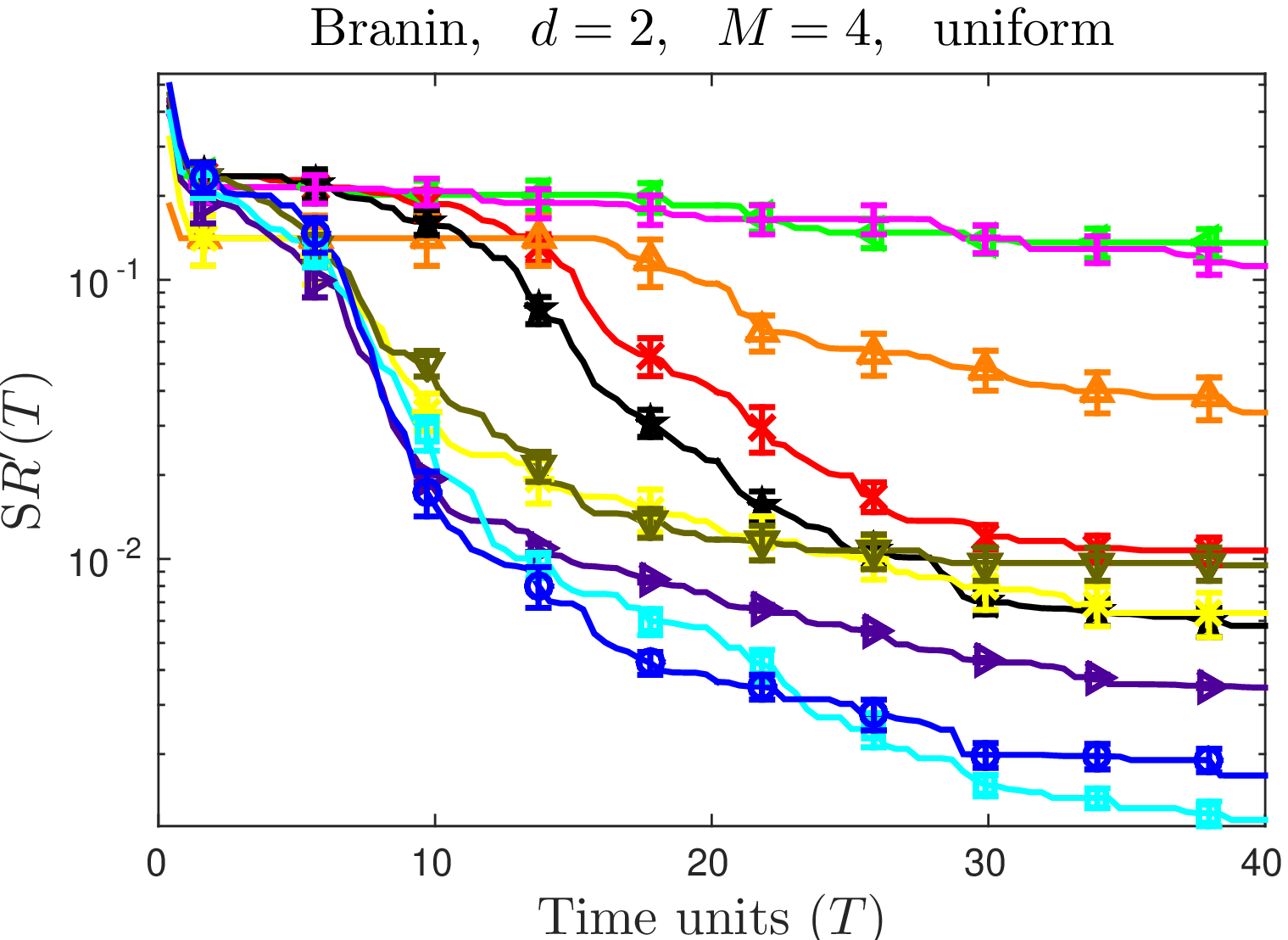}
\hspace{\imrightspace}
\\
\hspace{\imleftspace}
  \includegraphics[width=\imapparrwtwo]{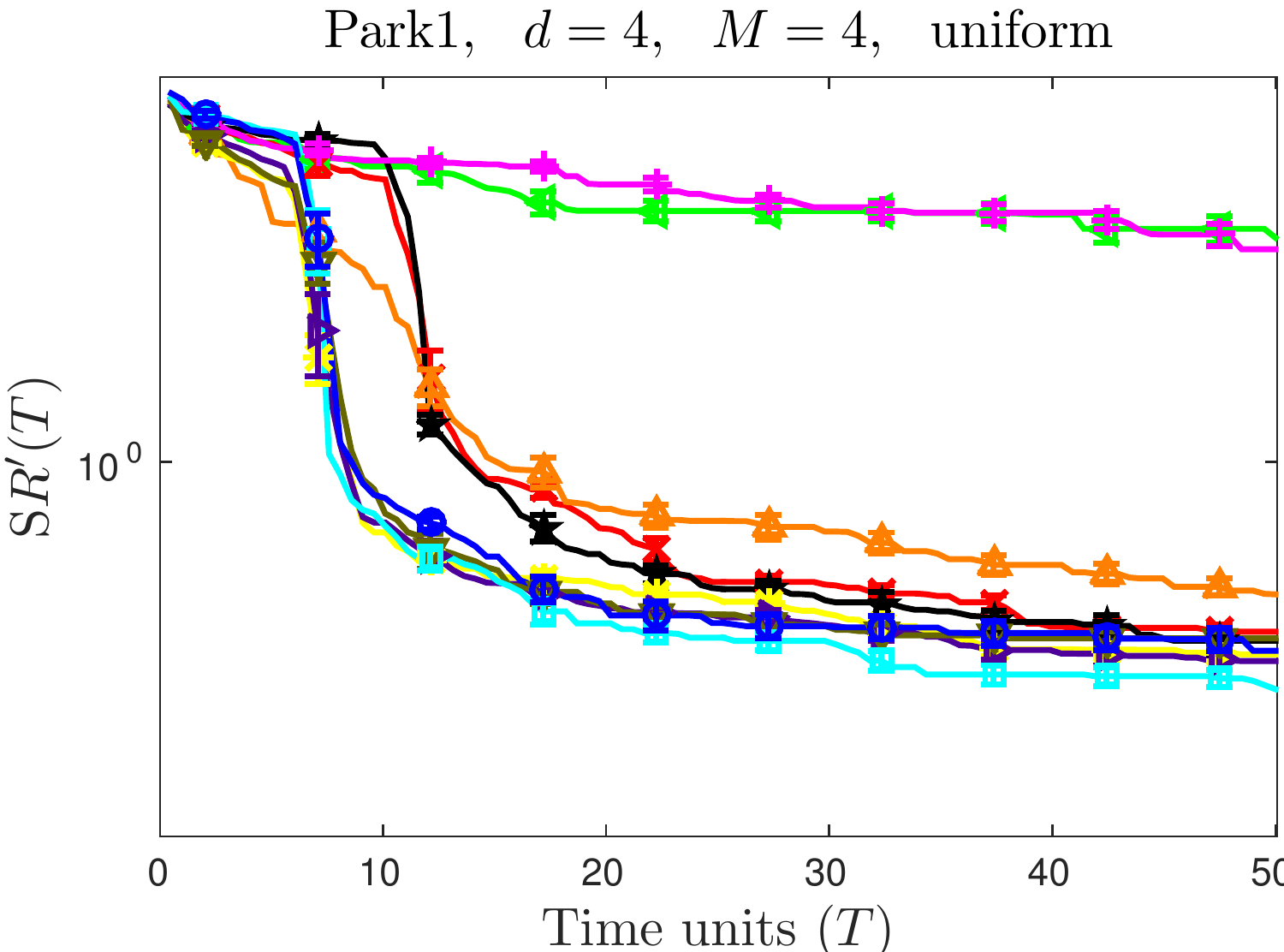} \hspace{\imapphsptwo}
  \includegraphics[width=\imapparrwtwo]{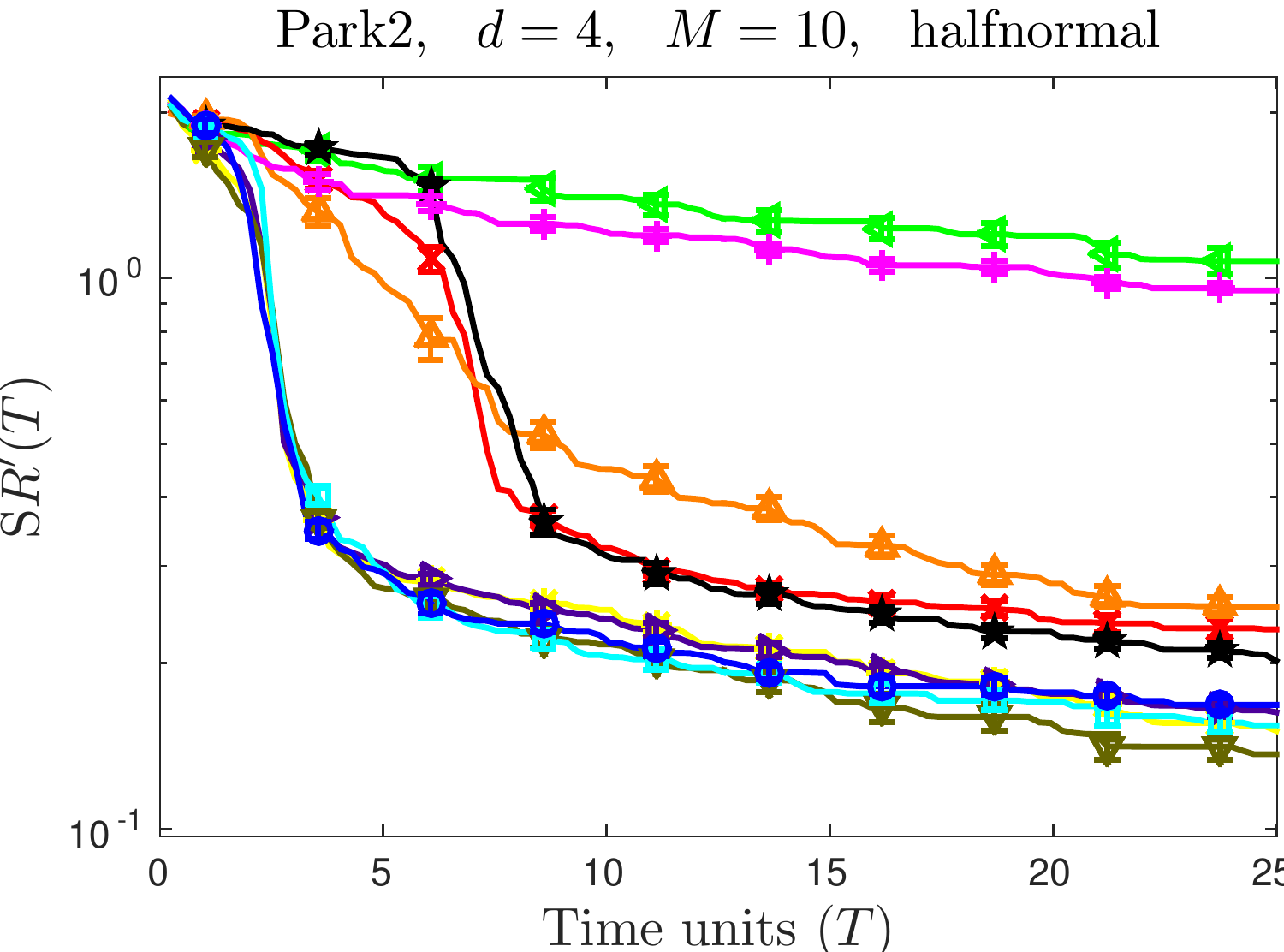}
\hspace{\imrightspace}
\\
\hspace{\imleftspace}
  \includegraphics[width=\imapparrwtwo]{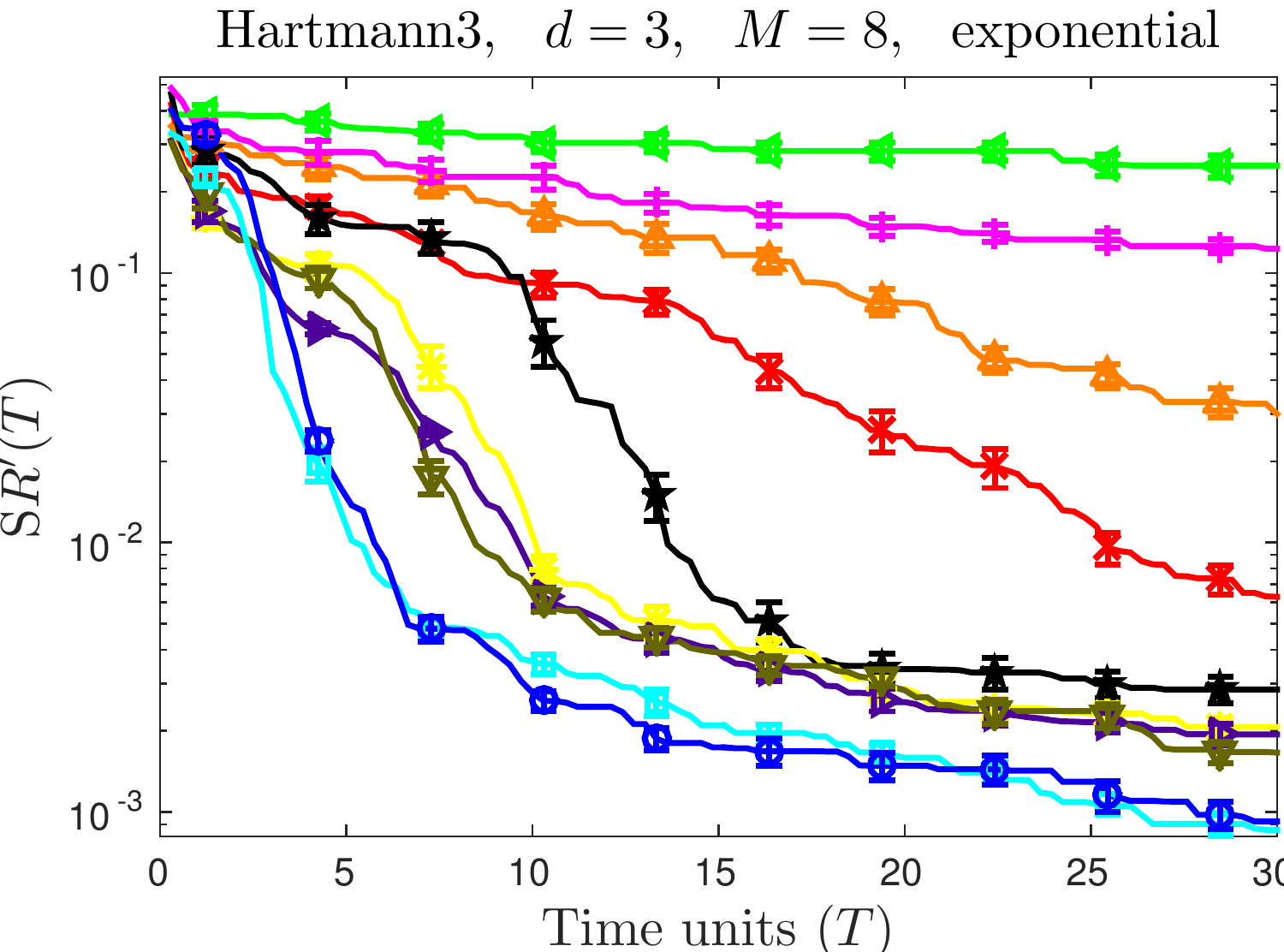} \hspace{\imapphsptwo}
  \includegraphics[width=\imapparrwtwo]{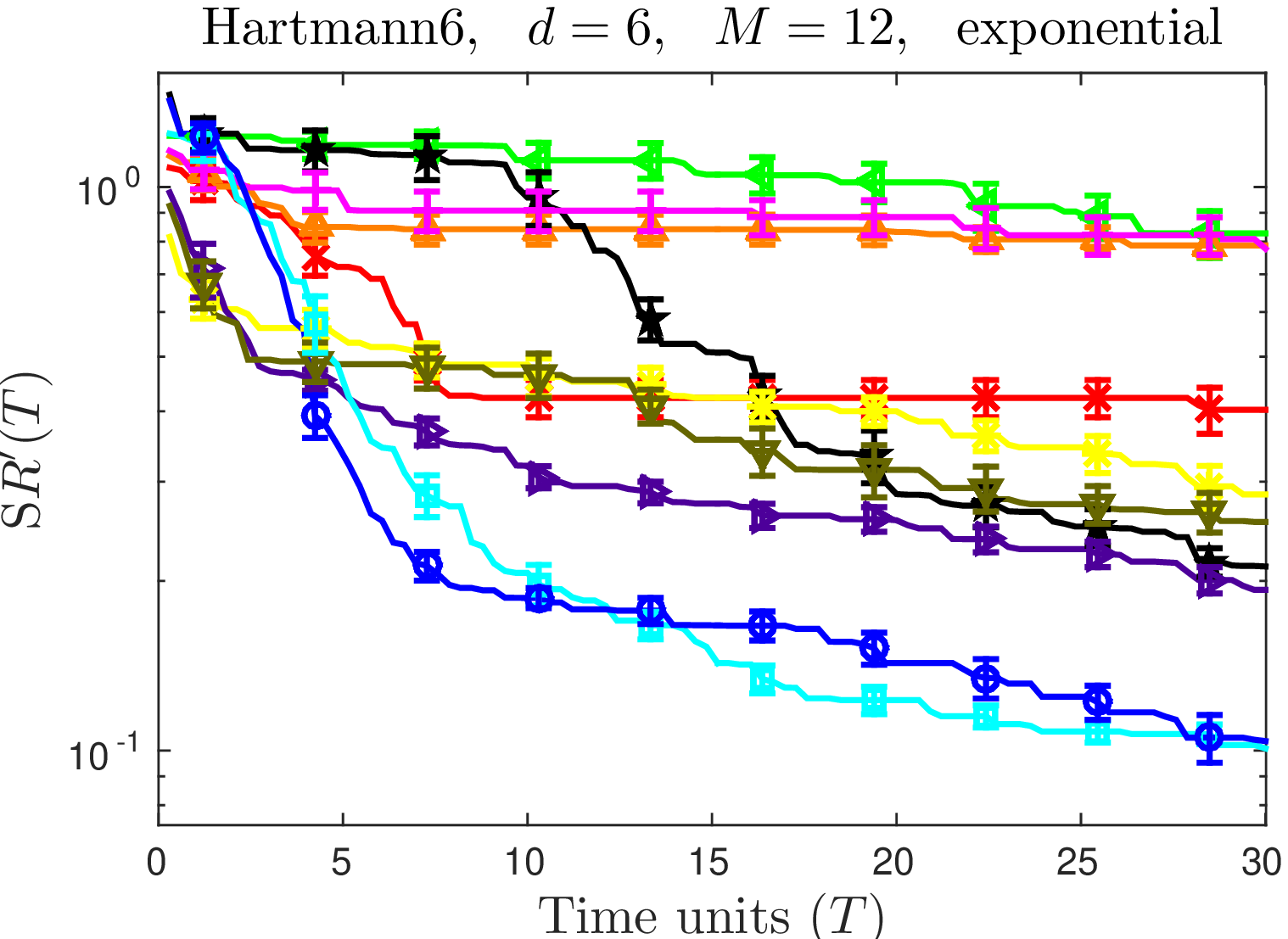}
\hspace{\imrightspace}
\\
% \vspace{-0.05in}
\caption{\small
\label{fig:toyappone}
Results on the synthetic experiments.
The title states the function used, its dimensionality $d$, the number of
workers $M$ and the distribution used for the time.
All distributions were constructed so that the expected time for one evaluation
was one time unit (for e.g., in the half normal $\HNcal(\zeta^2)$ in
Table~\ref{tb:rtv}, we used $\zeta = \sqrt{\pi/2}\;$).
All figures were averaged over at least $15$ experiments.
}
\end{figure}
}

\newcommand{\insertCifarFigResults}{
\begin{figure}
\centering
  \hspace{-0.3in}
  \includegraphics[width=2.4in]{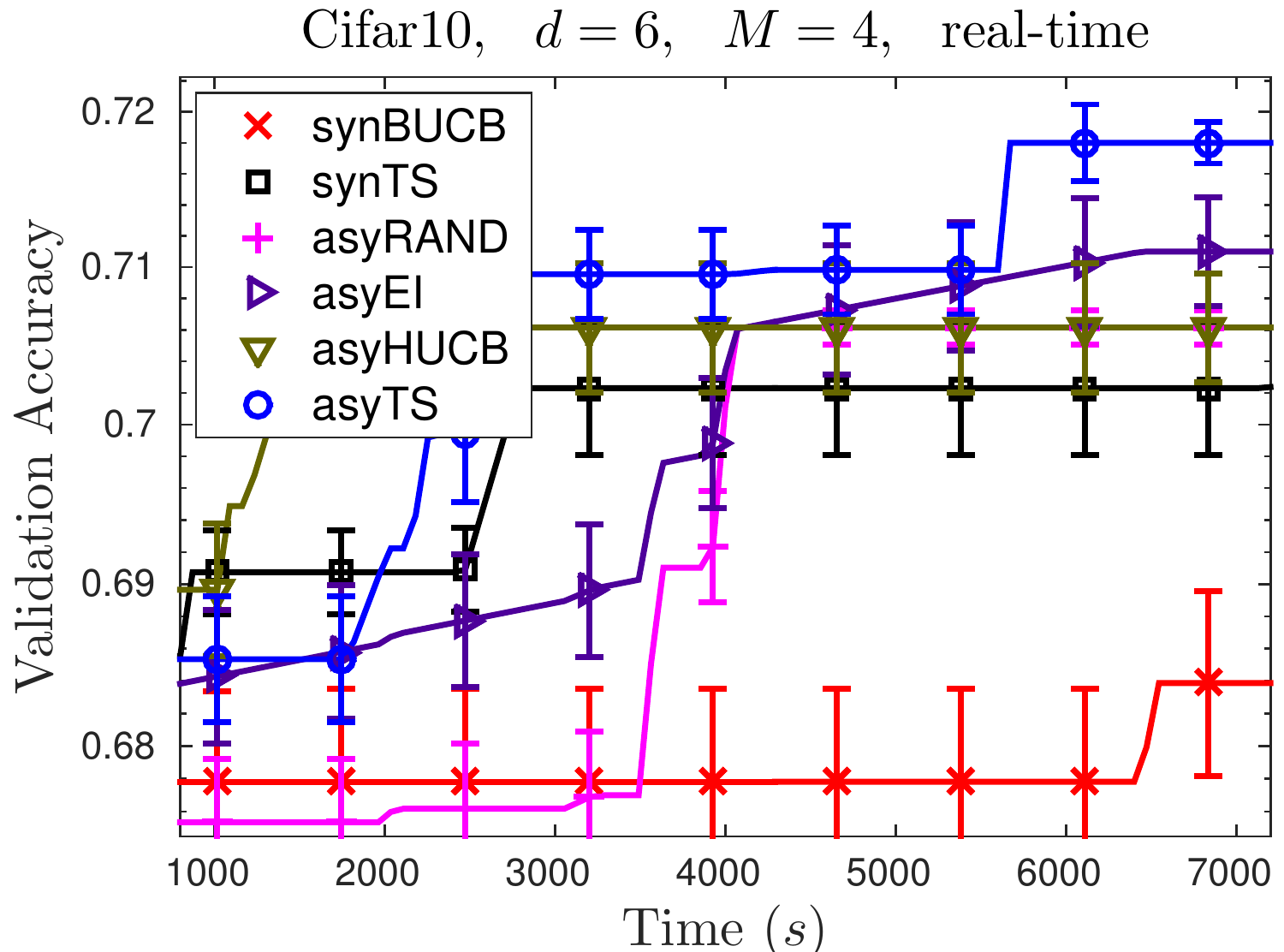}
  \hspace{0.07in}
  \begin{minipage}{3.2in}
  \centering
  \vspace{-1.7in}
\begin{tabular}{c|c|c}
\toprule
 \synbucb & \synts & \asyrand \\
\midrule
$74.37\pm 0.002$ & $77.17\pm 1.01$ & $76.07\pm 1.78$ \\
  \bottomrule \\[-0.10in]
\toprule
 \asyei & \asyhucb & \asyts \\
\midrule
$\bf 80.51\pm 0.21$ & $77.86\pm 1.12$ & $\bf 80.47\pm 0.11$ \\
\bottomrule
\end{tabular}
\caption{\small
\label{fig:cifar}
Results on the Cifar-10 experiment.
Left: The best validation set accuracy vs time for each method.
Top: Test set accuracy after training the best model chosen by each method
for $80$ epochs.
The results presented are averaged over $9$ experiments.
}
  \end{minipage}
%   \hspace{-0.1in}
  \vspace{-0.1in}
\end{figure}
}

\newcommand{\insertFigToyAppTwo}{
\begin{figure}
\centering
\hspace{\imleftspace}
  \includegraphics[width=\imapparrwtwo]{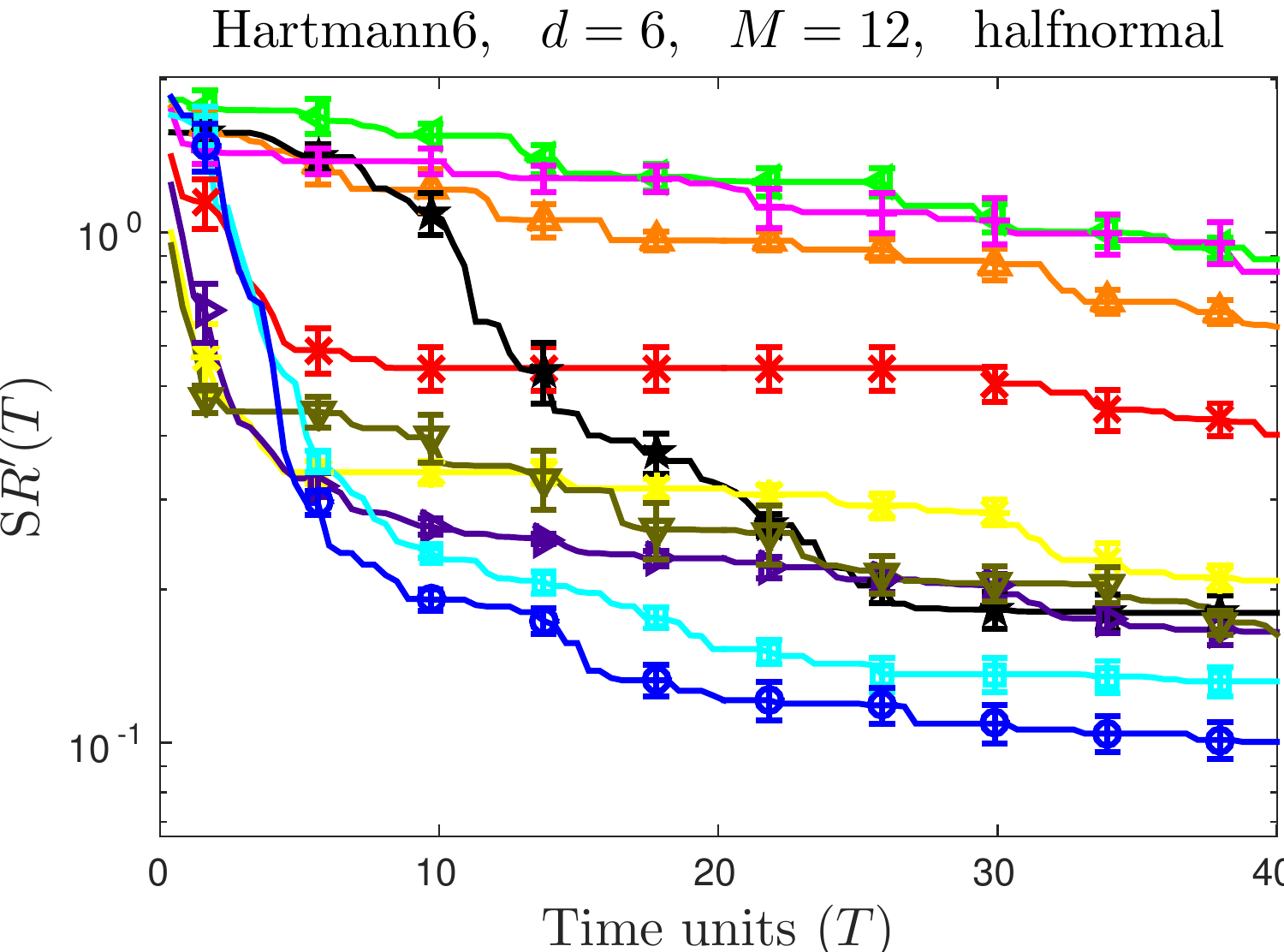} \hspace{\imapphsptwo}
  \includegraphics[width=\imapparrwtwo]{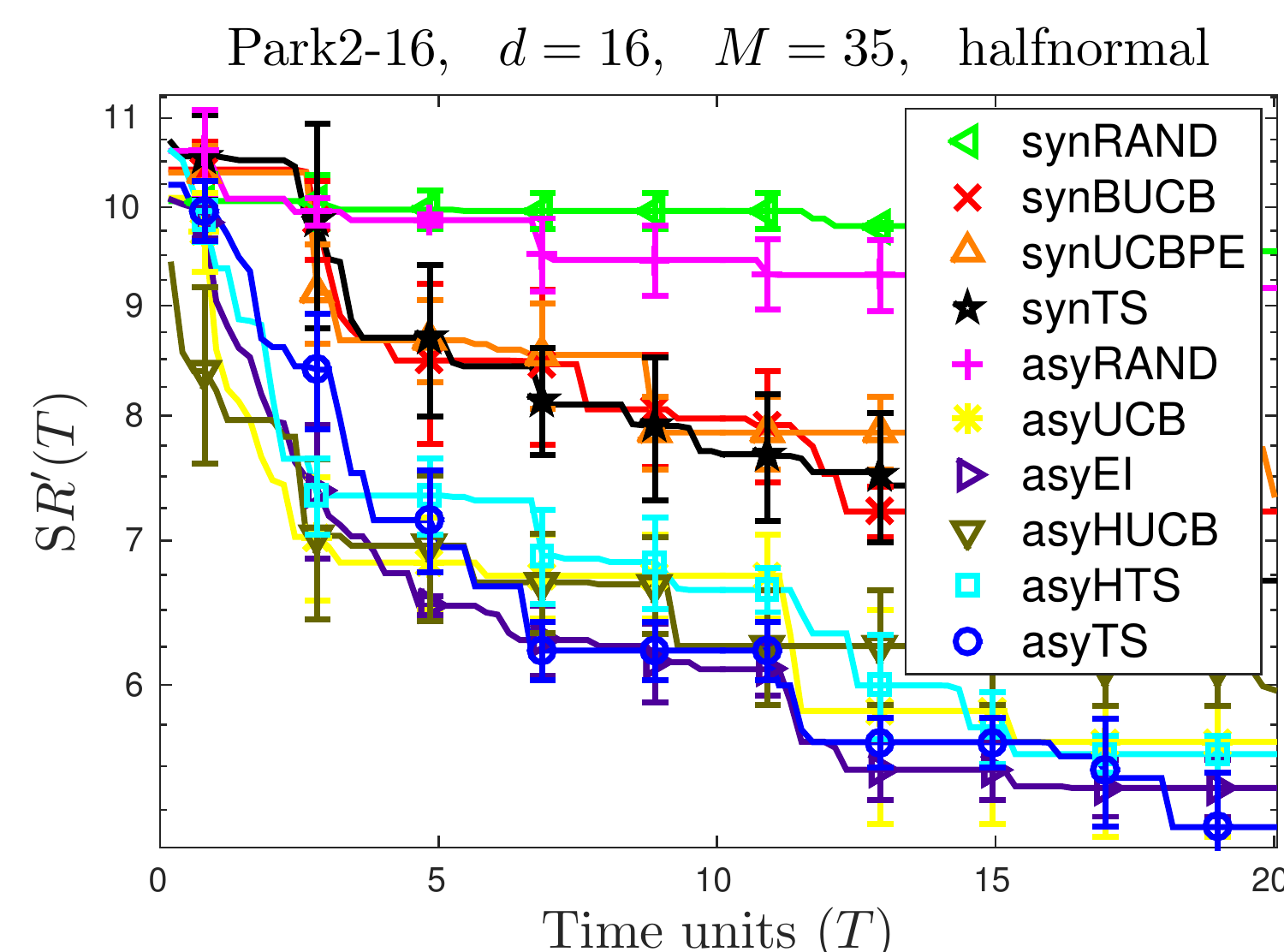}
\hspace{\imrightspace}
\\
\hspace{\imleftspace}
  \includegraphics[width=\imapparrwtwo]{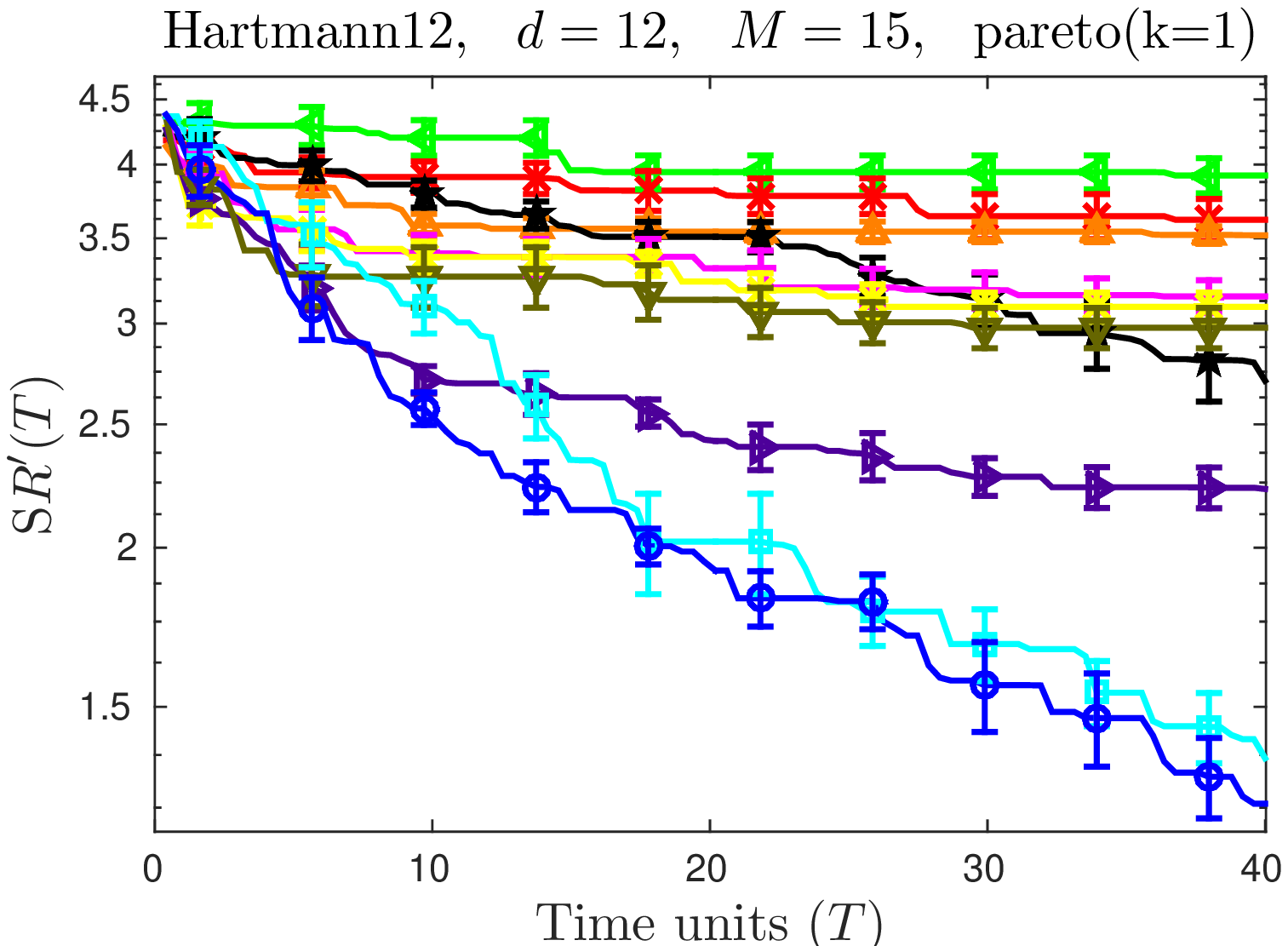} \hspace{\imapphsptwo}
  \includegraphics[width=\imapparrwtwo]{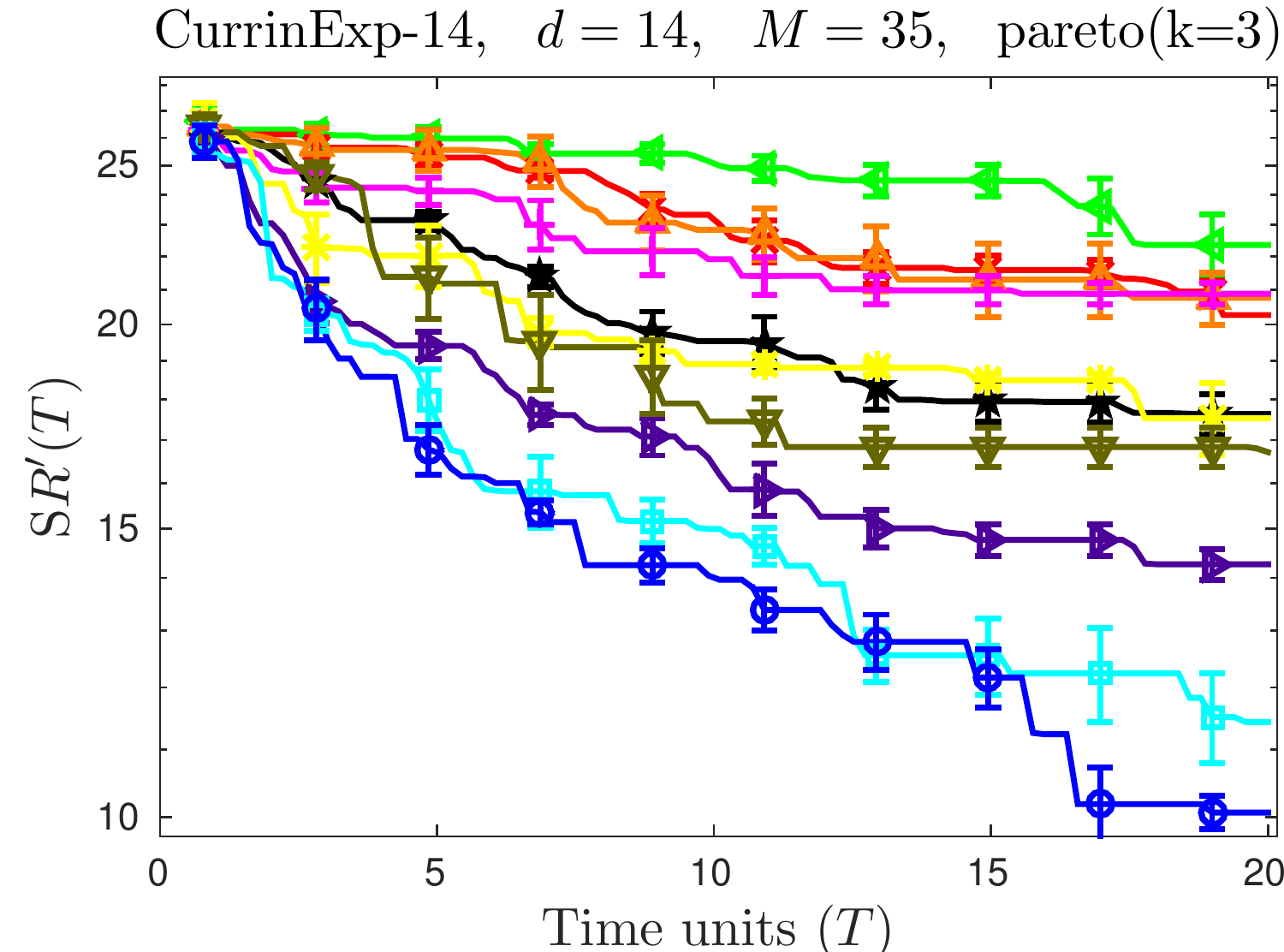}
\hspace{\imrightspace}
\\
\hspace{\imleftspace}
  \includegraphics[width=\imapparrwtwo]{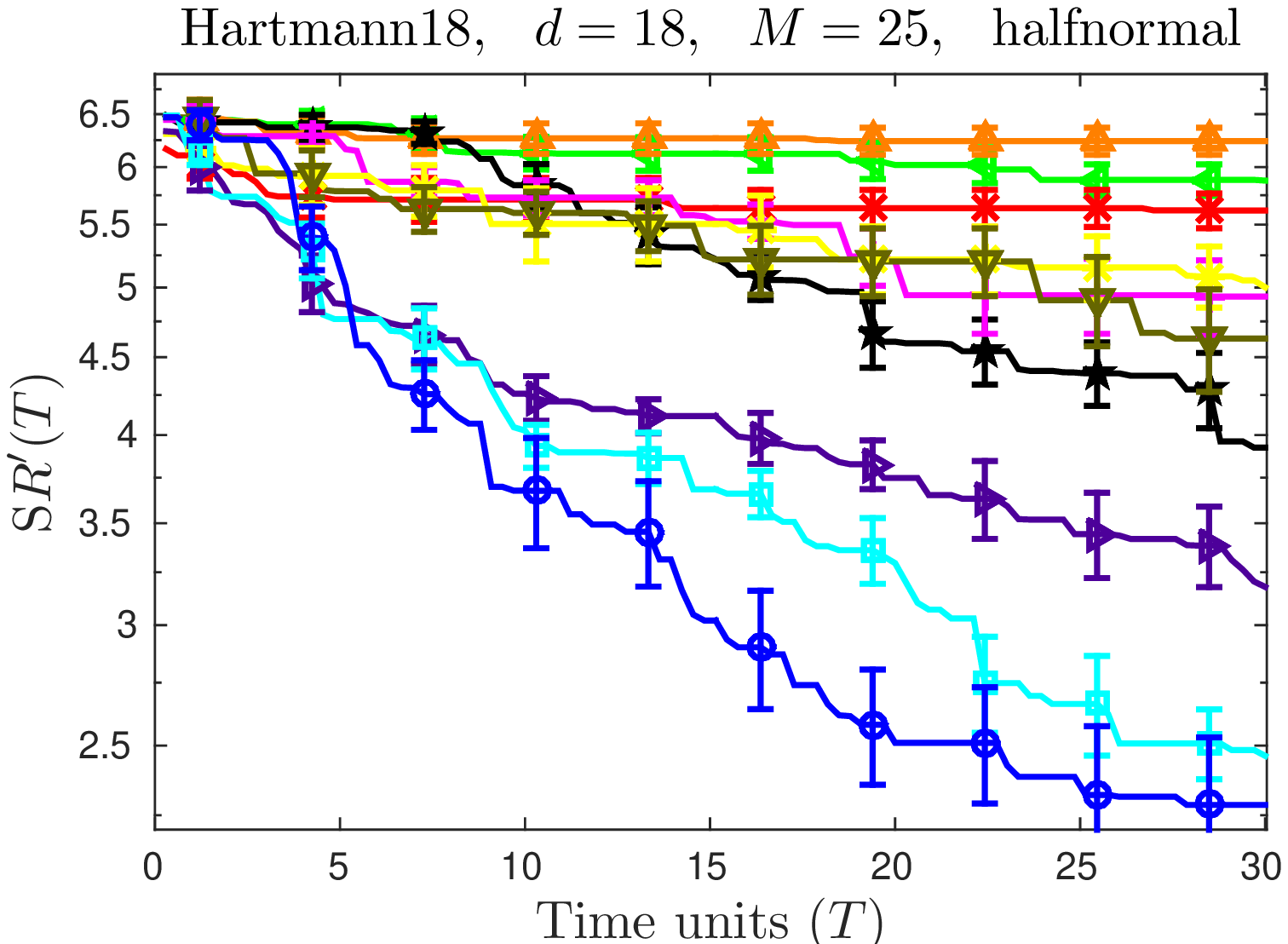} \hspace{\imapphsptwo}
  \includegraphics[width=\imapparrwtwo]{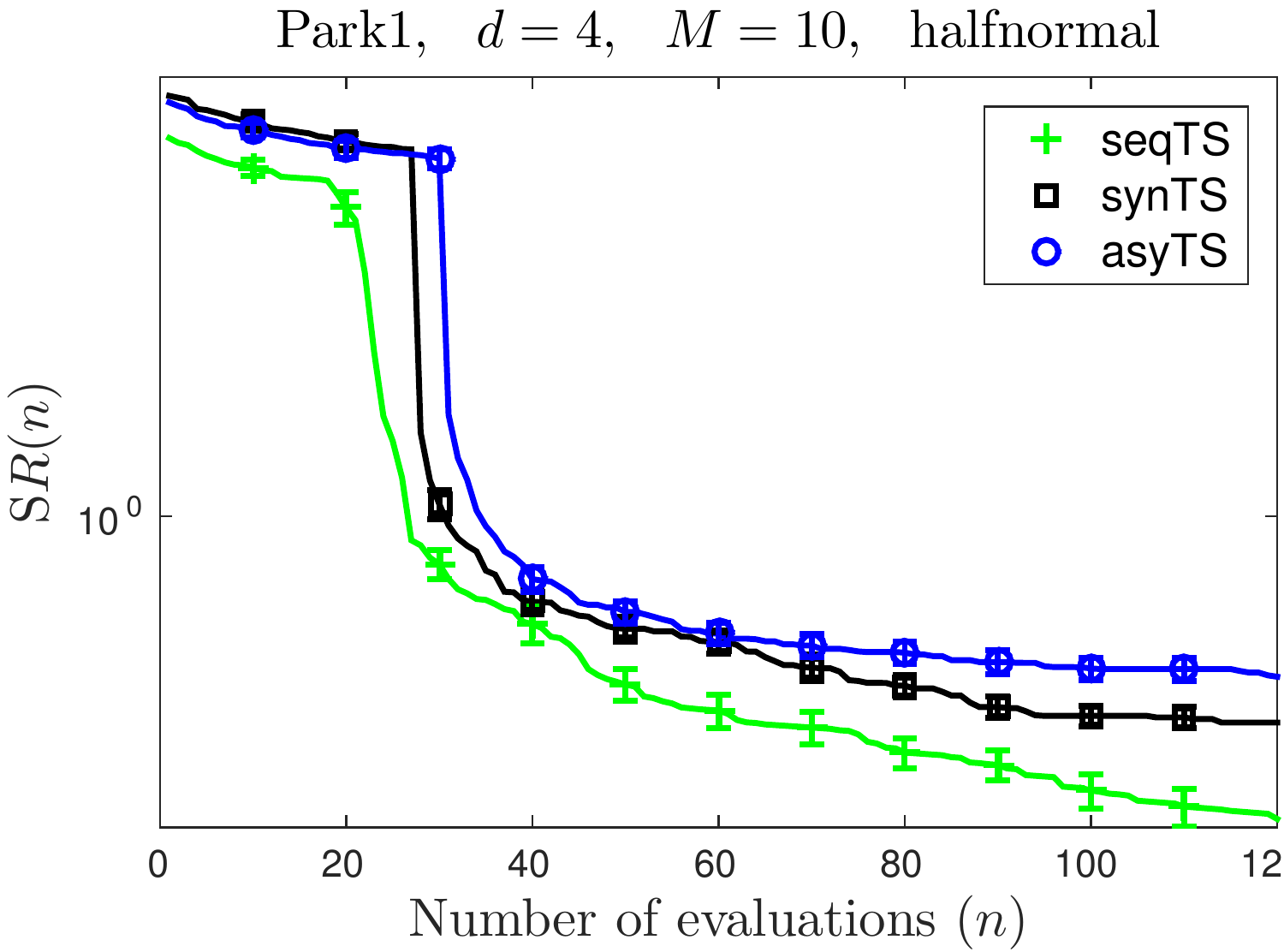}
\hspace{\imrightspace}
\\
% \vspace{-0.05in}
\caption{\small
\label{fig:toyapptwo}
The first five panels are results on synthetic experiments.
See caption under Figure~\ref{fig:toyappone} for more details.
The last panel compares \seqts, \synts, and \asytss against the number of
evaluations $n$.
}
\end{figure}
}

% All tables go here.

\newcommand{\insertAlgoasytstwo}{
\begin{algorithm}[H]
\vspace{0.02in}

\begin{algorithmic}[1]
\REQUIRE Prior GP $\;\;\GP(\zero,\kernel)$.
\STATE $\filtrjj{1} \leftarrow \emptyset$,
  $\quad\GPtt{1} \leftarrow \GP(\zero, \kernel)$.
\FOR{$j=1, 2, \dots$}
\STATE Wait for a worker to finish. \label{line:wait}
\STATE $\filtrj\leftarrow \filtrjj{j-1}\cup\{(x',y')\}$ where $(x',y')$ are the worker's 
  previous query/observation. \label{line:update_data}
\STATE Compute posterior $\GPj=\GP(\mufj,\kernelfj)$. \label{line:update_posterior}
\STATE Sample $g\sim\GPj$,  $\;\;\xj\leftarrow\argmax g(x)$. \label{line:sample}
\STATE Re-deploy worker to evaluate $\func$ at $\xj$. \label{line:redeploy}
\ENDFOR
\end{algorithmic}
\caption{$\;$\asyts \label{alg:asyts}}
\end{algorithm}
}

\newcommand{\insertAlgoasyts}{
\begin{algorithm}[H]
\vspace{0.02in}
\textbf{Input: }
Prior GP $\;\;\GP(\zero,\kernel)$.
% \vspace{-0.10in}
\begin{itemize}[leftmargin=0.17in]
\item $\filtrjj{1} \leftarrow \emptyset$,
  $\quad\GPtt{1} \leftarrow \GP(\zero, \kernel)$.
\item \textbf{for} $j=1, 2, \dots$
% \vspace{-0.13in}
\begin{enumerate}[leftmargin=0.17in,label={[\arabic*]}]
  \item Wait for free worker $m$. Obtain feedback $\yjp$ for 
        $m$'s last evaluation at $\xjp$.
%   \item $m\leftarrow$ Wait for a worker $m$. Obtain feedback $\yjp$ for 
%         $m$'s last evaluation at $\xjp$.
%   \item Update
  \item $\filtrj\leftarrow \filtrjj{j-1}\cup\{(\xjp,\yjp)\}$.
    Compute posterior $\GPj=\GP(\mufj,\kernelfj)$.
    \hfill See~\eqref{eqn:gpPost}.
  \item Sample $g\sim\GPj$.
  $\;\;\xj\leftarrow\argmax_{x\in\Xcal}g(x)$.
  \item Re-deploy $m$ with evaluation of $\func$ at $\xj$.
\end{enumerate}
% \vspace{-0.15in}
\end{itemize}
\caption{$\;$\asyts \label{alg:asyts}}
\end{algorithm}
}

\newcommand{\insertAlgotstwo}{
\begin{algorithm}[H]
\vspace{0.02in}
\begin{algorithmic}[1]
\REQUIRE Prior GP $\;\;\GP(\zero,\kernel)$.
\STATE $\filtrjj{1} \leftarrow \emptyset$, $\quad\GPtt{1} \leftarrow \GP(\zero, \kernel)$.
\FOR{ $j=1, 2, \dots$}
\STATE Sample $g\sim\GPj$.
  \STATE $\xj\leftarrow\argmax_{x\in\Xcal}g(x)$.
  \STATE $\yj\leftarrow$ Query $\func$ at $\xj$.
  \STATE $\filtrjj{j+1}\leftarrow \filtrj\cup\{(\xj,\yj)\}$.
  \STATE Compute posterior $\GPtt{j+1} = \GP(\mutt{\filtrjj{j+1}},$ $
      \kernelfj)$ conditioned on $\filtrjj{j+1}$.
    See~\eqref{eqn:gpPost}.
\ENDFOR
\end{algorithmic}
\caption{$\;$\seqts \label{alg:seqts}}
\end{algorithm}
}

\newcommand{\insertAlgots}{
\begin{algorithm}[H]
\vspace{0.02in}
\textbf{Input: }
Prior GP $\;\;\GP(\zero,\kernel)$.
% \vspace{-0.10in}
\begin{itemize}[leftmargin=0.17in]
\item $\filtrjj{1} \leftarrow \emptyset$,
  $\quad\GPtt{1} \leftarrow \GP(\zero, \kernel)$.
% \vspace{-0.05in}
\item \textbf{for} $j=1, 2, \dots$
% \vspace{-0.13in}
\begin{enumerate}[leftmargin=0.20in,label={[\arabic*]}]
  \item Sample $g\sim\GPj$.
  \item $\xj\leftarrow\argmax_{x\in\Xcal}g(x)$.
  \item $\yj\leftarrow$ Query $\func$ at $\xj$.
  \item $\filtrjj{j+1}\leftarrow \filtrj\cup\{(\xj,\yj)\}$.
  \item Compute posterior $\GPtt{j+1} = \GP(\mutt{\filtrjj{j+1}},$ $
      \kernelfj)$ conditioned on $\filtrjj{j+1}$.
    See~\eqref{eqn:gpPost}.
\end{enumerate}
% \vspace{-0.15in}
\end{itemize}
\caption{$\;$\seqts \label{alg:seqts}}
\end{algorithm}
}

\newcommand{\insertbothTSAlgos}{
\begin{table}
\begin{minipage}[t]{2.63in}
  \vspace{0pt}  
  \insertAlgotstwo
%   \begin{algorithm}[H]
%     \caption{Algo 1}
%     line 1\;
%     line 2\;
%   \end{algorithm}
\end{minipage}%
  \hspace{0.05in}
\begin{minipage}[t]{2.8in}
  \vspace{0pt}
    \insertAlgoasytstwo
\end{minipage}
  \vspace{-0.2in}
\end{table}
}

\newcommand{\insertCifarResults}{
\begin{table}
\begin{center}
\small
\begin{tabular}{c|c|c|c|c|c}
\toprule
 \synbucb & \synts & \asyrand & \asyei & \asyhts & \asyts \\
\midrule
$0.330\pm 0.002$ & $0.38\pm 0.011$ & $\bf 0.197\pm 0.001$ &
  $0.225 \pm 0.001$\\
\bottomrule
\end{tabular}
\vspace{0.05in}
\end{center}
\caption{\small
\label{tb:cifar}
The accuracy on a test set of 10K images on the Cifar-10 experiment.
We chose the best model for each method after $20$ epochs and then trained it
for 80 epochs and computed the test accuracy.
}
\vspace{-0.3in}
\end{table}
}

\newcommand{\insertbsrttable}{
\begin{table}
\begin{center}
\small
\begin{tabular}{l|c|c|c|c}
\toprule
 Distribution & pdf $\;\;p(x)$ & \seqts & \synts & \asyts \\
\midrule
 $\unif(a,b)$ & $\frac{1}{b-a}$ for $x\in(a,b)$ & $\nseq = \frac{2T}{b+a}$
  & $\nsyn=M\frac{T(M+1)}{a+bM}$  &  $\nasy = M\nseq \;\; (>\nsyn)$  \\
\midrule
 $\HNcal(\zeta^2)$ & $\frac{\sqrt{2}}{\zeta\sqrt{\pi}} e^{-\frac{x^2}{2\zeta^2}}$
                for $x>0$ & $\nseq= \frac{T\sqrt{\pi}}{\zeta\sqrt{2}}$ &
                $\nsyn \asymp \frac{M\nseq}{\sqrt{\log(M)}}$  & 
                $\nasy = M\nseq$  \\
\midrule
 $\exponential(\lambda)$ & $\lambda e^{-\lambda x}$ for $x>0$ & $\nseq = \lambda T$ &
                $\nsyn \asymp \frac{M\nseq}{\log(M)}$  &  $\nasy = M\nseq$ \\
\bottomrule
\end{tabular}
\vspace{0.05in}
\end{center}
\caption{\small
\label{tb:rtv}
The second column shows the probability density functions $p(x)$ for the uniform
$\unif(a,b)$,
half-normal $\HNcal(\zeta^2)$, and exponential $\exponential(\lambda)$ distributions.
The subsequent columns show the expected number of evaluations 
$\nseq,\nsyn,\nasy$ for \seqts, \synts, and \asytss with $M$ workers.
\syntss always completes fewer evaluations than \asyts; e.g., in the exponential
case, the difference could be a $\log(M)$ factor.
% The asynchronous version is able to achieve $M$ times as many evaluations than the
% sequential version because each worker is re-deployed immediately after it completes
% its evaluation.
% The synchronous version completes fewer evaluations as the workers are idle for some
% of the time. The difference between $\Nasy$ and $\Nsyn$ increases with $M$ and
% is more pronounced for heavier tailed distributions.
}
\vspace{-0.2in}
\end{table}
}

\begin{abstract}
We design and analyse variations of the classical Thompson sampling (TS) procedure for
Bayesian optimisation (BO) in settings where function evaluations are
expensive, but can be performed in parallel.
Our theoretical analysis shows that a direct application of the sequential Thompson
sampling algorithm in either synchronous or asynchronous parallel
settings yields a surprisingly powerful result: making $n$ evaluations
distributed among $M$ workers is essentially equivalent to performing $n$ evaluations
in sequence.
Further, by modeling the time taken to complete a function evaluation, we show
that, under a time constraint, asynchronously parallel TS achieves
asymptotically lower regret than both the synchronous and sequential versions.
These results are complemented by an experimental
analysis, showing that asynchronous TS 
outperforms a suite of existing parallel BO algorithms in simulations and
in a hyper-parameter tuning application in convolutional neural networks.
In addition to these, the proposed procedure is conceptually and computationally
much simpler than existing work for parallel BO.
\end{abstract}

% 
% 
% Bayesian optimisation is an effective tool in several black box optimisation
% applications such as hyperparameter tuning and experiment design.
% However, 
% applications, it is limited by its inherently sequential nature.
% 
% Our work builds on the classical Thompson sampling (\ts) procedure for multi-armed
% bandits.
% In this paper, we theoretically demonstrate that a direct
% application of \tss to synchronous and asynchronous parallel settings yields surprisingly
% strong bounds on the regret.
% Further, we bou

\section{Introduction}
\label{sec:intro}

Many real world problems require maximising an unknown function $\func$ from noisy
evaluations. 
Such problems arise in varied applications including hyperparameter tuning,
experiment design, online advertising, and scientific experimentation.
% As evaluations are typically very expensive in these applications, we would naturally
% like to optimise the function with a minimal number of evaluations.
As evaluations are typically expensive in such applications, we would
like to optimise the function with a minimal number of evaluations.
%% In stochastic bandit optimisation, we are tasked with optimising a noisy,
%% black-box function $\func$.
%% In typical applications, such as hyperparameter tuning, experiment design, online
%% advertising, clinical trials, and scientific experiments, an evaluation of $\func$ incurs
%% a large economic or computational cost and we need to keep the number of
%% evaluations to a minimum.
Bayesian optimisation (BO) refers to a suite of methods for black-box optimisation
under Bayesian assumptions on $\func$
that has been successfully applied in many of the above applications%
~\citep{snoek12practicalBO,hutter2011smac,martinez07robotplanning,%
parkinson06wmap3,gonzalez14gene}.

Most black-box optimisation methods, including BO, are inherently sequential in
nature, waiting for an evaluation to complete before issuing the next.
However, in many applications, we may have the opportunity to conduct
several evaluations in parallel, inspiring a surge of interest in
parallelising BO methods~\citep{ginsbourger2011dealing,janusevskis2012ei,wang2016parallel,%
gonzalez2015batch,desautels2014parallelizing,contal2013parallel,shah2015parallel,%
kathuria2016bodpps,wang2017batched,wu2016parallel}. Moreover, in these
applications, there is significant variability in the time to complete
an evaluation, and, while prior research typically studies the
relationship between optimisation performance and the number of
evaluations, we argue that, especially in the parallel setting, it is
important to account for evaluation times.
For example, consider the task of tuning the hyperparameters
of a machine learning system. This is a proto-typical example of
black-box optimisation, since we cannot analytically model the validation
error as a function of the hyperparameters and resort to noisy
train and validation procedures.
Moreover, while  training a single model
is computationally demanding, many hyperparameters can be
evaluated in parallel with modern computing infrastructure.
Further, training times are influenced by a myriad of factors, such as contention on
shared compute resources, and the
hyper-parameter choices, so they typically exhibit significant variability.

In this paper, we contribute to the line of research on parallel BO by
developing and analysing synchronous and asynchronously parallel
versions of Thompson Sampling (\ts)~\citep{thompson33sampling}, which we call \syntss and \asytss, respectively.
By modeling variability in evaluation times in our theoretical analysis,
we conclude that \asytss outperforms all existing parallel BO methods.
A key goal of this paper is to  champion this asynchronous
Thompson Sampling algorithm, due to its simplicity as well as its
strong theoretical and empirical performance.
%% setting when there is variability in evaluation times. 
 Our \textbf{main contributions} in this work are,
\vspace{-0.04in}
\begin{enumerate}[leftmargin=0.20in]
% \item A theoretical analysis for parallel \ts, demonstrating
\item A theoretical analysis demonstrating
that both \syntss and \asytss making $n$ evaluations distributed among $M$
workers is almost as good as if $n$ evaluations were made in sequence.
% as sequential \emph{TS} with $M$ times as many evaluations.
\insertitemspacing
\item By factoring time as a resource, we prove that under a time constraint, the
asynchronous version outperforms the synchronous and sequential versions.
\insertitemspacing
\item Empirically, we demonstrate that \asytss significantly outperforms
existing methods for parallel BO on several synthetic problems and a hyperparameter
tuning task.
\end{enumerate}

\vspace{-0.05in}
\subsection*{Related Work}
\vspace{-0.1in}

%% \toworkon{Check references inteh morning. @Akshay: can you also add the stuff you emailed
%% me somtime back. I haven't been able to check them.}
%% Bayesian optimisation, so named, starts with a prior belief distribution for $\func$.
%% Each time it makes an evaluation and observes the function value, it updates its
%% belief (posterior).
%% Then, it recommends a new point to evaluate $\func$ based on the posterior.
%% There are several methods to make this recommendation, 
%% some of the popular options
%% being expected improvement (EI)~\citep{jones98expensive} and upper confidence bound
%% (UCB)~\citep{srinivas10gpbandits} both of which are deterministic methods.
Bayesian optimisation methods start with a prior belief distribution for $\func$
and incorporate function evaluations into updated beliefs in the form of a posterior.
Popular algorithms choose points to evaluate $\func$ via deterministic query rules
such as expected improvement (EI)~\citep{jones98expensive} or upper confidence bounds
(UCB)~\citep{srinivas10gpbandits}.
We however, will focus on a randomised selection procedure known as Thompson sampling~\cite{thompson33sampling}, which selects a point by maximising a random sample from the posterior.
%% Each time it needs to choose a new point to evaluate $\func$,
%% \tss draws a sample from its posterior,
%% and recommends the maximiser of this sample.
% In contrast to most of the other BO criteria which are deterministic,
% \tss is a randomised strategy.
% Despite being
Some recent theoretical advances have characterised the performance of \tss in
sequential settings~\citep{russo2014learning,chowdhury2017kernelized,agrawal2012analysis,%
russo2016thompson}.

The sequential nature of BO is a serious bottleneck when scaling up to large scale
applications where parallel evaluations are possible, such as the hyperparameter tuning application.
Hence, there has been a flurry of recent activity in this area%
~\citep{ginsbourger2011dealing,janusevskis2012ei,wang2016parallel,%
gonzalez2015batch,desautels2014parallelizing,contal2013parallel,shah2015parallel,%
kathuria2016bodpps,wang2017batched,wu2016parallel}.
Due to space constraints, we will not describe each method in detail but instead
summarise the differences with our work.
In comparison to this prior work, 
%% Precisely, when compared to existing work on parallel Bayesian optimisation,
our approach enjoys  one or more of the following advantages.
\begin{enumerate}[leftmargin=0.2in]
\item \textbf{Asynchronicity:}
The majority of work on parallel BO are in the synchronous (batch) setting.
To our knowledge, only~\citep{ginsbourger2011dealing,janusevskis2012ei,wang2016parallel}
focus on asynchronous parallelisation.
\insertitemspacing
\item \textbf{Theoretical underpinnings:}
% Besides some work work using upper confidence bound techniques~\citep{kathuria2016bodpps,%
% contal2013parallel,desautels2014parallelizing}, the rest of the methods do not come
% with theoretical guarantees.
Most methods for parallel BO do not come with theoretical guarantees, with the exception
of some work using UCB techniques~\citep{kathuria2016bodpps,%
contal2013parallel,desautels2014parallelizing}.
Crucially, to the best of our knowledge, no theoretical guarantees are available for asynchronous methods. 
%% this includes all work in the asynchronous
%% setting.
\insertitemspacing
\item \textbf{Computationally and conceptually simple:}
When extending a sequential BO algorithm to the parallel setting, \emph{all} of the
above methods either introduce additional hyper-parameters and/or ancillary
computational subroutines.
Some methods become computationally prohibitive when there are a large number of
workers and must resort to
approximations~\citep{wang2016parallel,wu2016parallel,shah2015parallel,janusevskis2012ei}.
In contrast, our approach is conceptually simple -- a direct adaptation of the
sequential TS algorithm to the parallel setting.
It does not introduce any additional hyper-parameters or ancillary routines and
has the same computational complexity as sequential BO methods.
%% computationally is only as difficult as sequential BO methods.
\end{enumerate}

We mention that parallelised versions of \tss have been explored to
varying degrees in some applied domains of bandit and reinforcement
learning research~\citep{israelsen2017towards,
  hernandez2016distributed,osband2016deep}.  However, to our
knowledge, we are the first to theoretically analyse parallel \ts.
More importantly, we are also the first to propose and analyse \tss in
an asynchronous parallel setting.
Besides BO, there has been a line of work
on online learning with delayed feedback (as we have in the parallel
setting)~\citep{joulani2013online,quanrud2015online}.
In addition, ~\citet{jun2016top} study a best-arm identification problem when queries are
issued in batches.
But these papers do not address the general BO setting since they  consider finite
decision sets, nor do they model evaluation times to study trade-offs
when time is viewed as the primary resource.

% Turning away from BO methods,~\citet{jun2016top} study a best-arm
% identification problem when queries are issued in
% batches, while ~\citet{joulani2013online} provide black box
% reductions from online learning with delayed feedback (as we have in
% the parallel setting) to standard online
% learning.
% More generally, online learning
% with delayed feedback has received considerably attention in the
% literature~\cite{joulani2013online,quanrud2015online}, but these
% papers do not address the general BO setting since the consider finite
% decision sets, nor do they model evaluation times to study trade-offs
% when time is viewed as the primary resource.

\section{Preliminaries}
\label{sec:prelims}
\vspace{-0.1in}

Our goal is to maximise an unknown function $\func:\Xcal\rightarrow
\RR$ defined on a compact domain $\Xcal \subset \RR^d$, by repeatedly
obtaining noisy evaluations of $\func$: when we evaluate $\func$ at
$x\in\Xcal$, we observe $y=f(x) + \epsilon$ where the noise $\epsilon$
satisfies $\EE[\epsilon]=0$.  We work in the Bayesian paradigm,
modeling $\func$ itself as a random quantity. Following the
plurality of Bayesian optimisation literature, we assume that $\func$
is a sample from a Gaussian process~\cite{rasmussen06gps} and that the
noise, $\epsilon \sim \Ncal(0,\eta^2)$, is i.i.d normal.  A Gaussian process (GP)
is characterised by a mean function $\mu:\Xcal\rightarrow\RR$ and
prior (covariance) kernel $\kernel:\Xcal^2\rightarrow\RR$.  If
$\func\sim\GP(\mu,\kernel)$, then $\func(x)$ is distributed normally as
$\Ncal(\mu(x), \kernel(x,x))$ for all $x\in\Xcal$.  Additionally,
given $n$ observations $A=\{(x_i, y_i)\}_{i=1}^n$ from this GP, where
$x_i\in\Xcal$, $y_i = \func(x_i) + \epsilon_i \in\RR$, the posterior
process for $\func$ is also a GP with mean $\mu_A$ and covariance
$\kernel_A$ given by
\begin{align*}
\hspace{-0.05in}
% \mufn(x) = k^\top(K + \eta^2I_n)^{-1}Y, \hspace{0.35in}
% \kernelfn(x,\tilde{x}) = \kernel(x,\tilde{x}) - k^\top(K + \eta^2I_n)^{-1}\tilde{k},
\mu_A(x) = k^\top(K + \eta^2I_n)^{-1}Y, \hspace{0.35in}
\kernel_A(x,\tilde{x}) = \kernel(x,\tilde{x}) - k^\top(K + \eta^2I_n)^{-1}\tilde{k},
\numberthis \label{eqn:gpPost}
\end{align*}
where $Y\in\RR^n$ is a vector with $Y_i=y_i$, and
$k,\tilde{k}\in\RR^n$ are such that $k_i =
\kernel(x,x_i),\tilde{k}_i=\kernel(\tilde{x},x_i)$.  The Gram matrix
$K\in \RR^{n\times n}$ is given by $K_{i,j} = \kernel(x_i,x_j)$, and
$I_n\in\RR^{n\times n}$ is the identity matrix.  Some common choices
for the kernel are the squared exponential (SE) kernel and the
\matern kernel.  We refer the reader to chapter 2
of~\citet{rasmussen06gps} for more background on GPs.

Our goal is to find the maximiser $\xopt = \argmax_{x\in\Xcal}\func(x)$
of $\func$ through repeated evaluations. 
% This is equivalent to
% minimising the \emph{simple regret}, which is the difference between
% the optimal value $\func(\xopt)$ and the best evaluation of the
% algorithm.
In the BO literature, this is typically framed as 
minimising the \emph{simple regret}, which is the difference between
the optimal value $\func(\xopt)$ and the best evaluation of the
algorithm.
Since $\func$ is a random quantity, so is its optimal value 
and hence the simple regret.
This motivates studying the \emph{Bayes simple regret}, which is the
expectation of the simple regret. 
Formally, we define the simple regret,
$\srn$, and Bayes simple regret, $\bsrn$, of an algorithm after $n$ evaluations as,
\begin{align*}
\srn \,=\, \func(\xopt) - \max_{j=1,\dots,n} \func(\xj),
\hspace{0.5in}
\bsrn \,=\, \EE[\srn].
\numberthis \label{eqn:srDefn}
\end{align*}
The expectation in $\bsrn$ is with respect to the prior
$\func\sim\GP(\zero,\kernel)$, the noise in the observations
$\epsilon_j\sim\Ncal(0,\eta^2)$, and any randomness of the algorithm.
We focus on simple regret here mostly to simplify
exposition; our proof also applies for \emph{cumulative regret}, which may be
more familiar.

In many applications of BO, including hyperparameter optimisation, the
time required to evaluate the function is the dominant cost, and we
are most interested in maximising $\func$ in a short period of
time. Moreover, there is often considerable variability in the time
required for different evaluations, caused either because different
points in the domain have different evaluation costs, the randomness
of the environment, or other factors.  To adequately capture these
settings, we model the time to complete an evaluation as a random
variable, and measure performance in terms of the simple regret within
a time budget, $T$. Specifically, letting $N=N(T)$ denote the (random)
number of function evaluations performed by an algorithm within time
$T$, we define the simple regret $\srtT$ and the Bayes simple regret
$\bsrtT$ as
\begin{align*}
\srtT \,=\, \begin{cases}
\func(\xopt) - \max_{j\leq N} \func(\xj) \hspace{0.2in} & \text{if } N \geq 1 \\[0.03in]
\max_{x\in\Xcal} |\func(\xopt) - \func(x)| &\text{otherwise}
\end{cases},
\hspace{0.4in}
\bsrtT \,=\, \EE[\srtT].
\numberthis \label{eqn:srtDefn}
\end{align*}
This definition is very similar to~\eqref{eqn:srDefn}, except, when an
algorithm has not completed an evaluation yet, its simple regret is
the worst possible value.  In $\bsrtT$, the expectation now also
includes the randomness in the evaluation times in addition to the
three sources of randomness in $\bsrn$.  In this work, we will model
the evaluation time as a random variable independent from $\func$,
specifically we consider Uniform, Half-Normal, or Exponential random
variables. This model is appropriate in many applications of BO; for
example, in hyperparameter tuning, unpredictable factors such as resource
contention, initialisation, etc., may induce significant variability
in evaluation times.
%% variability in training time is
%% caused by many random factors such as resource contention,
%% initialisation etc.  
While the model does not precisely capture all
aspects of evaluation times observed in practice, we prefer it because
(a) it is fairly general, (b) it leads to a clean algorithm and
analysis, and (c) the resulting algorithm has good performance on real
applications, as we demonstrate in Section~\ref{sec:experiments}.
Studying other models for the evaluation time is an intriguing
question for future work and is discussed further in
Section~\ref{sec:conclusion}.

%% This model is appropriate in many applications of BO; for
%% example, in hyperparameter tuning, variability in training time is
%% caused by many random factors such as resource contention,
%% initialisation etc. 
%%  Of course, this does not capture all settings
%% where there might be variability in evaluation times.  Studying other
%% models for the evaluation time is an intriguing question for future
%% work and is discussed further in Section~\ref{sec:conclusion}.

% \akshay{This is appropriate for many applications of BO, including our
%   motivating example of hyperparameter tuning, since variability in
%   training time caused by resource contention, initialization,
%   etc. may be best modeled as random.}
%% \akshay{This is suitable in applications such as online advertising where we might have to
%% wait for a random amount of time for a user to respond after displaying an ad.}
% Of course, this does not capture all settings where there might be variability in
% evaluation times.
% Studying other models for the evaluation time is an intriguing
% question for future work and is discussed further in
% Section~\ref{sec:conclusion}. 

To our knowledge, all prior theoretical work for parallel
BO~\citep{desautels2014parallelizing,contal2013parallel,kathuria2016bodpps},
measures regret in terms of the total number of evaluations,
i.e. $\srn,\bsrn$. However, explicitly modeling evaluation times and
treating time as the main resource in the definition of regret is a
better fit for applications and leads to new conclusions in the
parallel setting as our results show.

\insertFigSynAsynSchemes

\textbf{Parallel BO:} We are interested in parallel
approaches for BO, where the algorithm has access to $M$ workers that
can evaluate $\func$ at different points in parallel.
%   \footnote{We do
%   not study the communication considerations, which is an interesting
%   avenue for future work. }.
% Commenting thsi out since I don't think it is necessary.
In this setup, we wish to differentiate between the synchronous and
asynchronous settings, illustrated in
Fig.~\ref{fig:parallelschemes}. In the former, the algorithm issues a
batch of $M$ queries simultaneously, one per worker, and waits for all
$M$ evaluations to be completed before issuing the next batch.  In contrast, in the
asynchronous setting, a new evaluation may be issued as soon as a
worker finishes its last job and becomes available.  In the parallel
setting, $N$ in~\eqref{eqn:srtDefn} will refer to the number of
evaluations completed by \emph{all} $M$ workers.  

Due to variability in evaluation times, worker utilisation is lower in
the synchronous setting than in the asynchronous setting, since, in
each batch, some workers may wait idly for others to finish.
However, when issuing queries, a synchronous algorithm has
  more information about $\func$, since all previous evaluations
  complete before a batch is selected, whereas asynchronous algorithms
  always issue queries with $M-1$ missing evaluations.
For example, in Fig.~\ref{fig:parallelschemes}, when dispatching the fourth job,
the synchronous version uses results from the first three evaluations whereas
the asynchronous version is only using the result of the first evaluation.
Foreshadowing our theoretical results, resource utilisation is more important
than information assimilation, and hence the asynchronous setting will
enable better bounds on $\bsrtT$.
Next, we present our algorithms.

\section{Thompson Sampling for Parallel Bayesian Optimisation}
\label{sec:method}

\insertbothTSAlgos

\textbf{A review of sequential \ts:}
Thompson sampling~\citep{thompson33sampling} is a randomised strategy for sequential
decision making.
At step $j$, \tss samples $\xj$ according to the posterior probability that it
is the optimum. 
%% stipulates we choose $\xj$ at step $j$ according to the posterior probability that
%% $\xj$ is optimal.
That is, $\xj$ is drawn  from the posterior density 
$p_{\xopt}(\cdot|\filtrj)$  where $\filtrj=\{(x_i,y_i)\}_{i=1}^{j-1}$ is the history
of query-observation pairs
up to step $j$. %, \ts samples $x_j$ from precisely this density. 
%% for any $x\in\Xcal$ we pick $\xj=x$ with probability
%% $p(\xj=x|\filtrj) = p(\xopt=x|\filtrj)$, where $\filtrj=\{(x_i,y_i)\}_{i=1}^{j-1}$
%% is the history of query-observation pairs. \akshay{There is a measure-theoretic issue here since the action space is continuous. Maybe we can rephrase to avoid this?}
For GPs, this allows for a very simple and elegant algorithm.  Observe
that we can write $p_{\xopt}(x|\filtrj) = \int p_{\xopt}(x|g)
\,p(g|\filtrj)\ud g$, and that $p_{\xopt}(\cdot|g)$ puts all its mass at
the maximiser $\argmax_{x}g(x)$ of $g$.  Therefore, at step $j$, we
draw a sample $g$ from the posterior for $\func$ conditioned on
$\filtrj$ and set $\xj = \argmax_x g(x)$ to be the maximiser of $g$.
We then evaluate $\func$ at $\xj$.  The resulting procedure, called
\seqts, is displayed in Algorithm~\ref{alg:seqts}. 
% \akshay{Samy you
%   have to fix measurability issues here. What quantities are densities
%   and what are a distributions?}

\textbf{Asynchronous Parallel \ts:}
For the asynchronously parallel setting, 
we propose a direct application of the above idea.
Precisely, when a worker finishes an evaluation, 
we update the posterior with the query-feedback pair, sample $g$ from the updated posterior,
and re-deploy the worker with an evaluation at $\xj=\argmax_{x}g(x)$.
We call the procedure~\asyts, displayed in Algorithm~\ref{alg:asyts}.
In the first $M$ steps, when at least one of the workers have not been assigned a
job yet,
the algorithm skips lines~\ref{line:wait}--\ref{line:update_posterior} and samples $g$ from the prior GP, $\GPtt{1}$, in line~\ref{line:sample}.

\textbf{Synchronous Parallel \ts:}
%% One of the key goals of this paper is to champion the asynchronous setting for parallel
%% BO over synchronous settings.
%% While our emphasis is on the asynchronous version of \ts, 
To
illustrate comparisons, we also introduce a synchronous parallel
version, \synts, which makes the following changes to
Algorithm~\ref{alg:asyts}.  In line~\ref{line:wait} we wait for all
$M$ workers to finish and compute the GP posterior with all $M$
evaluations in
lines~\ref{line:update_data}--\ref{line:update_posterior}.  In
line~\ref{line:sample} we draw $M$ samples and re-deploy all workers
with evaluations at their maxima in line~\ref{line:redeploy}.

We emphasize that \asytss and \syntss are conceptually simple and
computationally efficient, since they are essentially the same as their
sequential counterpart.  This is in contrast to existing work on
parallel BO discussed above which require additional hyperparameters
and/or potentially expensive computational routines to avoid redundant
function evaluations.
While encouraging ``diversity" of query points seems necessary to prevent
deterministic strategies such as UCB/EI from picking the same or
similar points for all $M$ workers, our main intuition is that the
inherent randomness of TS suffices to address the
exploration-exploitation trade-off when managing $M$ workers in
parallel. 
Hence, such diversity schemes are not necessary for parallel \ts.
We further demonstrate this empirically by constructing a variant \asyhtss of \asytss which
employs  one such diversity scheme found in the literature.
\asyhtss performs either about the same as or slightly worse than \asytss on many
problems we consider.
While we focus on GP priors for $\func$ in this exposition,
\tss applies to more complex models, such as neural networks.
That we can ignore the points currently in evaluation in \tss is useful
in such models, as it can lead
to efficient and distributed implementations~\citep{hernandez2016distributed}.

\subsection{Theoretical Results}
\label{sec:analysis}

We now present our theoretical contributions.  Our analysis is based
on~\citet{russo2014learning} and~\citet{srinivas10gpbandits}, and also
uses some techniques from~\citet{desautels2014parallelizing}.
% In order to convey the main intuitions we will only provide simplified versions
% of our theorems.
We provide informal theorem statements here to convey the main
intuitions, with all formal statements and proofs deferred to
Appendices~\ref{app:ts} and~\ref{app:rtv}.
We use $\asymp,\lesssim$ to
denote equality/inequality up to constant factors.

% \textbf{Maximum Information Gain:}
% To present our results we will need to introduce the \emph{Maximum Information Gain}
% (MIG)~\citep{srinivas10gpbandits}.
% The MIG is determined by the kernel $\kernel$ and measures the maximum information
% a set of $n$ observations have about $\func$; we denote it by $\IGn$.
% We define it more formally in Appendix~\ref{app:ts}, but for the current exposition it is
% sufficient to treat it as a quantity which measures the statistical difficulty of
% GP bandits.%
% ~\citet{srinivas10gpbandits} show that $\IGn$ is sublinear in $n$ for different classes
% of kernels.
% For e.g. for the SE kernel, $\IGn\propto \log(n)^{d+1}$ and for the \matern
% kernel with parameter $\nu$, $\IGn \propto n^{1 - \frac{\nu}{2\nu + d(d+1)}}$.

\textbf{Maximum Information Gain (MIG):} As in prior work, our regret
bounds involve the MIG~\citep{srinivas10gpbandits}, which captures the
statistical difficulty of the BO problem. It quantifies the maximum
information a set of $n$ observations provide about $\func$. To define
the MIG, and for subsequent convenience, we introduce one notational
convention. For a finite subset $A \subset \Xcal$, we use $y_A =
\{(x,f(x)+\epsilon) \mid x \in A\}$ to denote the query-observation
pairs corresponding to the set $A$.  The MIG is then defined as $\IGn
= \max_{A \subset \Xcal, |A| = n} I(f; y_A)$ where $I$ denotes the
Shannon Mutual Information.
%%  We define
%% it more formally in Appendix~\ref{app:ts}, but for the current
%% exposition it is sufficient to treat it as a measure of the
%% statistical difficulty of GP bandits.% 
~\citet{srinivas10gpbandits} show that $\IGn$ is sublinear in $n$ for
different classes of kernels; e.g. for the SE kernel, $\IGn\propto
\log(n)^{d+1}$ and for the \matern kernel with smoothness parameter $\nu$, $\IGn
\propto n^{1 - \frac{\nu}{2\nu + d(d+1)}}$.

% First we will bound the simple regret for Thompson sampling in the sequential,
% synchronously parallel and asynchronously parallel settings after $n$ completed
% evaluations.
Our first result bounds the Bayes simple regret $\bsrn$ for \seqts,
\synts, and \asytss purely in terms of the number of completed evaluations $n$.  In this
comparison, parallel algorithms are naturally at a disadvantage:
the sequential algorithm makes use of feedback from all its previous evaluations when
issuing a query, whereas a parallel algorithm
could be missing up to $M-1$ of them.
% up to $M$ observations will be unavailable to parallel algorithms.
%% When comparing solely in terms of the number of evaluations,
%% parallel algorithms are naturally at a disadvantage;
%% this is because when making the $j$\ssth recommendation $\xj$, a sequential algorithm
%% makes
%% use of the feedback from the previous $j-1$ evaluations whereas a parallel algorithm
%% could be missing up to $M$ of them.
\citet{desautels2014parallelizing} showed that this difference in
available information can be quantified in terms of a bound
$\xiM$ on the information we can gain about $\func$ from the
evaluations in progress conditioned on the past evaluations.  To
define $\xiM$, assume that we have already completed $n$ evaluations
to $\func$ at the points in $\filtrn$ and that there are $q$ current
evaluations in process at points in $A_q$.  That is
$\filtrn,A_q\subset\Xcal$, $|\filtrn| = n$ and $|A_q| = q< M$.
%% Denote the evaluations at $\filtrn$ by $y_{\filtrn}\in\RR^n$ and those
%% at $A_q$ by $y_{A_q}\in\RR^q$.  
Then $\xiM>0$ satisfies,
\begin{align*}
\text{for all $n\geq 1$},  \hspace{0.3in}
\max_{A_q\subset\Xcal, |A_q| < M} \, 
  I(\func; y_{A_q} | y_{\filtrn}) \;\leq\; \frac{1}{2}\log(\xiM).
\numberthis \label{eqn:xiMbound}
\end{align*}
$\xiM$ is typically increasing with $M$.
% $\xiM$, which is typically increasing with $M$,
% bounds how much information we can gain about $\func$ from
% the evaluations in progress conditioned on the past evaluations.
% Intuitively, $\xiM$ increases with $M$ simply because we can learn more about $\func$
% with many evaluations. \toworkon{clear?}
% Below is an informal theorem statement which bounds the Bayesian simple regret
The theorem below bounds the Bayesian simple regret for Thompson
sampling after $n$ evaluations in terms of $\xiM$ and the MIG $\IGn$.
% \akshay{Does $\xiM$ also require supremum of $\Dcal_n$? Or is it the
%   data generated by the algorithm?}
%   \toworkon{Data generated by algo.}

\insertthmprespacing
\begin{theorem}[Simple regret for \ts, Informal]
\label{thm:main}
Let $\func\sim\GP(\zero,\kernel)$ and
assume that condition~\eqref{eqn:xiMbound} holds.
Then the Bayes' simple regret~\eqref{eqn:srDefn} 
for \emph{\seqts}, \emph{\synts}, and \emph{\synts}
after $n$ evaluations can be bound as,
% can be bound respectively by
% \emph{$\tsseqbound$}, \emph{$\tssynbound$}, and \emph{$\tsasybound$} which
% satisfy the following. 
\emph{
\begin{align*}
\text{\seqts:} \hspace{0.1in}
  \bsrn \lesssim \sqrt{\frac{\logn \IGn}{n}}, \hspace{0.6in}
\text{\synts, \asyts:} \hspace{0.1in}
\bsrn \lesssim \sqrt{\frac{\xiM \logn  \IGn}{n}}.
% \hspace{0.4in}
% \tsseqbound \leq \tssynbound \leq \tsasybound.
\end{align*}
}
\end{theorem}
The theorem states that purely in terms of the number of
 evaluations $n$, \seqtss is better than the parallel versions.
This is to be expected for reasons explained before;
unlike a sequential method, a parallel method could be missing feedback for up to
$M-1$ of its previous evaluations.
Similarly, \syntss will outperform \asytss when measured against the number of
evaluations $n$. 
While we have stated the same upper bound for \syntss and \asyts, it is possible to
quantify the difference between the two algorithms 
(see Appendix~\ref{app:asyproofs}); however, the dominant effect, relative to the
sequential version, is the maximum number of missing evaluations which
is $M-1$ for both algorithms.

The main difference between the sequential and parallel versions is the dependence on the parameter $\xiM$. 
%% Now let us focus on $\xiM$ which quantifies the difference in the bounds
%% of the sequential and parallel versions.
While this quantity may not always be well controlled,%
~\citet{desautels2014parallelizing} showed that with a particular
initialisation scheme, $\xiM$ can be bounded by a constant for their
UCB based algorithm.
Fortunately, we can use the same scheme to bound $\xiM$ for \ts.  We
state their result formally below.

\insertthmprespacing
\begin{proposition}[\citep{desautels2014parallelizing}]
\label{prop:initialisation}
There exists an asynchronously parallelisable initialisation
scheme requiring
at most $\bigO(M\polylog(M))$ evaluations to $\func$ such that $\xiM$ is bounded
by a constant\footnote{%
After this initialisation,~\eqref{eqn:xiMbound} should be modified so that $\filtrn$
also contains the points in the initialisation.
Also, condition~\eqref{eqn:xiMbound} has close connections to the MIG but they
are not essential for this exposition.
}.
If we execute algorithms \emph{\synts}, \emph{\asyts} after this initialisation we have
% \emph{$\tssynbound, \tsseqbound \lesssim \sqrt{\logn \IGn/n}$}.
\emph{$\bsrn \lesssim \sqrt{\logn \IGn/n}$} for both.
\end{proposition}

The initialisation scheme is an uncertainty sampling procedure designed to 
reduce the posterior variance throughout the domain $\Xcal$.
Here, we first pick the point with the largest prior GP variance,
$\xinitjj{1} = \argmax_x\kernel(x,x)$.
We then iterate $\xinitjj{j} = \argmax_{x\in\Xcal}\kernel_{j-1}(x,x)$
where $\kernel_{j-1}$ denotes the posterior kernel with the previous $j-1$ evaluations.
As the posterior variance of a GP does not depend on the observations, this scheme is
asynchronously parallelisable: simply pre-compute the evaluation points and
then deploy them in parallel. 
%% However, for reasons explained in Section~\ref{sec:implementation}, this
%% scheme may not be practical.
%% Hence, we will not delve further, but defer the interested reader to
%% Appendix~\ref{app:init}.
We believe that such an initialisation may not be necessary for
\tss but currently do not have a proof.
Despite this,
Theorem~\ref{thm:main} and Proposition~\ref{prop:initialisation} imply a very powerful
conclusion: up to multiplicative constants,
\tss with $M$ parallel workers is almost as good as the sequential version
with as many evaluations.

\insertbsrttable

Now that we have bounds on the regret as a function of the number of
evaluations, we can turn to our main theoretical results: bounds on
$\bsrtT$, the simple regret with time as the main resource.
%% Finally, we present results for the simple regret $\bsrtT$ with time
%% as a resource.  
For this, we consider three different random distribution models for
the time to complete a function evaluation: uniform, half-normal, and
exponential.  We choose these three distributions since they exhibit
three different notions of tail decay, namely bounded, sub-Gaussian,
and sub-exponential\footnote{%
While we study  uniform, half-normal and exponential,
analogous results for other distributions
  with similar tail behaviour are possible with the appropriate
  concentration inequalities. See Appendix~\ref{app:rtv}.}.
Table~\ref{tb:rtv} describes these distributions and states the
expected number of evaluations $\nseq,\nsyn,\nasy$ for \seqts, \synts,
\asytss respectively with $M$ workers in time $T$. 
Our bounds on $\bsrtT$ for Thompson sampling variants are summarised in the following
theorem.

\insertthmprespacing
\begin{theorem}[Simple regret with time for \ts, Informal]
\label{thm:tstime}
Assume the same conditions as Theorem~\ref{thm:main} and that $\xiM$ is bounded by a
constant after suitable initialisation.
Assume that the time taken for completing an evaluation is a random variable with
either a uniform, half-normal or exponential distribution and let
$\nseq,\nsyn,\nasy$ be as given in Table~\ref{tb:rtv}.
Then $\nseq \leq \nsyn \leq \nasy$ and
$\emph{$\bsrtT$}$ can be upper bounded by
the following terms for \emph{\seqts}, \emph{\synts}, and \emph{\asyts}.
{\small
\emph{
\[
% \bsrttt{T, \text{\seqts}} \lesssim \sqrt{\frac{\log(\Nseq)\IGnn{\Nseq}}{\Nseq}},
% \hspace{0.1in}
% \bsrttt{T, \text{\synts}} \lesssim \sqrt{\frac{\log(\Nsyn)\IGnn{\Nsyn}}{\Nsyn}},
% \hspace{0.1in}
% \bsrttt{T, \text{\asyts}} \lesssim \sqrt{\frac{\log(\Nasy)\IGnn{\Nasy}}{\Nasy}},
\text{\seqts:} \;\; \sqrt{\frac{\log(\nseq)\IGnn{\nseq}}{\nseq}},
\hspace{0.25in}
\text{\synts:} \;\;\sqrt{\frac{\log(\nsyn)\IGnn{\nsyn}}{\nsyn}},
\hspace{0.25in}
\text{\asyts:} \;\; \sqrt{\frac{\log(\nasy)\IGnn{\nasy}}{\nasy}}.
\]
}
}
\end{theorem}
As the above bounds are decreasing with the number of
evaluations and since $\nasy > \nsyn > \nseq$, the bound for $\bsrtT$
shows the opposite trend to $\bsrn$; \asytss is better than \syntss is
better than \seqts.
\asytss can achieve asymptotically lower simple regret than both \seqtss and
\synts, given a target time budget $T$, as it can execute $M$ times as many evaluations
as a sequential algorithm.
On the other hand,
\syntss completes fewer evaluations as workers
may stay idle some time. The difference between $\nasy$ and $\nsyn$
increases with $M$ and is more pronounced for heavier tailed
distributions. 
% This result, along with the bounds in Table~\ref{tb:rtv}, lead to the
% opposite conclusions to Theorem~\ref{thm:main}, namely that \asytss
% can achieve asymptotically lower simple regret than both \seqtss and
% \syntss, given a target budget on time $T$. In comparison with
% \seqtss, \asytss can execute $M$ times as many evaluations due to
% parallelism, and hence completes far more evaluations than \seqtss.
% \syntss also completes fewer evaluations than \asytss since workers
% may stay idle some time. The difference between $\nasy$ and $\nsyn$
% increases with $M$ and is more pronounced for heavier tailed
% distributions. 

This is our main theoretical finding: given a budget $T$ on time,
\asytss, (and perhaps more generally asynchronous BO methods) can outperform
sequential or synchronous methods.

%% \asytss is able to achieve $M$ times as many evaluations than \seqtss
%% because each worker is re-deployed immediately after it completes its
%% evaluation.  On the other hand, \syntss completes fewer evaluations as
%% workers may stay idle some of the time.  The difference between
%% $\nasy$ and $\nsyn$ increases with $M$ and is more pronounced for
%% heavier tailed distributions.

% \toworkon{Can we say something more here? An easy criticism could be why they have
% to be exactly these time distributions.}
% 
% \akshay{They don't have to be exactly these distributions, but it
%   seems exponential concentration is important since we end up taking
%   a union bound over $n$ . So heavier tailed stuff likely will not
%   work. }

%% \input{implementation}

\vspace{-0.05in}
\section{Experiments}
\label{sec:experiments}
\vspace{-0.05in}
In this section we describe results from two experiments we conducted
to evaluate Thompson Sampling algorithms for Bayesian
optimisation. The first experiment is a synthetic experiment,
comparing Thompson Sampling variants with a comprehensive suite of
parallel BO methods from the literature, under a variety of
experimental conditions. In the second experiment, we compare TS with
other BO methods on the task of optimising the hyperparameters of a
convolutional neural network trained on the CIFAR-10 dataset.

% \paragraph{Implementation details:}
\textbf{Implementation details:}
In practice, the prior used for Bayesian optimisation is a modeling
 choice, but prior empirical
 work~\citep{snoek12practicalBO,kandasamy15addBO} suggest using a data dependent
  prior by estimating the kernel using past evaluations.
  Following this recommendation, we estimate and update the prior every
 $25$ iterations via the GP marginal likelihood~\citep{rasmussen06gps}
  in our Thompson Sampling implementations.
 Next, turn to initialisation.
The initialisation scheme in Proposition~\ref{prop:initialisation} may not be realisable
in practical settings as it will require that we know the kernel $\kernel$.
Unless prior knowledge is available, developing reasonable estimates of the kernel
\emph{before} collecting any data can be problematic.
In our experiments, we replace this by simply initialising TS (and other BO methods) with
evaluations at randomly selected points.
This is fairly standard in the BO literature~\citep{snoek12practicalBO} and intuitively
has a similar effect of minimising variance throughout the domain.
 Such mismatch
 between theory and practice is not uncommon for BO; most
 theoretical analyses assume knowledge of the prior kernel $\kernel$,
 but, as explained above, in practice it is typically estimated on the fly.
% 
% Such mismatch between theory and practice is not uncommon in BO. 
% For example, most papers, including ours, assume that the prior kernel $\kernel$ is known
% in theory.
% \toworkon{Added more detail: I think we need more detail here on why we don't use the initialisation scheme.}
% 
% Estimating
%  the prior also affects our intialisation scheme, which requires
%  knowledge of the kernel $\kernel$, and we alleviate this by simply
%  initalising TS (and other BO methods) with evaluations at randomly
%  selected points.
%   Random initialisation is fairly standard in the BO
%  literature~\citep{snoek12practicalBO} and intuitively has a similar
%  effect of minimising variance throughout the domain. Such mismatch
%  between theory and practice is not uncommon for BO, as most
%  theoretical analyses assume knowledge of the prior kernel $\kernel$,
%  but in practice it is typically estimated on the fly.

\insertFigToyOne

% \paragraph{The methods:}
\textbf{The methods:}
We compare \asytss to the following.
\emph{Synchronous Methods:}
\synrand: synchronous random sampling,
\synts: synchronous \ts,
\synbucbs from~\citep{desautels2014parallelizing},
\synucbpes from~\citep{contal2013parallel}.
\emph{Aynchronous Methods:}
\asyrand: asynchronous random sampling,
\asybucb: an asynchronous version of UCB with hallucinated
observations~\citep{desautels2014parallelizing,ginsbourger2011dealing},
\asyucb: asynchronous upper confidence bound~\citep{srinivas10gpbandits},
\asyei: asynchronous expected improvement~\citep{jones98expensive},
\asyhts: asynchronous \tss with hallucinated observations to explicitly encourage
diversity. This last method is based on \asytss but bases the posterior on $\filtrj \cup
\{(x, \mufj(x))\}_{x\in\hallucfiltrj}$ in line~\ref{line:update_posterior} of
Algorithm~\ref{alg:asyts}, where $\hallucfiltrj$ are the points in evaluation by other workers at
step $j$ and $\mufj$ is the posterior mean conditioned on just
$\filtrj$; this preserves the mean of the GP, but shrinks
the variance around the points in $\hallucfiltrj$.
This method is inspired
by~\cite{desautels2014parallelizing,ginsbourger2011dealing}, who use
such hallucinations for UCB/EI-type strategies so as to discourage
picking points close to those that are already in evaluation.
\asyucbs and \asyeis directly use the sequential UCB and EI criteria, since the the asynchronous versions do not repeatedly pick the same point for all workers.
\asyhucbs adds hallucinated observations to encourage diversity and is similar
to~\citep{ginsbourger2011dealing} (who use EI instead) and can also be
interpreted as an asynchronous version
of~\citep{desautels2014parallelizing}.
While there are other methods for parallel BO, many of them are either
computationally quite expensive and/or require tuning several
hyperparameters.  Furthermore, they are not straightforward to
implement and their implementations are not publicly available.
Appendix~\ref{app:expadd} describes additional implementation details
for all BO methods.

\textbf{Synthetic Experiments:}
We first present some results on a suite of benchmarks for global
optimisation.
To better align with our theoretical analysis, we add Gaussian noise to the function value
when querying.
% For a more challenging task, and to better align with
% our theoretical analysis, we add Gaussian noise to the function value
% when querying.
This makes the problem more challenging that standard global optimisation where
evaluations are not noisy.
In our first experiment, we corroborate the claims in
Theorem~\ref{thm:main} by comparing the performance of \seqts, \synts,
and \asytss in terms of the number of evaluations $n$ on the Park1
function.  The results, displayed in the first panel of Fig.~\ref{fig:toymain}, confirm
that when comparing solely in terms of $n$, the sequential version
outperforms the parallel versions while the synchronous does
marginally better than asynchronous.

%% .  For this, in the first panel of
%% Fig.~\ref{fig:toymain}, we compare \seqts, \synts, and
%% \asytss on the Park1 function against the  number of evaluations $n$.

Next, we present results on a series of global optimisation benchmarks with
different values for the number of parallel workers $M$.
We model the evaluation ``time'' as a random variable that is drawn from either a
uniform, half-normal, exponential, or Pareto\footnote{%
A Pareto distribution with parameter $k$ has pdf which decays $p(x) \propto x^{-(k+1)}$.
}
distribution.
Each time a worker makes an evaluation, we also draw a sample from this time
distribution and maintain a queueing data structure to simulate the different
start and finish times for each evaluation.
The results are presented in Fig.~\ref{fig:toymain} where we plot the
simple regret $\srtT$ against (simulated) time $T$.

In the Park2 experiment, all asynchronous methods perform roughly the
same and outperform the synchronous methods.  On all other the other
problems, \asytss performs best. \asyhtss, which also uses
hallucinated observations, performs about the same or slightly worse
than \asytss, demonstrating that there is no need for encouraging
diversity with TS. It is especially worth noting that the improvement
of \asytss over other methods become larger as $M$ increases
(e.g. $M>20$).  We believe this ability to scale well with the number
of workers is primarily due to the simplicity of our approach.  In
Appendix~\ref{app:expadd}, we provide these results in larger figures
along with additional synthetic experiments.

\textbf{Image Classification on Cifar-10:}
We also experiment with tuning hyperparameters of a $6$ layer
convolutional neural network on an image classification task on the
Cifar-10 dataset~\citep{krizhevsky2009cifar}.  We tune the number of
filters/neurons at each layer in the range $(16, 256)$.  Here, each
function evaluation trains the model on 10K images for $20$ epochs and
computes the validation accuracy on a validation set of 10K images.
Our implementation uses Tensorflow~\citep{abadi2016tensorflow} and we
use a parallel set up of $M=4$ Titan X GPUs.  The number of filters
influences the training time which varied between $\sim4$ to $\sim16$
minutes depending on the size of the model. Note that this deviates
from our theoretical analysis which treats function evaluation times
as independent random variables, but it still introduces variability to evaluation
times and demonstrates the robustness of our approach.
%% the benefits of an asynchronous algorithm over a synchronous one.
Each method is given a budget of $2$ hours to find the best model
by optimising accuracy on a validation set.  These
evaluations are noisy since the result of each training procedure
depends on the initial parameters of the network and other
stochasticity in the training procedure.  Since the true value of this
function is unknown, we simply report the best validation accuracy
achieved by each method.  Due to the expensive nature of this
experiment we only compare $6$ of the above methods.  The results are
presented in Fig.~\ref{fig:cifar}.
\asytss performs best on the validation accuracy.
The following are ranges for the number of evaluations for each method over
$9$ experiments;
\emph{synchronous:}
\synbucb: 56 - 68,
\synts:  56 - 68. 
\emph{asynchronous:}
\asyrand: 93 - 105,
\asyei: 83 - 92,
\asyhucb: 85 - 92,
\asyts: 80 - 88.

While $20$ epochs is insufficient to completely train a model, the validation error
gives a good indication of how well the model would perform after sufficient training.
In Fig.~\ref{fig:cifar}, we also give the error on a test set of 10K images after training
the best model chosen by each algorithm to completion, i.e. for $80$ epochs.
\asytss and \asyeis are able to recover the best models which achieve an accuracy of
about $80\%$.  While this falls short of state of the art results on
Cifar-10 (for e.g.~\citep{he2016deep}), it is worth noting that we use
only a small subset of the Cifar-10 dataset and a relatively small
model.  Nonetheless, it demonstrates the superiority of our approach
over other baselines for hyperparameter tuning.

\insertCifarFigResults

\vspace{-0.10in}
\section{Conclusion}
\label{sec:conclusion}
\vspace{-0.10in}

This paper studies parallelised versions of \tss for synchronous and asynchronous
BO.
We demonstrate that the algorithms \syntss and \asytss
perform as well as their purely sequential counterpart in terms of number of evaluations.
However, when we factor time in, \asytss outperforms the other two versions.
The main advantage of the proposed methods over existing literature
is its simplicity, which enables us to scale well with a large number of workers.

%% To our best knowledge, we are the first to analyse the regret of a parallel bandit
%% method by treating time as the resource.
We close with some intriguing avenues for future research. On a
technical level, is the initialisation scheme of Proposition~\ref{prop:initialisation}
necessary for \ts?
We are also interested in more general models for
evaluation times, for example to capture correlations between the
evaluation time and the query point $\xj \in \Xcal$ that arise
practice, such as in our CNN experiment. 
One could also consider models where some workers are slower than the rest.
We look forward to pursuing these directions.
 
%% We wish to discuss some open avenues for research in this regard.
%% For our analysis we treated the evaluation time as a random variable.
%% While this setting is natural in bandit applications such as on-line advertising,
%% it is perhaps not the most general.
%% For instance, one might consider situations where some workers are slower than the rest.
%% Another interesting model that we have already alluded to is on asynchronous algorithms
%% when the evaluation time depends on the point $\xt\in\Xcal$ such as in our
%% CNN experiment.
%% In the latter setting, even a sequential algorithm might be significantly different
%% than classical methods used in BO~\citep{russo2017time}.

% \newpage
{\small
\renewcommand{\bibsection}{\section*{References\vspace{-0.1em}} }
\setlength{\bibsep}{1.1pt}
\bibliography{kky,thompson}
}
\bibliographystyle{plainnat}

\newpage

\appendix
{\Large \textbf{Appendix}}
\vspace{0.1in}

\section{Theoretical Analysis for Parallelised Thompson Sampling in GPs}
\label{app:ts}

\subsection{Some Relevant Results on GPs and GP Bandits}
\label{app:tsrelated}

We first review some related results on GPs and GP bandits.
We begin with the definition of the \emph{Maximum Information Gain} (\mig)
which characterises the statistical difficulty of GP bandits~\citep{srinivas10gpbandits}.

\insertthmprespacing
\begin{definition}[Maximum Information Gain~\citep{srinivas10gpbandits}]
Let $f\sim \GP(\zero, \kernel)$
where $\kernel:\Xcal^2\rightarrow \RR$.
Let $A = \{x_1, \dots, x_n\} \subset \Xcal$ be a finite subset.
Let $f_{A}, \epsilon_{A}\in\RR^n$ such that $(f_{A})_i=f(x_i)$ and
$(\epsilon_{A})_i\sim\Ncal(0,\eta^2)$.
Let $y_{A} = f_{A}+\epsilon_{A} \in \RR^n$.
Denote the Shannon Mutual Information by $I$.
The MIG is the maximum information we can gain about $\func$
using $n$ evaluations. That is,
\[
\IGn = \max_{A\subset \Xcal, |A| = n} I(f; y_{A}).
\]
\label{def:infGain}
\end{definition}
% \insertthmpostspacing

% The \migs was introduced by~\citep{srinivas10gpbandits} who also 
\citet{srinivas10gpbandits} and~\citet{seeger08information} provide
bounds on the \migs for different classes of kernels.
For example for the SE kernel, $\IGn\asymp \logn^{d+1}$
and for the \matern kernel with smoothness parameter $\nu$,
$\IGn \asymp n^{\frac{d(d+1)}{2\nu + d(d+1)}}\logn$.
The next theorem due to~\citet{srinivas10gpbandits} bounds the sum of variances
of a GP using the~\mig.

\insertthmprespacing
\begin{lemma}[Lemma 5.2 and 5.3 in~\citep{srinivas10gpbandits}]
\label{lem:IGBoundLemma}
Let $\func\sim\GP(0,\kernel)$, $\func:\Xcal\rightarrow\RR$ and each time we query at
any $x\in\Xcal$ we observe $y = \func(x)+\epsilon$, where
$\epsilon\sim\Ncal(0,\eta^2)$.
Let $\{x_1,\dots,x_n\}$ be an arbitrary set of $n$ evaluations to $\func$ where
$\xj\in\Xcal$ for all $j$.
Let $\sigmasqjmo$ denote the posterior variance conditioned on the first
$j-1$ of these queries, $\{x_1,\dots,x_{j-1}\}$.
Then, $\sum_{j=1}^n\sigmasqjmo(\xj)\leq \frac{2}{\log(1+\eta^{-2})}\IGn$.
\end{lemma}
\insertthmpostspacing

Next we will need the following regularity condition on the derivatives of the
GP sample paths.
When $\func\sim\GP(\zero,\kernel)$, it is satisfied when $\kernel$ is four times
differentiable, e.g. the SE kernel and \matern{} kernel when
$\nu>2$~\citep{ghosal06gpconsistency}.

\insertthmprespacing
\begin{assumption}[Gradients of GP Sample Paths~\citep{ghosal06gpconsistency}]
Let $\func\sim\GP(\zero,\kernel)$, where
$\kernel:\Xcal^2 \rightarrow \RR$ is a stationary kernel.
%  $\kernel(\cdot, x)$ is $L$-Lipschitz for all $x$. 
The partial derivatives
of $\func$  satisfies the following condition.
There exist constants $a, b >0$ such that,
\[
\text{for all $J>0$, $\;$and for all $i \in \{1,\dots,d\}$},\quad \PP\left( \sup_{x} 
\Big|\partialfrac{x_i}{\func(x)}\Big| > J \right)
\leq a e^{-(J/b)^2}.
\]
\label{asm:kernelAssumption}
\end{assumption}
% \insertthmpostspacing

Finally, we will need the following result on the supremum of a Gaussian process.
It is satisfied when $\kernel$ is twice differentiable.

\insertthmprespacing
\begin{lemma}[Supremum of a GP~\citep{adler1990gps}]
\label{lem:supgp}
Let $\func\sim\GP(0,\kernel)$ have continuous sample paths.
Then, $\EE\|\func\|_\infty = \supgp < \infty$.
\end{lemma}
This, in particular implies that in the definition of $\bsrtT$ in~\eqref{eqn:srtDefn},
$\max_{x\in\Xcal}|f(\xopt) - f(x)| \leq 2\supgp$.

Finally, we will use the following result in our parallel analysis.
Recall that the posterior variance of a GP does not depend on the observations.

\insertthmprespacing
\begin{lemma}[Lemma 1 (modified) in~\citep{desautels2014parallelizing}]
\label{lem:stdmi}
Let $\func\sim\GP(0,\kernel)$.
Let $A,B$ be finite subsets of $\Xcal$.
Let $y_A\in\RR^{|A|}$ and $y_B\in\RR^{|B|}$ denote the observations when we evaluate
$\func$ at $A$ and $B$ respectively.
Further let $\sigma_A,\, \sigma_{A\cup B}\,:\,\Xcal\rightarrow\RR$
denote the posterior standard deviation of
the GP when conditioned on $A$ and $A\cup B$ respectively.
Then,
\[
\text{for all $x\in\Xcal$}, \hspace{0.2in}
\frac{\sigma_{A}(x)}{\sigma_{A\cup B}(x)} \leq \exp\big( I(\func; y_B | y_A) \big)
\]
\end{lemma}
The proof exactly mimics the proof in~\citet{desautels2014parallelizing}.
% It says that the ratio between the posterior standard deviations $\sigma_A,\sigma_{A\cup
% B}$ can be bounded in
Lemma~\ref{lem:stdmi} implies $\sigma_{A}(x)\leq \xiM^{1/2} \sigma_{A\cup B}(x)$
where $\xiM$ is from~\eqref{eqn:xiMbound}.

\subsection{Notation \& Set up}
\label{app:tsproofs}

% \insertFigSynAsynSchemes

% \textbf{Notation and set up:}
We will require some set up in order to unify the analysis for the sequential,
synchronously parallel and asynchronously parallel settings.
\begin{itemize}[leftmargin=0.3in]
\item
The first is an indexing for the function evaluations.
This is illustrated for the synchronous and asynchronous parallel settings
in Figure~\ref{fig:parallelschemes}.
In our analysis, the index $j$ or step $j$ will refer to the $j$\ssth function evaluation
dispatched by the algorithm.
In the sequential setting this simply means that there were $j-1$ evaluations before
the $j$\ssth.
For synchronous strategies we index the first batch from $j=1,\dots,M$ and then the
next batch $j=M+1,\dots,2M$ and so on as in Figure~\ref{fig:parallelschemes}.
For the asynchronous setting, this might differ as each evaluation takes different
amounts of time. For example, in Figure~\ref{fig:parallelschemes}, the first worker
finishes the $j=1$\superscript{st} job and then starts the $j=4$\superscript{th},
while the second
worker finishes the $j=2$\superscript{nd} job and starts the $j=6$\superscript{th}.

\item Next, we define $\filtrj$ at step $j$ of the algorithm to be the query-observation
pairs $(\xk,\yk)$ for function evaluations completed by step $j$.
In the sequential setting $\filtrj = \{(\xk,\yk): k\in\{1,\dots,j-1\}\}$ for all $j$.
For the synchronous setting in Figure~\ref{fig:parallelschemes},
$\filtrjj{1}=\filtrjj{2}=\filtrjj{3} = \emptyset$,
$\filtrjj{4}=\filtrjj{5}=\filtrjj{6} = \{(\xk,\yk): k\in\{1,2,3\}\}$,
$\filtrjj{7}=\filtrjj{8}=\filtrjj{9} = \{(\xk,\yk): k\in\{1,2,3,4,5,6\}\}$ etc.
Similarly, for the asynchronous setting,
$\filtrjj{1}=\filtrjj{2}=\filtrjj{3} = \emptyset$,
$\filtrjj{4}= \{(\xk,\yk): k\in\{1\}\}$,
$\filtrjj{5}= \{(\xk,\yk): k\in\{1,3\}\}$,
$\filtrjj{6}= \{(\xk,\yk): k\in\{1,2,3\}\}$,
$\filtrjj{7}= \{(\xk,\yk): k\in\{1,2,3,5\}\}$ etc.
Note that in the asynchronous setting $|\filtrj| = j-M$ for all $j>M$.
$\{\filtrj\}_{j\geq 1}$ determines the filtration 
when constructing the posterior GP at every step $j$.

\item Finally, in all three settings, $\mu_A:\Xcal\rightarrow\RR$ and
$\sigma_A:\Xcal\rightarrow\RR_+$ will refer to the posterior
mean and standard deviation of the GP conditioned on some evaluations $A$,
i.e. $A\subset\Xcal\times\RR$ is a set of $(x,y)$ values and $|A|<\infty$.
They can be computed by plugging in the $(x,y)$ values in $A$
to~\eqref{eqn:gpPost}.
For example, $\mufj,\sigmafj$ will denote the mean and standard deviation conditioned
on the completed evaluations, $\filtrj$.
% Noting that the posterior variance of a GP does not depend on the $y$ values we
% will overload notation for $\sigma_A$ where $A\subset\Xcal$ could be just a set
% of $\{x\}$ values.
Finally, when using our indexing scheme above we will also overload notation so that
$\sigmajmo$ will denote the posterior standard deviation conditioned on evaluations
from steps $1$ to $j-1$.
That is $\sigmajmo = \sigma_A$ where $A=\{(\xk,\yk)\}_{k=1}^{j-1}$.
\end{itemize}

\subsection{Parallelised Thompson Sampling}
\label{app:asyproofs}

% In the remainder of this section, $\betan,\betan\in\RR$ for all $n\geq 1$
% will denote the following values.
% \begin{align*}
% \betan \,=\, 4(d+1)\logn + 2d\log(dab\sqrt{\pi}),
% \,\asymp\, d\logn,
% \hspace{0.4in}
% \betan = \max(1, \xiM)\cdot\betan
% \numberthis\label{eqn:betan}
% \end{align*}
In the remainder of this section, $\betan\in\RR$ for all $n\geq 1$
will denote the following value.
\begin{align*}
\betan \,=\, 4(d+1)\logn + 2d\log(dab\sqrt{\pi})
\;\;\asymp\; d\logn,
% \hspace{0.4in}
% \betan = \max(1, \xiM)\cdot\betan
\numberthis\label{eqn:betan}
\end{align*}
Here $d$ is the dimension, $a,b$ are from
Assumption~\ref{asm:kernelAssumption}, and $n$ will denote the number of evaluations.
Our first theorem below is a bound on the simple regret for~\syntss and~\asytss
after $n$ completed evaluations.

\insertthmprespacing
\begin{theorem}
\label{thm:tsparallelmain}
Let $\func\sim\GP(\zero,\kernel)$ where $\kernel:\Xcal^2\rightarrow\RR$ satisfies
Assumption~\ref{asm:kernelAssumption}.
Further,  without loss of generality $\kernel(x,x')\leq 1$.
Then for \emph{\synts} and \emph{\asyts}, the Bayes simple regret after $n$ evaluations
satisfies,
\emph{
\[
\bsrn\;\leq\; \frac{C_1}{n} + \sqrt{\frac{C_2\xiM\betan\IGn}{n}},
\]
}
where $\IGn$ is the MIG in Definition~\ref{def:infGain}, $\betan$ is 
as defined in~\eqref{eqn:betan},
$\xiM$ is from~\eqref{eqn:xiMbound}, and $C_1=\pi^2/6 + \sqrt{2\pi}/12$,
$\;C_2=2/\log(1+\eta^{-2})$ are constants.
\end{theorem}
\insertthmpostspacing

\begin{proof}%[\textbf{Proof of Theorem~\ref{thm:tsmain}}]
Our proof is based on techniques from~\citet{russo2014learning} and%
~\citet{srinivas10gpbandits}.
We will first assume that the $n$ evaluations completed are the the evaluations
indexed $j=1,\dots,n$.

As part of analysis, we will discretise $\Xcal$ at each step $j$ of the algorithm.
Our discretisation $\nu_j$, is obtained via a grid of $\tau_j=j^2dab\sqrt{\pi}$
equally spaced points along each coordinate and has size $|\nu_j| = \tau_j^d$.
It is easy to verify that 
$\nu_j$ satisfies the following property: for all $x\in\Xcal$, $\|x-\xnuj\|_1 \leq
d/\tau_j$, where $\xnuj$ is the closest point to $x$ in $\nu_j$.
This discretisation is deterministically constructed ahead of time and does not depend on
any of the random quantities in the problem.
% Let 
% \[
% \Ecal=\left\{\forall i=1,\dots,d, \forall x\in\Xcal,\;
% \Big|\frac{\partial f}{\partial x_i}\Big| \leq J \right\}
% \]
% denote the event that all partial derivates are smaller than $J$.
% By Assumption~\ref{asm:kernelAssumption} we know that $\Ecal$ is true with probability
% greater than $\delta$.
% Further, when $\Ecal$ is true,
% \begin{align*}
% |\func(x) - \func(\xnuj)| \leq J \|x-\xnuj\|_1 \leq \frac{Jd}{\tau_j} = \frac{1}{j^2}.
% \label{eqn:funcLipschitzBound}
% \end{align*}
% % $|\nu_t| = \tau_t^d$ where 

For the purposes of our analysis, we define the Bayes cumulative regret after $n$
evaluations as,
\begin{align*}
\bcrn = \EE\bigg[\sum_{j=1}^n \func(\xopt) - \func(\xj) \bigg].
\end{align*}
Here, just as in~\eqref{eqn:srDefn}, the expectation is with respect to the randomness
in the prior, observations and algorithm.
Since the average is larger than the minimum, we have
$\frac{1}{n}\sum_j\func(\xopt) - \func(\xj) \geq
\min_j (\func(\xopt) - \func(\xj)) = \srn$;
hence $\bsrn \leq \frac{1}{n}\bcrn$.
% Therefore it is sufficient to prove
% $\bcrn\;\leq\; C_1 + \sqrt{C_2n\beta_n\IGn}$.

Following~\citet{russo2014learning} we denote
$\Uj(\cdot) = \mufj(\cdot) + \betajh\sigmafj(\cdot)$ to be an upper confidence bound
for $\func$ and begin by decomposing $\bcrn$ as follows,
\begingroup
\allowdisplaybreaks
\begin{align*}
\bcrn\;
&=\; \sum_{j=1}^n \EE \left[\func(\xopt) - \func(\xj)\right]
 \stackrel{(a)}{=}\;
  \sum_{j=1}^n \EE\big[ \EE\left[\func(\xopt) - \func(\xj)\,|\filtrj\right]\,\big] \\
&\stackrel{(b)}{=}\;\sum_{j=1}^n  \EE\big[ \EE[\func(\xopt)  - \func(\xj)
    - \func(\xoptnuj) + \func(\xoptnuj) 
-\Uj(\xoptnuj) + \Uj(\xjnuj) \,-\\
  &\hspace{3in} \func(\xjnuj) + \func(\xjnuj) -   \func(\xj)\,|\filtrj]\,\big] \\
&\stackrel{(c)}{=}\;
     \underbrace{\sum_{j=1}^n \EE[ \func(\xopt) - \func(\xoptnuj)]}_{A_1} \;+\;
     \underbrace{\sum_{j=1}^n \EE[ \func(\xjnuj) - \func(\xj)]}_{A_2} \;+\;
     \underbrace{\sum_{j=1}^n \EE[ \func(\xoptnuj) - \Uj(\xoptnuj)]}_{A_3} \\
  &\hspace{2in} \;+\;
     \underbrace{\sum_{j=1}^n \EE[ \Uj(\xjnuj) - \func(\xjnuj)]}_{A_4}.
\end{align*}
\endgroup
In $(a)$ we have used the tower property of expectation and in $(c)$ we have
simply rearranged the terms from the previous step.
In $(b)$ we have added and subtracted $\func(\xoptnuj)$ and then
$\func(\xjnuj)$.
The crucial step in $(b)$ is that we have also added $-\Uj(\xoptnuj) + \Uj(\xjnuj)$
which is justified if $\EE[\Uj(\xoptnuj)|\filtrj] = \EE[\Uj(\xjnuj)|\filtrj]$.
For this, first note that as $\xj$ is sampled from the posterior distribution
for $\xopt$ conditioned on $\filtrj$, both $\xj|\filtrj$ and $\xopt|\filtrj$ have the same
distribution.
Since the discretisation $\nu_j$ is fixed ahead of time,
and $\Uj$ is deterministic conditioned on $\filtrj$,
so do $\Uj(\xjnuj)|\filtrj$ and $\Uj(\xoptnuj)|\filtrj$.
Therefore, both quantities are also equal in expectation.
% Therefore,  $\EE[\Uj(\xoptnuj)|\filtrj] = \EE[\Uj(\xjnuj)|\filtrj]$.
% Further, as $\Uj$ is deterministic conditioned on $\filtrj$
% Therefore, for $\Uj(\xoptnuj)$ and  $\Uj(\xjnuj)$ to be equal in distribution
% it is sufficient if $\Uj:\Xcal\rightarrow \RR$ is deterministic once you condition on
% $\filtrj$.
% We will construct $\Uj$ accordingly.

% Therefore it is sufficient if $\Ujp,\Uj:\Xcal\rightarrow\RR$ be deterministic when
% you condition on $\filtrj$ and satisfy $\Ujp \leq \Uj$.
% We set,
% \[
% \Ujp(\cdot) = \mufj(\cdot) + \betajh\sigmafj(\cdot), \hspace{0.2in}
% \Uj(\cdot) = \mufj(\cdot) + \betajh\sigmafj(\cdot)
% \]
% to be upper confidence bounds for $\func$.
% They satisfy the two conditions as $\sigmafj \geq 0$ and $\betaj\geq \betaj$.

To bound $A_1, A_2$ and $A_3$ we use the following Lemmas.
The proofs are in Sections~\ref{app:discboundproof} and~\ref{app:discucbboundproof}.

% \insertthmprespacing
\begin{lemma}
\label{lem:discbound}
At step $j$, for all $x\in\Xcal$, $\EE[|\func(x) - \func(\xnuj)|] \leq \frac{1}{2j^2}$.
\end{lemma}
% \insertthmpostspacing
\begin{lemma}
\label{lem:discucbbound}
At step $j$, for all $x \in \nu_j$, 
$\EE[\indfone\{\func(x)>\Uj(x)\}\cdot(\func(x) - \Uj(x))]
  \leq \frac{1}{j^2 \sqrt{2\pi}|\nu_j|}$.
% At step $j$, for all $x\in\Xcal$, $\EE[|\func(x) - \func(\xj)|] \leq \frac{1}{2j^2}$.
\end{lemma}
% \insertthmpostspacing

Using Lemma~\ref{lem:discbound} and the fact that $\sum_j j^{-2} = \pi^2/6$, we
have $A_1 + A_2 \leq \pi^2/6$.
We bound $A_3$ via,
\begingroup
\allowdisplaybreaks
\begin{align*}
A_3\;
&\leq\; \EE\bigg[\sum_{j=1}^n \indfone\{\func(\xoptnuj) > \Uj(\xoptnuj)\}\cdot
  (\func(\xoptnuj) - \Uj(\xoptnuj)) \bigg] \\
&\leq\; \sum_{j=1}^n \sum_{x\in\nu_j}
    \EE\big[\,\indfone\{\func(x) > \Uj(x)\}\cdot (\func(x) - \Uj(x)) \,\big] 
\leq\; \sum_{j=1}^n \sum_{x\in\nu_j} \frac{1}{j^2\sqrt{2\pi}|\nu_j|}
= \frac{\sqrt{2\pi}}{12}
\end{align*}
\endgroup
In the first step we upper bounded $A_3$ by only considering the positive
terms in the summation.
The second step bounds the term for $\xoptnuj$ by the sum of corresponding terms
for all $x\in\nu_j$.
We then apply Lemma~\ref{lem:discucbbound}.

Finally, we bound each term inside the summation of $A_4$ as follows,
\begin{align*}
&\EE[\Uj(\xjnuj) - \func(\xjnuj)] = 
\EE[\mufj(\xjnuj) + \betajh\sigmafj(\xjnuj) - \func(\xjnuj)] 
\numberthis \label{eqn:afourdecomp} \\
&\hspace{0.4in}=\EE[\mufj(\xjnuj) + \betajh\sigmafj(\xjnuj) - \EE[\func(\xjnuj)|\filtrj]] 
=\EE[\betajh\sigmafj(\xjnuj)] 
\end{align*}
Once again, we have used the fact that $\mufj,\sigmafj$ are deterministic given $\filtrj$.
Therefore,
% \begin{align*}
% A_4 &=\sum_{j=1}^n \EE[\betajh\sigmafj(\xjnuj)]
% \stackrel{(a)}{\leq}\betanh \sum_{j=1}^n \EE[\xiM\sigmafj(\xjnuj)]
% \stackrel{(b)}{\leq}\betanh \EE\Bigg[\sum_{j=1}^n \sigmajmo(\xjnuj) \Bigg] \\
% &\stackrel{(c)}{\leq} \betanh \EE\Bigg[ \bigg(n\sum_{j=1}^n \sigmasqj(\xjnuj)\bigg)^{1/2}\Bigg]
% \stackrel{(d)}{\leq} \sqrt{\frac{2n\betan\IGn}{\log(1+\eta^{-2})}}
% \numberthis \label{eqn:afour}
% \end{align*}
\begin{align*}
A_4 &\stackrel{(a)}{\leq}\betanh\sum_{j=1}^n \EE[\sigmafj(\xjnuj)]
% \stackrel{(a)}{\leq}\betanh \sum_{j=1}^n \EE[\xiM\sigmafj(\xjnuj)]
\stackrel{(b)}{\leq}\betanh\xiM^{1/2} \EE\Bigg[\sum_{j=1}^n \sigmajmo(\xjnuj) \Bigg]  \\
&\stackrel{(c)}{\leq} \betanh\xiM^{1/2} \EE\Bigg[ \bigg(n\sum_{j=1}^n
    \sigmasqj(\xjnuj)\bigg)^{1/2}\Bigg]
\stackrel{(d)}{\leq} \sqrt{\frac{2\xiM n\betan\IGn}{\log(1+\eta^{-2})}}
\numberthis \label{eqn:afour}
\end{align*}
Here, $(a)$ uses~\eqref{eqn:afourdecomp} and that 
$\betaj$ is increasing in $j$~\eqref{eqn:betan}.
$(c)$ uses the Cauchy-Schwarz inequality and $(d)$ uses Lemma~\ref{lem:IGBoundLemma}.
For $(b)$, first we note that $\filtrj\subseteq \{(\xii{i},\yii{i})\}_{i=1}^{j-1}$.
In the synchronously parallel setting $\filtrj$ could be missing up to $M$ of these
$j-1$ evaluations, i.e. $|\filtrj| = \floor{(j-1)/M}$.
In the asynchronous setting we will be missing exactly $M$ evaluations except during
the first $M$ steps, i.e. $|\filtrj| = j-M$ for all $j>M$.
In either case, letting $A=\filtrj$ and
$B=\{(\xii{i},\yii{i})\}_{i=1}^{j-1} \backslash \filtrj$
in Lemma~\ref{lem:stdmi} we get,
\begin{align*}
\text{for all $x\in\Xcal$}, \hspace{0.2in}
\sigmafj(x) 
\,\leq\, \exp\big(I(f;y_{B}|y_{\filtrj}) \big) \,\sigmajmo(x) 
\,\leq\, \xiM^{1/2} \sigmajmo(x).
\numberthis \label{eqn:sigmafjbound}
\end{align*}
The last step uses~\eqref{eqn:xiMbound} and that $|B| < M$.
Putting the bounds for $A_1,A_2,A_3$, and $A_4$ together we get,
$\bcrn\;\leq\; C_1 + \sqrt{C_2n\beta_n\IGn}$.  The theorem follows
from the relation $\bsrn\leq \frac{1}{n}\bcrn$.

Finally consider the case where the $n$ evaluations completed
are not the first $n$ dispatched.
Since $A_1,A_2,A_3$ are bounded by constants summing over all $n$ we only need to worry
about $A_4$.
In step $(a)$ of~\eqref{eqn:afour}, we have bounded $A_4$ by the sum of
posterior variances $\sigmafj([\xj]_j)$.
Since $\sigma_{\filtrjj{j'}}([\xj]_j) < \sigmafj([\xj]_j)$ for $j'>j$, the sum 
for any $n$ completed evaluations can be bound by the same sum for the
first $n$ evaluations dispatched.
The result follows accordingly. 
\end{proof}

The bound for the sequential setting in Theorem~\ref{thm:main} follows directly by setting
$M=1$ in the above analysis.

\insertthmprespacing
\begin{corollary}
\label{cor:tsseq}
Assume the same setting and quantities as in Theorem~\ref{thm:tsparallelmain}.
Then for \emph{\seqts}, the Bayes' simple regret after $n$ evaluations satisfies,
\emph{
\[
\bsrn\;\leq\; \frac{C_1}{n} + \sqrt{\frac{C_2\betan\IGn}{n}},
\]
}
\end{corollary}
\begin{proof}
The proof follows on exactly the same lines as above.
The only difference is that we will not require step $(b)$ in~\eqref{eqn:afour}
and hence will not need $\xiM$.
\end{proof}
\insertthmpostspacing

We conclude this section with justification for the discussion following
Theorem~\ref{thm:main}. 
Precisely that $\tsseqbound \leq \tssynbound \leq \tsasybound$ where
$\tsseqbound, \tssynbound, \tsasybound$ refer to the best achievable upper bounds
using our analysis. 
We first note that in Lemma~\ref{lem:stdmi},  $\sigma_{A\cup B} \leq \sigma_{A}$
as the addition of more points can only decrease the posterior variance.
Therefore, $\xiM$ is necessarily larger than $1$.
Hence $\tsseqbound \leq \tssynbound, \tsasybound$.
The result
$\tssynbound\leq\tsasybound$ can be obtained by a more careful analysis.
Precisely, in~\eqref{eqn:afour} and~\eqref{eqn:sigmafjbound} we will have to use $\xiM$ 
for the asynchronous setting for all $j>M$ since $|B| = 
|\{(\xii{i},\yii{i})\}_{i=1}^{j-1} \backslash \filtrj| = M$.
However, in the synchronous setting we can use a bound $\xi_{|B|}$ where
$|B|$ is less than $M$ most of the time.
Of course, the difference of $\tssynbound,\tsasybound$ relative to the
$\tsseqbound$ is dominated by the maximum number of missing evaluations which is
$M-1$ for both \syntss and \asyts.
We reiterate that $\tsseqbound,\tssynbound,\tsasybound$ are upper bounds
on the Bayes' regret $\bsrn$ and not the actual regret itself.

\subsubsection{Proof of Lemma~\ref{lem:discbound}}
\label{app:discboundproof}
Let $L = \sup_{i=1,\dots,d} \sup_{x\in\Xcal} \big|\frac{\partial \func(x)}{\partial
x_i}\big|$.
By Assumption~\ref{asm:kernelAssumption} and the union bound we have
$\PP(L \geq t) \leq da\exp^{-t^2/b^2}$.
Let $x\in\Xcal$.
We bound,
\begin{align*}
\EE[|\func(x) - \func(\xnuj)|] &\leq \EE[L\|x-\xnuj\|_1] \leq \frac{d}{\tau_j} \EE[L]
= \frac{d}{\tau_j}\int_0^\infty \PP(L \geq t)\ud t \\
&\leq \frac{d}{\tau_j}\int_0^\infty ae^{t^2/b^2}\ud t
= \frac{dab\sqrt{\pi}}{2\tau_j} = \frac{1}{2j^2}.
\end{align*}
The first step bounds the difference in the function values by the largest partial
derivative and the $L^1$ distance between the points.
The second step uses the properties of the discretisation $\nu_j$ and the third
step uses the identity $\EE X =\int \PP(X>t)\ud t$
for positive random variables $X$.
The last step uses the value for $\tau_j$ specified in the main proof.
\hfill \qed

\subsubsection{Proof of Lemma~\ref{lem:discucbbound}}
\label{app:discucbboundproof}
The proof is similar to Lemma 2 in~\citep{russo2014learning}, but we provide it here
for completeness.
We will use the fact that for $Z\sim\Ncal(\mu,\sigma^2)$, we have
$\EE[Z\indfone(Z > 0)] = \frac{\sigma}{\sqrt{2\pi}}e^{-\mu^2/(2\sigma^2)}$.
Noting that $\func(x)-\Uj(x)|\filtrj\sim \Ncal(-\betajh\sigmafj(x), \sigmasqfj(x))$,
we have,
\[
\EE[\indfone\{\func(x)>\Uj(x)\}\cdot(\func(x) - \Uj(x))|\filtrj]
= \frac{\sigmafj(x)}{\sqrt{2\pi}} e^{\betaj/2}
\leq \frac{1}{\sqrt{2\pi}|\nu_j| j^2}.
\]
Here, the last step uses that $\sigmafj(x) \leq \kernel(x,x) \leq 1$ 
and that $\betaj = 2\log(j^2|\nu_j|)$.
\hfill \qed

\subsection{On the Initialisation Scheme and Subsequent Results --
Proposition~\ref{prop:initialisation}}
\label{app:init}

The bound $\xiM$ could be quite large growing as fast as $M$, which is very unsatisfying
because then the bounds are no better than a strictly sequential algorithm for $n/M$
evaluations. 
~\citet{desautels2014parallelizing} show however that $\xiM$ can be bounded by a
constant $C'$ if we bootstrap a procedure with an uncertainty sampling procedure.
Precisely, we pick $\xinitjj{1} = \argmax_{x\in\Xcal}\sigmasqtt{0}(x)$ where
$\sigmasqtt{0}$ is the prior variance;
We then iterate $\xinitjj{j} = \argmax_{x\in\Xcal}\sigmasqtt{j-1}(x)$ until
$j=\ninit$.
As the posterior variance of a GP does not depend on the evaluations, this scheme is
asynchronously parallelisable: simply pre-compute the $\ninit$ evaluation points and
then deploy them in parallel.

By using the same initialisation scheme, we can achieve, for both \syntss and 
\asyts, the following:
\[
\bsrn \;\leq\; C'\tsseqbound \,+\, \frac{2\supgp\ninit}{n}
\]
Here, $\tsseqbound$ is the simple regret of \seqts.
The proof simply replaces the unconditional mutual information in the definition
of the MIG with the mutual information conditioned on the first $\ninit$ evaluations.%
~\citet{desautels2014parallelizing} provide bounds for $C'$ for different kernels.
For the SE kernel, $C'=\exp((2d/e)^d)$ and for the \matern kernel $C'=e$.
They also show that $\ninit$ typically scales as $\bigO(M\polylog(M))$.
If $M$ does not grow too large with $n$, then the first term above dominates and we
are worse than the sequential bound only up to constant factors.
% \akshay{I think we should state their theorem here.}

In practice however, as we have alluded to already in the main text, there are two
shortcomings with this initialisation scheme.
First, it requires that we know the kernel before any evaluations to $\func$.
Most BO procedures tune the kernel on the fly using its past evaluations, but this is
problematic without any evaluations.
Second the size of the initialisation set $\ninit$ has some problem dependent constants
that we will not have access to in practice.
We conjecture that \tss will not need this initialisation and wish to resolve this in
future work.
% For these reasons we employ a random intialisation scheme in our experiments.
% 
% \akshay{Do we want a paragraph about why we think even random
%   initialization is unecessary for Thompson Sampling?  Maybe we can make a conjecture?}

\section{Proofs for Parallelised Thompson Sampling with Random
Evaluation Times}

\label{app:rtv}

The goal of this section is to prove Theorem~\ref{thm:tstime}.
In Section~\ref{app:timervresults} we derive some concentration results for uniform
and half-normal distributes and their maxima.
 In Section~\ref{sec:exprv} we do the same for
exponential random variables.
We put everything together in Section~\ref{app:tsrtv} to prove Theorem~\ref{thm:tstime}.
We begin by reviewing some well known concepts in concentration of measure.

\subsection{Some Relevant Results}
\label{app:rtvrelated}

We first introduce the notion of sub-Gaussianity, which characterises one of the
stronger types of tail behaviour for random variables.

\insertthmprespacing
\begin{definition}[Sub-Gaussian Random Variables]
\label{defn:subgauss}
A zero mean random variable is said to be $\tau$ \emph{sub-Gaussian} if it satisfies,
$\EE[e^{\lambda X}] \;\leq\; e^\frac{\tau^2 \lambda^2}{2}$ for all $\lambda\in\RR$.
\end{definition}
\insertthmpostspacing

It is well known that Normal $\Ncal(0,\zeta^2)$ variables are $\zeta$ sub-Gaussian
and bounded random variables with support in $[a, b]$ are $(b-a)/2$
sub-Gaussian.
% 
% \insertthmprespacing
% \begin{lemma}[Sub-Gaussianity of bounded random variables~\citep{wainwright2015tail}]
% Let $X$ be a random variable such that $\supp(X) \subset [a, b]$.
% Then $X$ is $(b-a)/2$ sub-Gaussian.
% \end{lemma}
% \insertthmpostspacing
% 
For sub-Gaussian random variables, we have the following important and well known result.

\insertthmprespacing
\begin{lemma}[Sub-Gaussian Tail Bound]
\label{thm:hoeffding}
Let $X_1,\dots, X_n$ be zero mean independent random variables such that $X_i$ is
$\sigma_i$ sub-Gaussian. Denote 
$\Sn = \sum_{i=1}^n X_i$ and $\sigma^2 = \sum_{i=1}^n\sigma_i^2$.
Then, for all $\epsilon > 0$,
\[
\PP\left(\Sn \geq \epsilon\right)
\;\leq\; \exp\left( \frac{-\epsilon^2}{2\sigma^2} \right), \hspace{0.2in}
\PP\left(\Sn \leq \epsilon\right)
\;\leq\; \exp\left( \frac{-\epsilon^2}{2\sigma^2} \right).
\]
\end{lemma}

We will need the following result for Lipschitz functions of Gaussian
random variables in our analysis of the half-normal distribution for
time, see Theorem 5.6 in Boucheron et
al.~\cite{boucheron2013concentration}.

\insertthmprespacing
\begin{lemma}[Gaussian Lipschitz Concentration~\cite{boucheron2013concentration}]
\label{lem:gausslipschitz}
Let $X\in\RR^n$ such that $X_i\sim\Ncal(0,\zeta^2)$ iid for $i=1,\dots,n$.
Let $F:\RR^n\rightarrow\RR$ be an
$L$-Lipschitz function, i.e. $|F(x)-F(y)| \leq L\|x-y\|_2$ for all $x,y\in\RR^n$.
Then, for all $\lambda>0$,
$\EE[\exp^{\lambda F(X)}] \leq \exp\left(\frac{\pi^2 L^2 \zeta^2}{8}\lambda^2\right)$.
That is, $F(X)$ is $\frac{\pi L\zeta}{2}$ sub-Gaussian.
% \[
% \PP(|F(\Xon) - \EE F(\Xon)| \geq \epsilon) \leq 2\exp(-2\epsilon^2/\pi^2)
% \]
\end{lemma}

We also introduce Sub-Exponential random variables, which have a
different tail behavior.
\insertthmprespacing
\begin{definition}[Sub-Exponential Random Variables]
\label{defn:subexp}
A zero mean random variable is said to be \emph{sub-Exponential} with
parameters $(\tau^2,b)$ if it satisfies, $\EE[e^{\lambda X}] \;\leq\;
e^\frac{\tau^2 \lambda^2}{2}$ for all $\lambda$ with $|\lambda| \le
1/b$.
\end{definition}
\insertthmpostspacing

Sub-Exponential random variables are a special case of Sub-Gamma
random variables (See Chapter 2.4 in Boucheron et
al.~\cite{boucheron2013concentration}) and allow for a Bernstein-type
inequality.

\insertthmprespacing
\begin{proposition}[Sub-Exponential tail bound~\cite{boucheron2013concentration}]
\label{thm:subexponential}
Let $X_1,\dots, X_n$ be independent sub-exponential random variables
with parameters $(\sigma_i^2,b_i)$. Denote $\Sn = \sum_{i=1}^n X_i$
and $\sigma^2 = \sum_{i=1}^n\sigma_i^2$ and $b = \max_i b_i$.  Then,
for all $\epsilon > 0$,
\begin{align*}
\PP\left(\Big|\Sn-\sum_{i=1}^n \mu_i \Big| \geq \sqrt{2\sigma^2 t} + b t\right) \le 2\exp(-t).
\end{align*}
\end{proposition}
\insertthmpostspacing

\subsection{Results for Uniform and Half-normal Random Variables}
\label{app:timervresults}

In the next two lemmas,
let $\{X_i\}_{i=1}^M$ denote a sequence of $M$ i.i.d random variables and
$Y = \max_{i} X_i$ be their maximum.
We note that the results or techniques in
Lemmas~\ref{lem:uniformmax},~\ref{lem:halfnormalmax}
are not particularly new.
% variables and $\Xon = \{X_1,\dots, X_n\}$ will denote the first $n$ of them.
% When given a sequence of a collection of $M$ i.i.d random variables 
%  $\{(X_{i1},\dots, X_{iM})\}_{i\geq 1}$,
% $\{Y_i\}_{i\geq 1}$ where $Y_i = \max_{j=1,\dots,M} X_{ij}$ will
% denote the maximum of each collection.
% $\Yon = \{Y_1,\dots,Y_n\}$ will denote the first $n$ of them.

\insertthmprespacing
\begin{lemma}
\label{lem:uniformmax}
Let \emph{$X_i\sim \text{Unif}(a,b)$}. Then $\EE X_i = \theta$ and
$\EE Y = \theta + \frac{M-1}{M+1}\frac{b-a}{2}$ where $\theta = (a+b)/2$.
\begin{proof}
The proof for $\EE X_i$ is straightforward. The cdf of $Y$ is
$
\PP(Y\leq t) = \prod_{i=1}^M \PP(X_i\leq t) =
(\frac{t-a}{b-a})^M
$.
Therefore its pdf  is $p_Y(t) = M(t-a)^{M-1}/(b-a)^M$ and its expectation is
\begin{align*}
\EE[Y] = \int_a^b t M(t-a)^{M-1}/(b-a)^M \ud t = \frac{a+bM}{M+1} = \theta + 
\frac{M-1}{M+1}\frac{b-a}{2}.
\end{align*}
\end{proof}
\end{lemma}
\insertthmpostspacing

\insertthmprespacing
\begin{lemma}
\label{lem:halfnormalmax}
Let \emph{$X_i\sim \HNcal(\zeta^2)$}. Then $\EE X_i = \zeta\sqrt{2/\pi}$ and
$\EE Y$ satisfies,
% $\EE Y \asymp \sqrt{\logM}\cdot \EE X_i$. That is, it satisfies
\[
\zeta K \sqrt{\log(M)} \leq \EE Y \leq \zeta \sqrt{2\log(2M)}.
\]
Here $K$ is a universal constant.
Therefore, $\EE Y \in \bigTheta(\sqrt{\logM}) \EE X_i$.
% $\EE Y \in \bigTheta\big(\sqrt{\logM}\big)\cdot\theta$ where $\theta= $.
\begin{proof}
The proof for $\EE X_i$ just uses integration over the pdf
$p_Y(t) = \frac{\sqrt{2}}{\sqrt{\pi\zeta^2}}e^{-\frac{t^2}{2\sigma^2}}$.
For the second part, writing the pdf of $\Ncal(0,\zeta^2)$ as $\phi(t)$ we have,
\[
\EE[e^{\lambda X_i}] = 2 \int_0^\infty e^{\lambda t}\phi(t) \ud t
\leq 2\int_{-\infty}^{\infty} e^{\lambda t}\phi(t) \ud t
= 2\EE_{Z\sim\Ncal(0,\zeta^2)}[e^{\lambda Z}] = 2 e^{\zeta^2\lambda^2/2}.
\]
The inequality in the second step uses that the integrand is positive.
Therefore, using Jensen's inequality and the fact that the maximum is smaller than the sum
we get,
\[
e^{\lambda \EE[Y]} \leq \EE[e^{\lambda Y}] \leq \sum_{i=1}^n \EE[e^{\lambda X_i}]
\leq 2M e^{\lambda^2\zeta^2/2} 
\hspace{0.1in} \implies 
\hspace{0.1in}
\EE[Y] \leq \frac{1}{\lambda}\log(2M) + \frac{\zeta^2 \lambda}{2}.
\]
Choosing $\lambda=\frac{\sqrt{2\log(2M)}}{\zeta}$ yields the upper bound.
The lower bound follows from Lemma 4.10 of~\citet{adler1990gps} which 
establishes a $K\sqrt{\logM}$ lower bound for $M$ i.i.d standard normals $Z_1,\dots,Z_M$.
We can use the same lower bound since $|Z_i| \geq Z_i$.
\end{proof}
\end{lemma}
\insertthmpostspacing

\insertthmprespacing
\begin{lemma}
\label{lem:sgconc}
Suppose we complete a sequence of jobs indexed $j=1,2,\dots$.
The time taken for the jobs $\{X_j\}_{j\geq 1}$ are i.i.d with mean $\theta$
and sub-Gaussian parameter $\tau$.
Let $\delta\in(0,1)$, and $N$ denote the number of completed jobs after time $T$.
That is, $N$ is the random variable such that
$N = \max\{n\geq 1; \sum_{j=1}^n X_j \leq T\}$.
Then, with probability greater than $1-\delta$,
for all $\alpha\in(0,1)$, there exists $\Tad$ such that for all $T>\Tad$,
$N\in \Big(\frac{T}{\theta(1+\alpha)}-1,  \frac{T}{\theta(1-\alpha)}\Big)$.
\begin{proof}
We will first consider the total time taken $\Sn=\sum_{i=1}^n X_i$ after $n$ evaluations.
Let $\epsln = \tau \sqrt{n\log(n^2\pi^2/(3\delta))}$ throughout this proof.
Using Lemma~\ref{thm:hoeffding}, we have
$\PP(|\Sn-n\theta|>  \epsln) = 6\delta/(n\pi^2)$.
By a union bound over all $n\geq 1$, we have that with 
with probability greater than $\delta$, the following event $\Ecal$ holds.
\begin{align*}
\Ecal = \{
 \forall n\geq 1,\;\;
|\Sn-n\theta| \leq \epsln.
\}
\label{eqn:SnEcal}\numberthis
\end{align*}
Since $\Ecal$ is a statement about all time steps, it is also for true the random number
of completed jobs $N$.
Inverting the condition in~\eqref{eqn:SnEcal} and using the definition of $N$,
we have 
\begin{align*}
N\theta -\epslN \leq S_N \leq T \leq S_{N+1} \leq (N+1)\theta + \epslNpo.
\label{eqn:Tbound}\numberthis
\end{align*}
Now assume that there exists $\Tad$ such that for all $T\geq \Tad$ we have,
$\epslN \leq N\alpha\theta$.
Since $\epsln$ is sub-linear in $n$, it  also follows that 
$\epslNpo \leq (N+1)\alpha\theta$.
Hence, $N\theta(1-\alpha) \leq T \leq (N+1)\theta(1+\alpha)$ and the result follows.

All that is left to do is to establish that such a $\Tad$ exists under event $\Ecal$,
for which we will once again appeal to~\eqref{eqn:Tbound}.
The main intuition is that as $\epslN\asymp\sqrt{N\log(N)}$, the condition
$\epslN \leq N\alpha\theta$ is satisfied for $N$ large enough.
But $N$ is growing with $T$, and hence it is satisfied for $T$ large enough.
More formally, 
since $\frac{N}{\epslN}\asymp \frac{N+1}{\epslNpo}$ using the upper bound for $T$
 it is sufficient to show
$\frac{T}{\epslNpo} \gtrsim \frac{1}{\alpha\theta}$ for all $T\geq \Tad$.
But since $\epslNpo\asymp\sqrt{N\log(N)}$ and the lower bound for $T$ is
$T\gtrsim N$, it is sufficient if 
$\frac{T}{\sqrt{T\log(T)}} \gtrsim \frac{1}{\alpha\theta}$ for all $T\geq \Tad$.
This is achievable as the LHS is increasing with $T$ and the RHS is constant.
% 
% % I will resolve this after the deadline.
% \akshay{This last paragraph isn't clear to me. I understand what
%   you're trying to do, but I do not see a clean way to explain
%   it. Here is what I was thinking though. The condition on $N$ is
%   satisfied for $N$ large enough (for all $N \ge N_0$, since $\epslN
%   \asymp \sqrt{N \log N}$ and $N\alpha\theta \asymp N$. But $N(T)$ is
%   clearly growing in $T$ and hence for all $T$ large enough the
%   condition is satisfied.
%   I think this is roughly what you're saying but some clarification
%   might help. Also doing things more precisely might help.}
\end{proof}
\end{lemma}
\insertthmpostspacing

Our final result for the uniform and half-normal random variables follows
as a consequence of Lemma~\ref{lem:sgconc}.

\insertthmprespacing
\begin{theorem}
\label{thm:subgaussian_times}
\label{thm:sgrvN}
Let the time taken $X$ for completing an evaluation to $\func$
be a random variable.
\begin{itemize}
\item If \emph{$X\sim\textrm{Unif}(a,b)$}, denote
$\theta=(a+b)/2$, $\,\thetaM = \theta + \frac{M-1}{M+1}\frac{b-a}{2}$,
and $\tau=(b-a)/2$.
\item If $X\sim\HNcal(\tau^2)$, denote
$\theta=\zeta\sqrt{2/\pi}$, $\;\thetaM = \theta\cdot\bigTheta(\sqrt{\logM})$, and
$\tau=\zeta\pi/2$.
\end{itemize}
% Denote the number of evaluations within time $T$ by a synchronous algorithm by
% \emph{$\Nsyn$} and by an asynchronous algorithm by \emph{$\Nasy$}.
Denote the number of evaluations within time $T$ by sequential,
synchronous parallel and asynchronous parallel algorithms by
\emph{$\Nseq,\Nsyn,\Nasy$} respectivey.
% \emph{$\Nsyn$} and by an asynchronous algorithm by \emph{$\Nasy$}.
Let $\delta\in(0,1)$.
Then, with probability greater than $1-\delta$,
for all $\alpha\in(0,1)$, there exists $\Tad$ such that for all
$T\geq \Tad$, we have each of the
following,
\emph{
\begin{align*}
% \Nseq\in \left(\frac{T}{\theta(1+\alpha)},  \frac{T}{\theta(1-\alpha)}\right),
% \hspace{0.1in}
% \Nsyn\in \left(\frac{MT}{\thetaM(1+\alpha)},  \frac{MT}{\thetaM(1-\alpha)}\right),
% \hspace{0.1in}
% \Nasy\in \left(\frac{MT}{\theta(1+\alpha)},  \frac{MT}{\theta(1-\alpha)}\right).
&\Nseq\in \left(\frac{T}{\theta(1+\alpha)}-1,  \frac{T}{\theta(1-\alpha)}\right),
\hspace{0.1in}
\Nsyn\in \left(M\bigg[\frac{T}{\thetaM(1+\alpha)} -1\bigg], 
          \frac{MT}{\thetaM(1-\alpha)}\right), \\[0.05in]
% \hspace{0.1in}
&\Nasy\in \left(M\bigg[\frac{T}{\theta(1+\alpha)} - 1\bigg], 
                \frac{MT}{\theta(1-\alpha)}\right).
\end{align*}
}
% \akshay{Do you also want to include sequential for completeness?}
% provided that .......\toworkon{fill in complex condition here.}
\begin{proof}
We first show $\tau$ sub-Gaussianity of $X$ and $Y=\max_{j=1,\dots,M}X_j$ when
$X,X_1,\dots,X_M$ are either uniform or half-normal.
For the former, both $X$ and $Y$ are $\tau = (b-a)/2$ sub-Gaussian since they are
bounded in $[a,b]$.
For the Half-normal case, we note that $X = |Z|$ and
$Y = \max_{j=1,\dots,M} |Z_{i}|$ for some i.i.d. $\Ncal(0,\zeta^2)$ variables
$Z, \{Z_i\}_{i=1}^M$.
Both are $1$-Lipschitz functions of $Z_i$ and $(Z_{i1},\dots,Z_{iM})$
respectively and $\tau = \zeta\pi/2$
sub-Gaussianity follows from Lemma~\ref{lem:gausslipschitz}.

Now in synchronous settings, the algorithm dispatches the $k$\ssth batch with
evaluation times $\{(X_{k1},\dots, X_{kM})\}$.
It releases its $(k+1)$\ssth batch when all evalutions finish after time 
$Y_k= \max_{i=1,\dots,M}X_{ki}$.
The result for $\Nsyn$ follows by applying Lemma~\ref{lem:sgconc} on the sequence
$\{Y_k\}_{k\geq 1}$.
For the sequential setting, each worker receives its $(k+1)$\ssth job after
completing its $k$\ssth evaluation in time $X_k$.
We apply Lemma~\ref{lem:sgconc} on the sequence $\{X_k\}_{k\geq 1}$ for one worker
to obtain that the number of jobs completed by this worker is
% as given in the theorem.
% $\Nseq \in \big\{\frac{T}{\theta(1\pm\alpha)}\big\}$.
$\Nseq \in \big(\frac{T}{\theta(1+\alpha)}-1, \frac{T}{\theta(1-\alpha)}\big)$.
In the asynchronous setting, a worker receives his new job immediately after finishing
his last.
Applying the same argument as the sequential version to all workers but with
$\delta\leftarrow\delta/M$ in Lemma~\ref{lem:sgconc} and the union bound
yields the result for $\Nasy$.
% The result follows by applying this to each worker.
\end{proof}
\end{theorem}
\insertthmpostspacing

% \begin{itemize}
% 
% \end{itemize}

\subsection{Results for the Exponential Random Variable}
\label{sec:exprv}

In this section we derive an analogous result to
Theorem~\ref{thm:subgaussian_times} for the case when the completion
times are exponentially distributed. The main challenges stem from
analysing the distribution of the maxima of a finite number of
exponential random variables. Much of the analysis is based on results
from Boucheron and Thomas~\cite{boucheron2012concentration} (See also
chapter 6 of Boucheron et al.~\cite{boucheron2013concentration}).

In deviating from the notation used in Table~\ref{tb:rtv}, we will denote the parameter of
the exponential distribution as $\theta$, i.e. it has pdf $p(x) = \theta x^{-\theta x}$.
The following fact about exponential random variables will be instrumental.

\insertthmprespacing
\begin{fact}
  \label{fact:renyi}
  Let $X_1,\ldots,X_n \sim \Exp(\theta)$ iid. Also let
  $E_1,\ldots,E_n \sim \Exp(\theta)$ iid and independent from
  $X_1^n$. If we define the order statistics $X_{(1)} \ge X_{(2)} \ge
  \ldots \ge X_{(n)}$ for $X_1,\dots,X_n$, we have
  \begin{align*}
    (X_{(n)},\ldots,X_{(1)}) \sim \left(E_n/n, \ldots, \sum_{k=i}^n E_k/k, \ldots, \sum_{k=1}^n E_k/k \right). 
  \end{align*}
\begin{proof}
This is Theorem 2.5 in~\cite{boucheron2012concentration} but we
include a simple proof for completeness.  We first must analyse the
minimum of $n$ exponentially distributed random variables. This is a
simple calculation.
\begin{align*}
  \PP[\min_i X_i \ge t] = \prod_{i=1}^n \PP[X_i \ge t] = \prod_{i=1}^n \exp(-\theta t) = \exp(-n\theta t)
\end{align*}
This last expression is exactly the probability that an independent
$\Exp(n\theta)$ random variables is at least $t$.

This actually proves the first part, since $E_n/n \sim \Exp(n\theta)$.
Now, using the memoryless property, conditioning on $X_{(n)} = x$ and
$X_{(n)} = X_i$ for some $i$, we know that for $j \ne i$
\begin{align*}
\PP[X_j \ge x'+x | X_{(n)} = x, X_{(n)} = X_i] = \exp(-\theta x'). 
\end{align*}
Removing the conditioning on the index achieving $X_{(n)}$, and using
the same calculation for the minimum, we now get
\begin{align*}
  \PP[X_{(n-1)} \ge x'+x | X_{(n)} = x] = \exp(-(n-1)\theta x')
\end{align*}
Thus we have that $X_{(n-1)} - X_{(n)} \sim \Exp((n-1)\theta)$. The
claim now follows by induction.
\end{proof}
\end{fact}
\insertthmpostspacing

As before the first step of the argument is to understand the expectation of the maximum. 
\insertthmprespacing
\begin{lemma}
\label{lem:exponentialmax}
Let \emph{$X_i\sim \Exp(\theta)$}. Then $\EE X_i = 1/\theta$ and
$\EE Y = h_M/\theta$ where $Y=\max_{i=1,\dots,M}$ is the maximum of the $X_i$'s
and $h_M = \sum_{i=1}^M i^{-1}$ is the
$M$\ssth harmonic number.
\begin{proof}
  Using the relationship between the order statistics and the spacings in Fact~\ref{fact:renyi} we get
  \begin{align*}
    \EE \max_i X_i = \EE X_{(1)} = \EE \sum_{k=1}^M E_k/k = \sum_{k=1}^M\frac{1}{k\theta}
  = \frac{h_M}{\theta}.\tag*\qedhere
  \end{align*}
\end{proof}
\end{lemma}
\insertthmpostspacing

Recall that $h_M \asymp \log(M)$ accounting for the claims made in Table~\ref{tb:rtv} and
the subsequent discussion.

While obtaining polynomial concentration is straightforward via
Chebyshev's inequality it is insufficient for our purposes, since we
will require a union bound over many events. However, we can obtain
exponential concentration, although the argument is more complex.  Our
analysis is based on Herbst's argument, and a modified logarithmic
Sobolev inequality, stated below in Theorem~\ref{thm:logsobolev}. To
state the inequality, we first define the entropy $\Ent[X]$ of a
random variable $X$ as follows (not to be confused with Shannon
entropy),
\[
\Ent[X] \triangleq \EE[X\log(X)] - \EE[X]\log(\EE[X]).
\]

\insertthmprespacing
\begin{theorem}[Modified logarithmic Sobolev inequality (Theorem 6.6 in~\cite{boucheron2013concentration})]
\label{thm:logsobolev}
Let $X_1,\ldots,X_n$ be independent random variables taking values in
some space $\Xcal$, $f:\Xcal^n\rightarrow \RR$, and define the random
variable $Z = f(X_1,\ldots,X_n)$. Further let $f_i:
\Xcal^{n-1}\rightarrow \RR$ for $i \in \{1,\dots,n\}$ be arbitrary functions and
$Z_i = f_i(X^{(i)}) = f_i(X_1,\ldots,X_{i-1},X_{i+1},\ldots,X_n)$.
Finally define $\tau(x) = e^x - x - 1$. Then for all $\lambda \in \RR$
% \begin{align*}
% \Ent[e^{\lambda Z}] \triangleq \lambda \EE [Z e^{\lambda Z}] - \EE[e^{\lambda Z}]\log \EE[e^{\lambda Z}] \le \sum_{i=1}^n \EE[e^{\lambda Z}\tau(-\lambda(Z - Z_i))]
% \end{align*}
\begin{align*}
\Ent[e^{\lambda Z}] \le \sum_{i=1}^n \EE[e^{\lambda Z}\tau(-\lambda(Z - Z_i))].
\end{align*}
\end{theorem}
\insertthmpostspacing

Application of the logarithmic Sobolev inequality in our case gives:
\insertthmprespacing
\begin{lemma}
Let $X_1,\ldots,X_M \sim \Exp(\theta)$ iid, $E \sim \Exp(\theta)$ also
independently and define $Z = \max_{i \in \{1,\dots,M\}} X_i$ and $\mu = \EE
Z$. Define $\tau(x) = e^x -x -1$ and $\psi(x) = \exp(x)\tau(-x) = 1 +
(x-1)e^{x}$. Then for any $\lambda \in \RR$
\begin{align*}
\Ent[\exp\{\lambda (Z - \mu)\}] &\le
  \EE[\exp\{\lambda (Z - \mu)\}] \times \EE \psi(\lambda E),\\
\Ent[\exp\{\lambda (\mu-Z)\}] &\le \EE[\exp\{\lambda (\mu-Z)\}] \times \EE \tau(\lambda E).
\end{align*}
\end{lemma}
\begin{proof}
We apply Theorem~\ref{thm:logsobolev} with $Z = f(X_1,\ldots,X_M) =
\max_i X_i$ and $Z_i = f_i(X^{(i)}) = \max_{j \ne i} X_j$. Notice that
in this case, $Z_i = Z$ except when $X_i$ is the maximiser, in which
case $Z_i = X_{(2)}$ the second largest of the samples. This applies here since the maximiser is unique with probability $1$. Thus, Theorem~\ref{thm:logsobolev} gives
\begin{align*}
\Ent[\exp\{\lambda Z\}] &\le \sum_{i=1}^M \EE[\exp\{\lambda Z\}\tau(-\lambda(Z - Z_i))] =
  \EE \left[\exp\{\lambda Z\}\tau(-\lambda (X_{(1)} - X_{(2)}))\right]\\
& = \EE \left[\exp\{\lambda X_{(2)}\} \exp\{\lambda (X_{(1)} - X_{(2)})\} \tau(-\lambda
(X_{(1)} - X_{(2)}))\right]\\
& = \EE[\exp{\lambda X_{(2)}}] \EE[\psi(\lambda E)]
\le \EE[\exp{\lambda X_{(1)}}] \EE[\psi(\lambda E)]
\end{align*}
The first inequality is Theorem~\ref{thm:logsobolev}, while the first
equality uses the definitions of $f_i$ and the fact that $Z_i \ne Z$
for exactly one index $i$. The second equality is straightforward and
the third uses Fact~\ref{fact:renyi} to write $X_{(1)} - X_{(2)}$ as
an independent $\Exp(\theta)$ random variable, which also allows us to
split the expectation. Finally since $X_{(2)}\le X_{(1)}$ almost
surely, the final inequality follows. Multiplying both sides by
$\exp(-\lambda \mu)$, which is non-random, proves the first
inequality, since $\Ent(aX) = a\Ent(X)$.

The second inequality is similar. Set $Z = -\max_i X_i$ and $Z_i =
-\max_{j \ne i} X_j$ and using the same argument, we get
\begin{align*}
  \Ent[\exp\{-\lambda X_{(1)} \}] &\le \EE[\exp\{-\lambda X_{(1)}\}\tau(\lambda (X_{(1)} - X_{(2)}))] \\
  &= \EE\left[\exp\left\{-\lambda(E_1 + \sum_{k=2}^ME_k/k)\right\}\tau(\lambda E_1)\right]
\end{align*}
Here the inequality follows from Theorem~\ref{thm:logsobolev} and the
identity uses Fact~\ref{fact:renyi}. We want to split the expectation,
and to do so, we use Chebyshev's association principle. Observe that
$\exp(-\lambda E_1)$ is clearly non-increasing in $E_1$ and that
$\tau(\lambda E_1)$ is clearly non-decreasing in $E_1$ for $E_1 \ge 0$
($E_1>0$ a.s.). Hence, we can split the expectation to get
\begin{align*}
\Ent[\exp\{-\lambda X_{(1)} \}] \le \EE[\exp\{-\lambda X_{(1)}\}]\times \EE[\tau(\lambda E)]
\end{align*}
The second inequality follows now by multiplying both sides by
$\exp(\lambda \mu)$.
\end{proof}
\insertthmpostspacing

\insertthmprespacing
\begin{theorem}
\label{thm:max_subexponential}
Let $X_1,\ldots,X_M \sim \Exp(\theta)$ iid and define $Z = \max_i X_i$
then $Z - \EE Z$ is sub-exponential with parameters $(4/\theta^2,2/\theta)$.
\end{theorem}
\begin{proof}
We use the logarithmic Sobolev inequality, and proceed with Herbst's
method. Unfortunately since our inequality is not in the standard
form, we must reproduce most of the argument. However, we can unify
the two tails by noticing that we currently have for centered $Y$
(e.g., $Y = X_{(1)}- \EE X_{(1)}$ or $Y = \EE X_{(1)} - X_{(1)}$),
\begin{align}
\Ent[\exp\{\lambda Y\}] \le \EE[\exp\{\lambda Y\}] f(\lambda)
\label{eq:log_sobolev}
\end{align}
for some differentiable function $f$, which involves either $\tau$ or
$\psi$ depending on the tail. We will use such an inequality to bound
the moment generating function of $Y$.

For notational convenience, define $\phi(\lambda) = \log \EE \exp\{\lambda Y\}$ and observe that
\begin{align*}
  \phi'(\lambda) = \frac{1}{\lambda}\left( \frac{\Ent[\exp\{\lambda Y\}]}{\EE \exp\{\lambda Y\}} + \log \EE \exp\{\lambda Y\} \right)
\end{align*}
Together with the inequality in Eq.~\eqref{eq:log_sobolev}, this gives
\begin{align*}
&  \lambda \phi'(\lambda) - \phi(\lambda) = \frac{\Ent[\exp\{\lambda Y\}]}{\EE \exp\{\lambda Y\}} \le f(\lambda)\\
\Leftrightarrow & \frac{\phi'(\lambda)}{\lambda}- \frac{\phi(\lambda)}{\lambda^2} \le
f(\lambda)/\lambda^2, \hspace{0.3in} \forall\;\lambda>0
\end{align*}
Observe now that the left hand side is precisely the derivative of the
function $G(\lambda) = \phi(\lambda)/\lambda$. Hence, we can integrate
both sides from $0$ to $\lambda$, we get
\begin{align*}
\frac{\phi(\lambda)}{\lambda} \le \int_0^\lambda f(t)/t^2 dt.
\end{align*}
This last step is justified in part by the fact that
$\lim_{t\rightarrow 0} \phi(t)/t = 0$ by L'Hopital's rule.  Thus we
have $\log \EE \exp\{\lambda Y\} \le \lambda \int_0^\lambda f(t)/t^2
dt$.

\textbf{The upper tail:} For the upper tail $Z - \EE Z$, we have $f(t) = \EE \psi(tE)$ where $E \sim \Exp(\theta)$ and $\psi(x) = 1 + (x-1)e^x$.
By direct calculation, we have for $t<\theta$
\begin{align*}
\EE \psi(tE) &= 1 + \EE tE \exp (tE) - \EE \exp(t E)\\
& = 1 - \frac{\theta}{\theta - t} + t \int_{0}^\infty x \exp(t x) \theta \exp(-\theta x) dx\\
& = 1 - \frac{\theta}{\theta - t} + \frac{t\theta}{(\theta - t)^2} = \frac{t^2}{(\theta -
t)^2}.
\end{align*}
Thus, we get
\begin{align*}
\log \EE \exp\{\lambda (Z - \EE Z)\} \le \lambda \int_0^\lambda \frac{1}{(\theta - t)^2} dt = \frac{\lambda^2}{\theta(\theta - \lambda)}.
\end{align*}
If $\lambda \le \theta/2$, this bound is $2\lambda^2/\theta^2$. Thus,
according to definition~\ref{defn:subexp}, $Z - \EE Z$ is
sub-exponential with parameters $(4/\theta^2, 2/\theta)$.

\textbf{The lower tail:} For the lower tail $\EE Z - Z$ we need to
control $\EE \tau(t E)$ where $E \sim \Exp(\theta)$, $\tau(x) = e^x -x
- 1$. Direct calculation, using the moment generating function of exponential random variables gives
\begin{align*}
  \EE \tau(tE) = \frac{t^2}{\theta (\theta-t)}
\end{align*}
So the integral bound is
\begin{align*}
\log \EE \exp\{\lambda(\EE Z - Z)\} &\le \lambda \int_0^\lambda \frac{1}{\lambda(\lambda - t)} = \frac{\lambda}{\theta}\log\left(\frac{\theta}{\theta-\lambda}\right)\\
& = \frac{\lambda}{\theta}\left(\sum_{i=1}^\infty (\lambda/\theta)^i/i\right)\\
& = \frac{\lambda^2}{\theta^2}\left(\sum_{i=1}^\infty (\lambda/\theta)^{i-1}/i\right)
\end{align*}
If $\lambda/\theta \le 1/2$ the series inside the paranthesis is clearly bounded by
$2$. Thus $\EE Z - Z$ is sub-exponential with parameters
$(4/\theta^2,2/\theta)$ as before.
\end{proof}
% \insertthmpostspacing

Now that we have established that the maximum is sub-Exponential, we can bound the number of evaluations for the various methods. This is the main result for this section.

\insertthmprespacing
\begin{theorem}
\label{thm:exprvN}
Let the time taken $X$ for completing an evaluation to $f$ be a random
variable that is $\Exp(\theta)$ distributed. Let $\delta \in (0,1)$
and denote $\Nsyn$ and $\Nasy$ denote the number of evaluations by
synchronous and asynchronous algorithms with time $T$. Then with probability at least $1-\delta$, for any $\alpha \in (0,1)$ there exists $T_{\alpha,\theta}$ such that
\emph{
\begin{align*}
&\Nseq \in \left(\frac{T\theta}{(1+\alpha)}-1,
  \frac{MT\theta}{(1-\alpha)}\right), \hspace{0.2in}
\Nsyn \in \left(M \left(\frac{T\theta}{h_M(1+\alpha)}-1\right),
\frac{MT\theta}{h_M(1-\alpha)}\right) \\ %\hspace{0.2in}
&\Nasy \in \left(M\left(\frac{T\theta}{(1+\alpha)}-1\right),
  \frac{MT\theta}{(1-\alpha)}\right)
\end{align*}
}
\end{theorem}
\begin{proof}
  In the synchronous setting, the $k$\ssth batch issues $M$ jobs with
  lengths $(X_{k1},\ldots,X_{kM})$ and the batch ends after $Y_k =
  \max_i X_{ki}$. Since the sequence of random variables
  $\{Y_k\}_{k\ge1}$ are all iid and sub-exponential with parameters
  $(4/\theta^2,2/\theta)$, in a similar way to the proof of
  Lemma~\ref{lem:sgconc}, with $S_n = \sum_{k=1}^n Y_k$ we get that
  \begin{align*}
    \PP\left(\exists n;\; |S_n - \EE S_n| \ge \underbrace{\sqrt{8n\theta^{-2} \log(n^2\pi^2/(3\delta))} + \frac{2}{\theta}\log(n^2\pi^2/(3\delta))}_{\triangleq \epsilon_n}\right) \le \delta
  \end{align*}
  This follows from Bernstein's inequality
  (Proposition~\ref{thm:subexponential}) and the union bound. As in
  Lemma~\ref{lem:sgconc} this means that:
  \begin{align*}
    \frac{Nh_M}{\theta} - \epsilon_N \le S_N \le T \le S_{N+1}\le \frac{(N+1)h_M}{\theta} + \epsilon_{N+1}.
  \end{align*}
  Here we also used the fact that $\EE Y_k = h_M/\theta$ from
  Lemma~\ref{lem:exponentialmax}. Now assuming there exists
  $T_{\alpha,\delta}$ such that for all $T \ge T_{\alpha,\delta}$, we
  have $\epsilon_N \le Nh_M\alpha/\theta$, we get
  \begin{align*}
    \frac{Nh_M}{\theta}(1-\alpha) \le T \le \frac{(N+1)h_m}{\theta}(1+\alpha).
  \end{align*}
  The existence of $T_{\alpha,\delta}$ is based on the same argument as
  in Lemma~\ref{lem:sgconc}. Re-arranging these inequalities, which
  gives a bound on the number of batches completed, leads to the bounds
  on the number of evaluations for the synchronous case. 

  Applying the same argument to a single worker on the sequence
  $\{X_k\}_{k\geq 1}$, we get
  \begin{align*}
    \frac{\Nseq}{\theta}(1-\alpha) \le T \le \frac{\Nseq+1}{\theta}(1+\alpha).
  \end{align*}
  Repeating this argument for all $M$ workers with $\delta\leftarrow \delta/M$
  and then taking a union bound yields the result for $\Nasy$.
%   $\delta\leftarrow\ld
%   For the asynchronous case, essentially the same proof applies, but
%   we must take a union bound over the $M$ machines.
%   Nevertheless the same argument leads to the inequalities
%   \begin{align*}
%     \frac{N}{\theta}(1-\alpha) \le T \le \frac{N+1}{\theta}(1+\alpha)
%   \end{align*}
%   where $N$ is the number of evaluations executed by a single machine.
%   The bound on $\Nasy$ follows by simple calculations. 
\end{proof}
\insertthmpostspacing

\subsection{Putting it altogether}
\label{app:tsrtv}

Finally, we put the results in Theorems~\ref{thm:tsparallelmain},~\ref{thm:sgrvN}
and~\ref{thm:exprvN} together to
obtain the following result.
This is a formal version of Theorem~\ref{thm:tstime} in the main text.

\insertthmprespacing
\begin{theorem}
\label{thm:tsparalleltime}
Let $\func\sim\GP(\zero,\kernel)$ where $\kernel:\Xcal^2\rightarrow\RR$ satisfies
Assumption~\ref{asm:kernelAssumption} and $\kernel(x,x')\leq 1$.
Then for all $\alpha > 0$, the Bayes simple regret for 
\emph{\seqts}, \emph{\synts} and \emph{\asyts}, satisfies the following for sufficiently
large $T$.
\emph{
\begin{align*}
&\text{\seqts:}\hspace{0.4in}
\bsrtT\;\leq\; \frac{C'_1}{\nseq} + \sqrt{\frac{C_2\beta_{\nseq}\IGnn{\nseq}}{\nseq}},
\hspace{0.59in}\nseq = \frac{T}{\theta(1+\alpha)} - 1 \\
&\text{\synts:}\hspace{0.4in}
\bsrtT\;\leq\; \frac{C'_1}{\nsyn} + \sqrt{\frac{C_2\xiM\beta_{\nsyn}\IGnn{\nsyn}}{\nsyn}},
\hspace{0.4in}\nsyn = M\bigg[\frac{T}{\thetaM(1+\alpha)} - 1\bigg] \\
&\text{\asyts:}\hspace{0.4in}
\bsrtT\;\leq\; \frac{C'_1}{\nasy} + \sqrt{\frac{C_2\xiM\beta_{\nasy}\IGnn{\nasy}}{\nasy}},
\hspace{0.43in}\nasy = M\bigg[\frac{T}{\theta(1+\alpha)} - 1\bigg]
\end{align*}
}
Here, $\theta,\thetaM$ are defined as follows for the uniform, half-normal and
exponential cases.
\emph{
\begin{align*}
&\unif(a,b):\hspace{0.4in} \theta = \frac{a+b}{2}, \hspace{0.23in}
  \thetaM = \frac{a+bM}{M+1} \\
&\HNcal(\zeta^2):\hspace{0.47in} \theta = \frac{\zeta\sqrt{2}}{\sqrt{\pi}},
  \hspace{0.28in} \thetaM \in \zeta \cdot \bigTheta(\log(M)) \\
&\Exp(\lambda):\hspace{0.53in}
  \theta = \frac{1}{\lambda}, \hspace{0.46in} \thetaM = \frac{h_M}{\lambda}
\end{align*}
}
Further, $\IGn$ is the MIG in Definition~\ref{def:infGain}, $\betan$ is
as defined in~\eqref{eqn:betan},
$\xiM$ is from~\eqref{eqn:xiMbound}, and $C_1=\pi^2/6 + \sqrt{2\pi}/12 + 1$,
$\;C_2=2/\log(1+\eta^{-2})$ are constants.
\begin{proof}
We will prove the result for \asytss as the others are obtained by an identical
argument.
Let $\Vcal$ denote the event that 
$N\geq M\big[\frac{T}{\theta(1+\alpha)} - 1\big]$. 
Theorems~\ref{thm:sgrvN} and~\ref{thm:exprvN} give us control on this event with
probability at least $1-\delta$.
We will choose $\delta = \frac{1}{2\supgp \nasy}$ where $\supgp$ is the expected
maximum of the GP in Lemma~\ref{lem:supgp}.
Since the randomness in the evaluation times are independent of the prior, noise
and the algorithm, we can decompose $\bsrtT$ as follows and use the result in
Theorem~\ref{thm:tsparallelmain} for $\bsrn$.
\begin{align*}
\bsrtT &\leq \EE[\bsrii{N}|\Vcal]\PP(\Vcal) + \EE[\bsrii{N}|\Vcal^c]\PP(\Vcal^c) \\
  &\leq \bsrii{\nasy} \cdot 1 \,+\, 2\supgp \delta
\end{align*}
Here we have used the definition of the Bayes simple regret with time
in~\eqref{eqn:srtDefn}
which guarantees that it is never worse than $\sup_x |\func(\xopt) - \func(x)| \leq
2\supgp$.
The theorem follows by plugging in values for $\bsrn$ and $\delta$.
The ``sufficiently large $T$'' requirement is because Theorems~\ref{thm:sgrvN}
and~\ref{thm:exprvN} hold only for $T>\Tad = T_{\alpha,\frac{1}{2\supgp \nasy}}$.
Since the dependencies of $\delta$ on $\Tad$ is polylogarithmic,
and as $\nasy$ is growing linearly with $T$, the above condition
is equivalent to $T \gtrsim \polylog(T)$ which is achievable for large enough $T$.
\end{proof}
\end{theorem}

\section{Addendum to Experiments}
\label{app:expadd}

\subsection{Implementation Details for BO methods}

We describe some implementation details for all BO methods below.

\begin{itemize}[leftmargin=0.25in]
\item \textbf{Domain:} Given a problem with an arbitrary $d$ dimensional domain, we map it
to $[0,1]^d$ by linearly transforming each coordinate.
\item \textbf{Initialisation:} As explained in the main text, all BO methods were
initialised by uniformly randomly picking $\ninit$ points in the domain.
To facilitate a fair comparison with the random strategies, we also afford them with
the same initialisation, and begin our comparisons in the figures after the
initialisation.
\item \textbf{GP kernel and other hyperarameters:}
The GP hyper-parameters are first learned by maximising the marginal
likelihood~\citep{rasmussen06gps} after the initialisation phase 
 and then updated every $25$ iterations.
For all BO methods, we use a SE kernel and tune the bandwidth for each dimension,
the scale parameter of the kernel and the GP noise variance $(\eta^2)$.
The mean of the GP is set to be the median of all observations.
\item \textbf{UCB methods:} Depending on the methods used, the UCB criterion typically
takes a form $\mu + \betajh\sigma$ where $\mu,\sigma$ are the posterior mean and standard
deviation of the GP. $\betajh$ is a parameter that controls the exploration exploitation
trade-off in UCB methods. 
Following recommendations in~\citep{kandasamy15addBO}, we set it
$\betaj = 0.2d\log(2j + 1)$.
\item \textbf{Selection of $\xj$:}
In all BO methods, the selection of $\xj$ typically takes the form $\xj=\argmax_x
\varphi_j(x)$ where $\varphi_j$ is a function of the GP posterior at step $j$.
$\varphi_j$ is usually called the acquisition in the BO literature.
We maximise $\varphi_j$ using the dividing rectangles algorithm~\citep{jones93direct}.

\end{itemize}

\subsection{Synthetic Experiments}

\textbf{Additional Experiments:}
In Figures~\ref{fig:toyappone} and~\ref{fig:toyapptwo} we
present results on additional synthetic experiments and also repeat those in the main
text in larger figures.
The last panel of Figure~\ref{fig:toyapptwo} compares \seqts, \synts, and \asytss on
the Park1 function.

We describe the construction of the synthetic experiments below.
All the design choices were made arbitrarily.

\textbf{Construction of benchmarks:}
To construct our test functions, we start with the following benchmarks for global
optimisation commonly used in the literature:
Branin $(d=2)$,
Currin-exponential $(d=2)$,
Hartmann3 $(d=3)$,
Park1 $(d=4)$,
Park2 $(d=4)$,
and
Hartmann6 $(d=6)$.
The descriptions of these functions are available in, for e.g.~\citep{kandasamy16mfbo}.
To construct the high dimensional variants, we repeat the same function 
by cycling through different groups of coordinates and add them up.
For e.g. the Hartmann12 function was constructed as
$f(x_{1:12}) = g(x_{1:6}) + g(x_{7:12})$ where $g$ is the Hartmann6 function.
Similarly, for the Park2-16 function we used the Park2-function $4$ times,
for Hartmann18, we used Hartmann6 thrice, and
for CurrinExp-14 we used the Currin-exponential function $7$ times.

\textbf{Noise:}
To reflect the bandit setting, we added Gaussian noise with standard deviation $\eta$
in our experiments.
We used $\eta=0.2$ for CurrinExp, Branin, Park1, Park2, Hartmann3, and Hartmann6;
$\eta=1$ for Park2, Park2-16, Hartmann12, CurrinExp-14, and Hartmann18.
The two choices were to reflect the ``scale'' of variability of the function values
themselves on each test problem.

\textbf{Time distributions:}
The time distributions are indicated on the top of each figure.
In all cases, the time distributions were constructed so that the expected time
to complete one evaluation is $1$ time unit.
Therefore, for e.g. in the Hartmann6 problem, an asynchronous version would
use roughly $12\times 30 = 360$ evaluations while a synchronous version
would use roughly $\frac{12\times 30}{\log(8)} \approx 173$ evaluations.

\subsection{Cifar-10 Experiment}

In the Cifar-10 experiment we use a $6$ layer convolutional neural network.
The first $5$ layers use convolutional filters while the last layer is a fully connected
layer.
We use skip connections~\citep{he2016deep} between the first and third layers and then
the third and fifth layers;
when doing so, instead of just using an indentity transformation $\phi(x) = x$,
we use a linear transformation $\phi(x) = Wx$ as the number of filters could be different
at the beginning and end of a skip connection.
The weights of $W$ are also learned via back propagation as part of the training
procedure.
This modification to the Resnet was necessary in our set up as we are tuning the number of
filters at each layer.

The following are ranges for the number of evaluations for each method over
$9$ experiments: \\
synchronous:
\synbucb: 56 - 68,
\synts:  56 - 68. \\
asynchronous: 
\asyrand: 93 - 105,
\asyei: 83 - 92,
\asyhucb: 85 - 92,
\asyts: 80 - 88.

\newpage

\insertFigToyAppOne
\insertFigToyAppTwo

\end{document}